\documentclass[10pt,twocolumn,letterpaper]{article}

\usepackage[pagenumbers]{cvpr} 

\usepackage[utf8]{inputenc} 
\usepackage[T1]{fontenc}    
\usepackage{url}            
\usepackage{booktabs}       
\usepackage{amsfonts}       
\usepackage{nicefrac}       
\usepackage{microtype}      
\usepackage{xcolor}         

\usepackage{verbatim}
\usepackage{bm}
\usepackage{bbm}
\usepackage{multirow}
\usepackage{titletoc}
\usepackage{array}
\usepackage{pifont}
\usepackage{nccmath}
\usepackage{titletoc}
\usepackage{subcaption}
\usepackage[table]{xcolor}

\usepackage{physics}
\usepackage{amsmath}
\usepackage{tikz}
\usepackage{mathdots}
\usepackage{yhmath}
\usepackage{cancel}
\usepackage{color}
\usepackage{siunitx}
\usepackage{array}
\usepackage{multirow}
\usepackage{amssymb}
\usepackage{gensymb}
\usepackage{tabularx}
\usepackage{extarrows}
\usepackage{enumitem}
\usepackage{booktabs}
\usetikzlibrary{fadings}
\usetikzlibrary{patterns}
\usetikzlibrary{shadows.blur}
\usetikzlibrary{shapes}
\usepackage{wrapfig}

\usepackage[linesnumbered,ruled,vlined]{algorithm2e}
\usepackage{chngcntr}
\usepackage{tocloft}
\usepackage[titletoc,toc,title]{appendix}

\usepackage{tcolorbox}
\tcbuselibrary{skins}

\definecolor{cvprblue}{rgb}{0.21,0.49,0.74}
\definecolor{purple}{rgb}{0.55,0.33,0.78}
\usepackage[pagebackref,breaklinks,colorlinks,anchorcolor=cvprblue, citecolor=cvprblue,linkcolor=cvprblue]{hyperref}

\usepackage[capitalize]{cleveref}
\crefname{table}{Table}{Tables}
\crefname{figure}{Figure}{Figures}
\crefname{algorithm}{Algorithm}{Algorithms}
\crefname{section}{Section}{Sections}
\crefname{appendix}{Appendix}{Appendices}
\crefname{equation}{Eq.}{Eqs.}
\crefname{assmp}{Assumption}{Assumptions}

\usepackage{amsthm}
\theoremstyle{plain}
\newtheorem{thm}{Theorem}
\newtheorem{lem}{Lemma}

\theoremstyle{definition}
\newtheorem{rem}{Remark}
\newtheorem{defn}{Definition}
\newtheorem{assmp}{Assumption}

\crefname{thm}{Theorem}{Theorems}
\crefname{rem}{Remark}{Remarks}
\crefname{prop}{Proposition}{Propositions}

\def \eg {\textit{e.g.}}
\def \ie {\textit{i.e.}}

\def \wrt {w.r.t.}

\def \methodname {Co-Settle}
\def \methodnameb {\textbf{Co-Settle}}
\def \methodfullname {\textit{\textbf{Co}nsistency-\textbf{Se}parability \textbf{T}rade-off \textbf{T}ransfer \textbf{Le}arning}}

\def \tza {\bm{z}_1}
\def \tzb {\bm{z}_2}
\def \tzi {\bm{z}_i}
\def \tbarza {\bar{\bm{z}}_1}
\def \tbarzb {\bar{\bm{z}}_2}
\def \tbarzi {\bar{\bm{z}}_i}
\def \tJ {\bm{J}}
\def \tW {\bm{W}}
\def \tWla {\bm{W}_{1}}
\def \tWlb {\bm{W}_{2}}
\def \tU {\bm{U}}
\def \tQa {\bm{Q}_1}
\def \tQb {\bm{Q}_2}
\def \tSigma {\bm{\Sigma}}
\def \tbarSigma {\bar{\bm{\Sigma}}}
\def \tLambdaW {\bm{\Lambda}_{W}}
\def \tLambdaWla {\bm{\Lambda}_{1}}
\def \tLambdaWlb {\bm{\Lambda}_{2}}
\def \tLambdaSigma {\bm{\Lambda}_{\Sigma}}
\def \tLambdabarSigma {\bm{\Lambda}_{\bar{\Sigma}}}

\def \vta {\bm{v}_{t_1}}
\def \vtb {\bm{v}_{t_2}}
\def \zta {\bm{z}_{t_1}}
\def \ztb {\bm{z}_{t_2}}
\def \ztatilde {\widetilde{\bm{z}}_{t_1}}
\def \pta {\bm{p}_{t_1}}
\def \ptb {\bm{p}_{t_2}}
\def \ptatilde {\widetilde{\bm{p}}_{t_1}}
\def \qta {\bm{q}_{t_1}}
\def \qtb {\bm{q}_{t_2}}
\def \qtatilde {\widetilde{\bm{q}}_{t_1}}
\def \Atatb {\bm{A}_{t_1}^{t_2}}
\def \Atbta {\bm{A}_{t_2}^{t_1}}
\def \Atbtatilde {\widetilde{\bm{A}}_{t_2}^{t_1}}

\definecolor{soft_red}{RGB}{240, 110, 110}
\definecolor{soft_blue}{RGB}{100, 170, 220}
\definecolor{define_red}{RGB}{255, 20, 20}
\definecolor{define_green}{RGB}{41, 165, 31}

\def\paperID{35919} 
\def\confName{CVPR}
\def\confYear{2026}

\title{From Static to Dynamic: Exploring Self-supervised Image-to-Video Representation Transfer Learning}

\author{
	Yang Liu\textsuperscript{1}\hspace{1.2em} Qianqian Xu\textsuperscript{2,3,}\thanks{Corresponding authors}\hspace{1.2em} Peisong Wen\textsuperscript{1}\hspace{1.2em} Siran Dai\textsuperscript{4,5}\hspace{1.2em} Xilin Zhao\textsuperscript{6}\hspace{1.2em} Qingming Huang\textsuperscript{1,2,*} \\
	{\textsuperscript{1}School of Computer Science and Technology, University of Chinese Academy of Sciences} \\
	{\textsuperscript{2}State Key Laboratory of AI Safety, Institute of Computing Technology, Chinese Academy of Sciences} \\
    {\textsuperscript{3}Beijing Academy of Artificial Intelligence} \\
    {\textsuperscript{4}Institute of Information Engineering, Chinese Academy of Sciences} \\
    {\textsuperscript{5}School of Cyber Security, University of Chinese Academy of Sciences} \\
    {\textsuperscript{6}School of Computer Science and Technology, Beijing Institute of Technology} \\
	{\tt\small liuyang232@mails.ucas.ac.cn\hspace{2em} xuqianqian@ict.ac.cn\hspace{2em} wenpeisong@ucas.ac.cn } \\ 
    {\tt\small daisiran@iie.ac.cn\hspace{2em} 1120230539@bit.edu.cn\hspace{2em}  qmhuang@ucas.ac.cn }
}

\begin{document}
\maketitle

\begin{abstract}
Recent studies have made notable progress in video representation learning by transferring image-pretrained models to video tasks, typically with complex temporal modules and video fine-tuning.
However, fine-tuning heavy modules may compromise \textbf{inter-video semantic separability}, \ie, the essential ability to distinguish objects across videos.
While reducing the tunable parameters hinders their \textbf{intra-video temporal consistency}, which is required for stable representations of the same object within a video.
This dilemma indicates a potential \textbf{trade-off} between the intra-video temporal consistency and inter-video semantic separability during image-to-video transfer.
To this end, we propose the \methodfullname{} (\methodnameb{}) framework, which applies a lightweight projection layer on top of the frozen image-pretrained encoder to adjust representation space with a temporal cycle consistency objective and a semantic separability constraint.
We further provide a theoretical support showing that the optimized projection yields a better trade-off between the two properties under appropriate conditions.
Experiments on eight image-pretrained models demonstrate consistent improvements across multiple levels of video tasks with only five epochs of self-supervised training.
The code is available at \href{https://github.com/yafeng19/Co-Settle}{https://github.com/yafeng19/Co-Settle}.

\end{abstract}

\section{Introduction}
\label{sec:introduction}

It has been a long-standing pursuit for the video community to exploit meaningful representations that benefit a wide range of video understanding scenarios. 
Driven by more comprehensive data, more powerful models, and more efficient algorithms, video representation learning has continuously evolved over the past decade, featuring heavy temporal processing mechanisms, such as 3D convolutions~\cite{SlowFast,MAE_ST,VideoMAE,VideoMAEv2}, temporal attention~\cite{TimeSformer,ViViT,XViT}, and inter-frame contrastive frameworks~\cite{contrast1,contrast2,contrast3,SiamMAE,tcore}.

Recent studies~\cite{adapter1,adapter2,adapter3,adapter4,two_stage_1,two_stage_2} have demonstrated that transferring image-pretrained models to the video domain can rival video-pretrained counterparts on multiple video downstream tasks. This observation raises a noteworthy question: 
\textit{How do image-pretrained models contribute to performance improvements on video downstream tasks?}

In fact, image models are typically pretrained on large-scale image datasets with diverse categories~\cite{Imagenet,CLIP}. Such pretraining encourages favorable \textbf{inter-video semantic separability}, \ie, semantic discrimination between different visual categories across videos~\cite{huangtowards,wen2025semantic,cluster}.
Meanwhile, these models exhibit an approximate form of \textbf{intra-video temporal consistency}~\cite{DINOv2,DINOv3}, which produces relatively stable representations for the same object across frames within a video.
However, since this consistency is obtained by pretraining with \textit{simple geometric changes} (\eg, rotation), they fail to establish reliable temporal consistency, due to their lack of exposure to \textit{real-world temporal dynamics}, such as a horse leaping over obstacles with complex transformations.

To leverage the advantages of image-pretrained, prior studies incorporate temporal modeling modules and fine-tune them on video datasets via indirect auxiliary tasks~\cite{adapter1,adapter2,adapter3,adapter4,two_stage_1}.
Yet, on the one hand, as the number of tunable parameters increases without proper constraint, the transfer process risks catastrophic forgetting of the semantic separability acquired from the image-pretraining stage~\cite{forgetting1,forgetting2,forgetting3}.
On the other hand, if we restrict the number of tunable parameters to preserve separability, these methods fail to achieve sufficient temporal consistency. We argue that part of the parameters is occupied by indirect auxiliary tasks to learn information beyond consistency.
This dilemma reveals a potential \textbf{trade-off} between the \textbf{intra-video temporal consistency} and \textbf{inter-video semantic separability}, thus calling for a careful balance between these two properties during image-to-video transfer.

In light of these challenges, we propose a \methodfullname{} (\methodnameb{}) framework, which applies a learnable lightweight projection on top of the frozen image encoder to adjust the representation space.
To enhance temporal consistency, we design a cycle consistency objective for fine-grained correspondence learning across frames.
To maintain semantic separability, we introduce a Kullback-Leibler divergence constraint to mitigate forgetting of semantic separability after the projection.

To provide an interpretable insight into the proposed framework, we further present theoretical justification for the trade-off mechanism between temporal consistency and semantic separability. Spectral analysis of the projection layer reveals that the optimized projection can increase the margin between inter- and intra-video distance in the representation space, leading to a more effective trade-off between the two properties. 
As shown in \cref{fig:intro}, this improved trade-off results in better performance on video tasks.

Experimental results on eight ViT-based image-pretrained models demonstrate consistent improvements across several dense-level, frame-level, and video-level downstream video tasks, using only five epochs of self-supervised training on video datasets. These results suggest a potential solution for efficient image-to-video representation transfer learning.

\begin{figure}[t]
  \centering
    \includegraphics[width=0.98\linewidth]{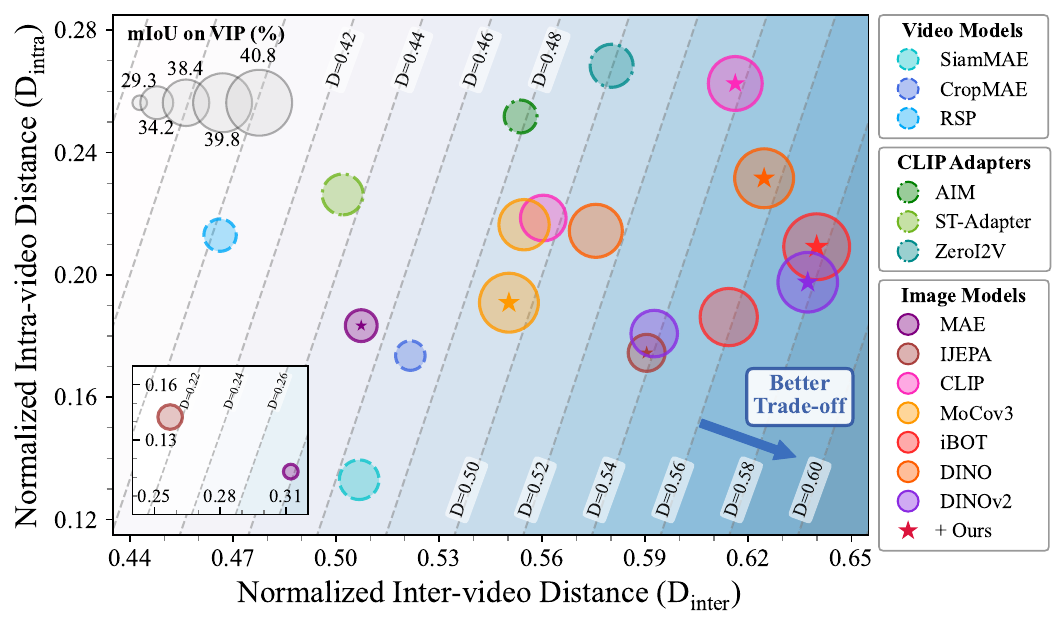}
    \vspace{-10pt}
    \caption{Comparison of video representation quality with recent visual representation learning models on the Kinetics-400~\cite{Kinetics} validation set. Favorable video representations should exhibit strong intra-video temporal consistency (lower intra-video distance $D_{intra}$) and clear inter-video semantic separability (higher inter-video distance $D_{inter}$) jointly, yet the two objectives often compete since the two distances co-vary.
    Applying our method to image-pretrained models leads to consistent improvements on the margin of inter- and intra-video distance $D=D_{inter}-\gamma D_{intra}$ (detailed in \cref{subsec:distance_metrics}), indicating a better trade-off between the two properties, and therefore leading to improved performance on video downstream tasks.}
    \label{fig:intro}
  \vspace{-16pt}
\end{figure}

The main contributions are summarized as follows:
\begin{itemize}
    \item Methodologically, we propose the \methodnameb{} framework, which applies a lightweight projection layer on the image-pretrained encoder for representation space adjustment to balance temporal consistency and semantic separability.
    \item Theoretically, through spectral analysis, we derive optimal conditions for the projection layer, which leads to a more effective trade-off between the two properties.
    \item Experimentally, we evaluate our method on eight image foundation models, achieving consistent performance improvements across multiple granularity levels of video tasks after an efficient self-supervised learning process.
\end{itemize}

\section{Related Work}
\label{sec:related_work}

\textbf{Self-supervised visual representation learning}.
Self-supervised learning has enabled models to develop generalizable representations for diverse downstream tasks.
Contrastive learning methods leverage discriminative signals from different views~\cite{MoCov3, BYOL, SimCLRv1, SiameseIM, SimSiam, dai2025exploring2} to enforce semantical consistency for related contents.
Masked modeling methods learn meaningful representations by predicting the raw pixels with encoder-decoder structures~\cite{MAE, I_JEPA, StoP, BEiTv1, BEiTv2}, or by reconstructing latent tokens within a self-distillation framework~\cite{iBOT, DINO, DINOv2, DINOv3, wen2025semantic}.
Extending these methods with the temporal dimension, early works introduce masked video modeling with high masking ratios~\cite{MAE_ST, VideoMAE, VideoMAEv2, DropMAE, MaskViT, VideoMAC} while recent efforts~\cite{SiamMAE, CropMAE, RSP, STP, tcore} seek more efficient temporal modeling and prediction algorithms.

\textbf{Image-to-video transfer learning}.
Recent works adopt parameter-efficient transfer strategies that update only a small subset of parameters while preserving comparable performance.
Mainstream approaches insert adaptation modules into CLIP-pretrained Vision Transformers~\cite{Vision_Transformer}, enabling spatiotemporal adaptation via convolutional or attention-based operators~\cite{adapter1,adapter2,adapter3,adapter4,adapter5}.
However, due to the reliance on task-specific supervision~\cite{Kinetics,ssv2}, these methods require separate fine-tuning when applied to different tasks.

\textbf{Temporal cycle consistency}.
The inherent visual correspondence between adjacent frames provides natural supervisory signals to capture spatiotemporal coherence~\cite{continuity, smooth}.
Many studies exploit this property to learn semantically consistent representations within a cycle structure, benefiting downstream tasks such as classification~\cite{Kinetics, ssv2, wang2022openauc, wang2025unified}, video retrieval~\cite{s2vs, liu2024not}, object segmentation~\cite{DAVIS17, VIP, sam2, han2024aucseg, li2024size}, and point tracking~\cite{BADJA,tap_davis,Kinetics_track}. 
Early methods focus on bidirectional patch-/object-level tracking~\cite{cycle_patch_1,cycle_patch_2,cycle_patch_3,cycle_patch_4,dense_adapter}, while others align feature distributions across related videos~\cite{cycle_match_1,cycle_match_2,cycle_match_3,clip_adapter}.
Another line of work introduces random walk strategies~\cite{cycle_walk_1,cycle_walk_2,cycle_walk_3}, which guide representation learning by maximizing the probability that each patch returns to itself through a palindrome sequence.

\begin{figure*}[t]
  \centering
    \includegraphics[width=0.94\linewidth]{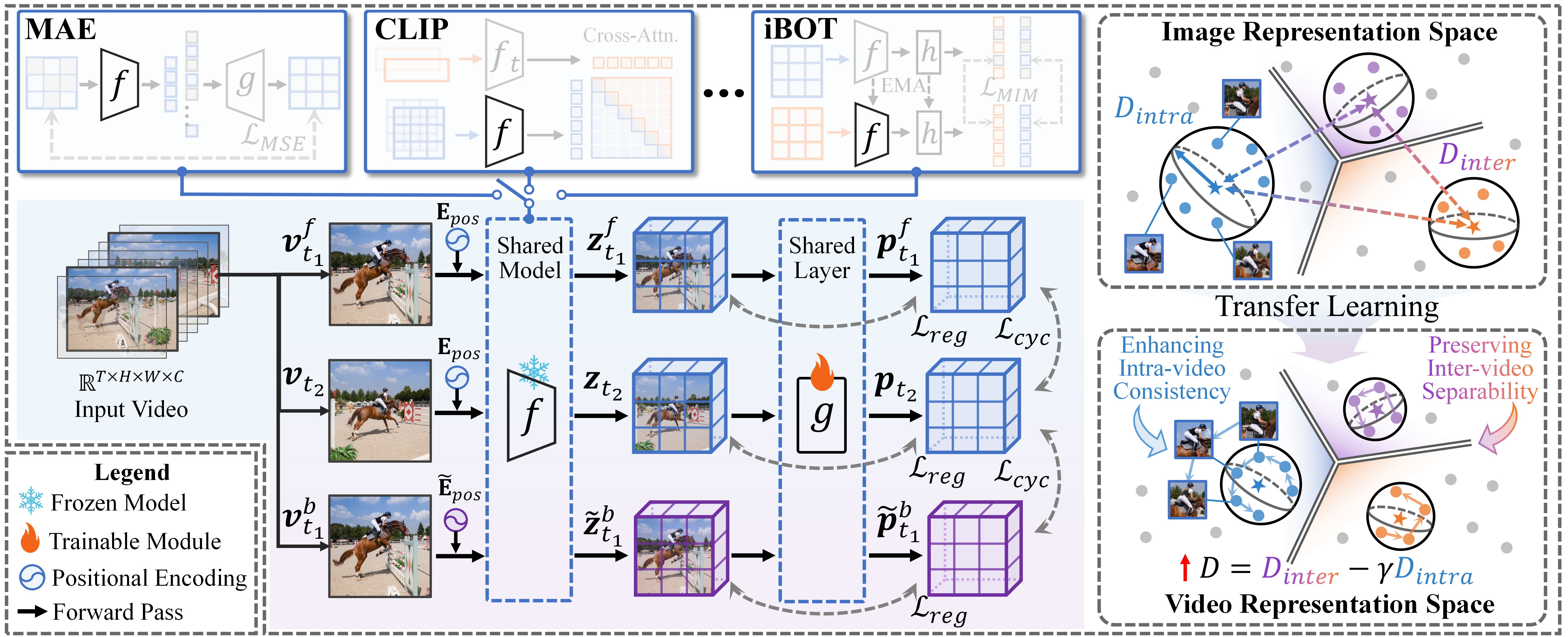}
    \vspace{-4pt}
    \caption{Overview of our image-to-video transfer learning framework. Two frames are sampled from each video to construct a cyclic sequence. A frozen image-pretrained encoder extracts patch-level features, which are then mapped by a learnable projection layer. The projection layer is trained with a temporal cycle-consistency loss and a semantic separability constraint for representation adjustment, thereby promoting a better trade-off between intra-video temporal consistency and inter-video semantic separability.}
  \label{fig:pipeline}
  \vspace{-10pt}
\end{figure*}

\section{Image-to-video Representation Transfer}
\label{sec:method}

To facilitate image-to-video transfer learning, we propose the \methodfullname{} (\methodnameb{}) framework as shown in \cref{fig:pipeline}. We first encode frames with a frozen image-pretrained encoder and then leverage a learnable lightweight layer to project the representation space (\cref{subsec:task_defination}). We optimize the projected representations for two goals: \textbf{1)} intra-video temporal consistency, enforced by a cycle-consistency learning strategy (\cref{subsec:temporal_consistency}); and \textbf{2)} inter-video semantic separability, preserved by a dimensionality constraint on the representations (\cref{subsec:semantic_separability}).

\subsection{Task definition}
\label{subsec:task_defination}

As a spatiotemporal volume~\cite{volume_1,volume_2,Kinetics_track}, a video can be represented as an ordered sequence of $T$ frames $\bm{V} = \{\bm{v}_{t} \in \mathbb{R}^{H \times W \times C}\}_{t=1}^{T}$, where $H$, $W$, and $C$ denote the height, width, and channels of each frame. 
Each frame $\bm{v}_{t}$ can be divided into $N = N_H \times N_W = \lceil H / p \rceil \times  \lceil W / p \rceil$ non-overlapping patches of $p^2$ pixels, where $p$ is the patch size.

Given a video $\bm{V}$, we randomly sample two frames $\vta$ and $\vtb$ with a temporal offset determined by $\delta \in (0,1)$.
Each frame is embedded through a frozen image-pretrained encoder $f: \mathbb{R}^{H \times W \times C} \rightarrow \mathbb{R}^{N \times d}$ with embedding dimension $d$, producing frame-wise representations $\zta = f(\vta; \mathbf{E}_{\text{pos}})$ and $\ztb = f(\vtb; \mathbf{E}_{\text{pos}})$, where $\mathbf{E}_{\text{pos}}$ is the positional encoding of $f$.
These representations inherit rich semantic priors from the image-pretrained model, yet still lack clear temporal correspondence.
Then, we apply a lightweight layer $g: \mathbb{R}^{N \times d} \rightarrow \mathbb{R}^{N \times d}$ with parameters limited to a linear layer and a \texttt{LayerNorm} to project $\zta, \ztb$ as $\pta = g(\zta)$ and $\ptb = g(\ztb)$. This projection maps the static representations into a shared latent space, where we aim to enhance temporal consistency while preserving semantic separability.

\begin{figure*}[t]
    \centering
    \begin{subfigure}{0.3\linewidth}
        \centering
        \includegraphics[width=\linewidth]{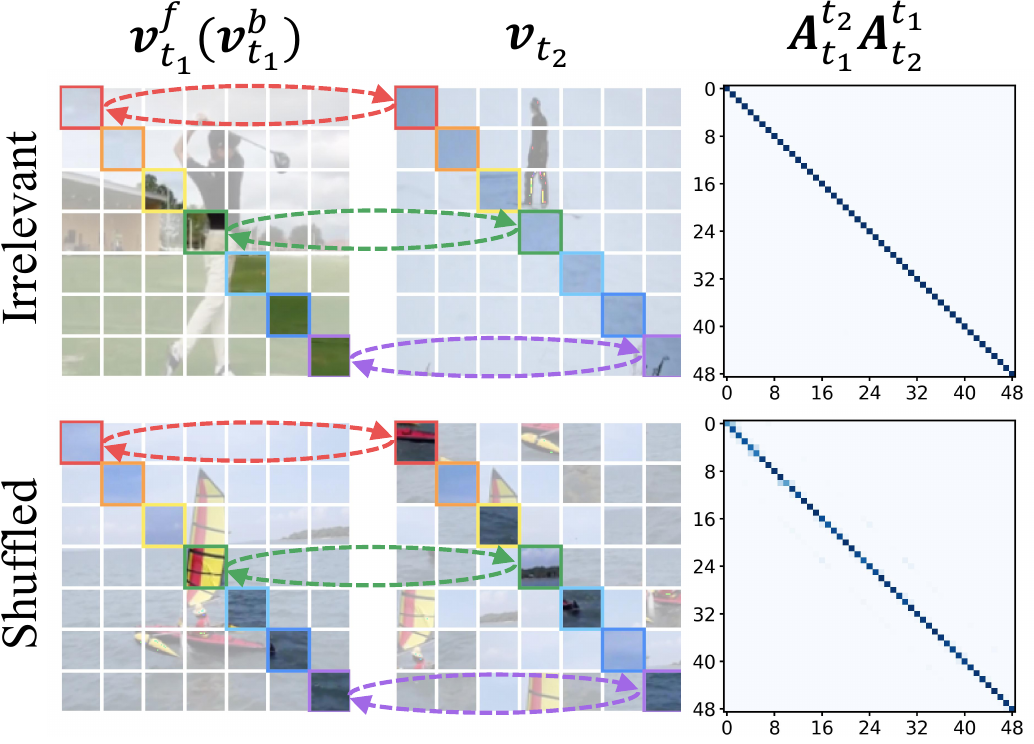}
        \label{fig:shortcut_observation}
    \end{subfigure}
    \begin{subfigure}{0.45\linewidth}
        \centering
        \includegraphics[width=\linewidth]{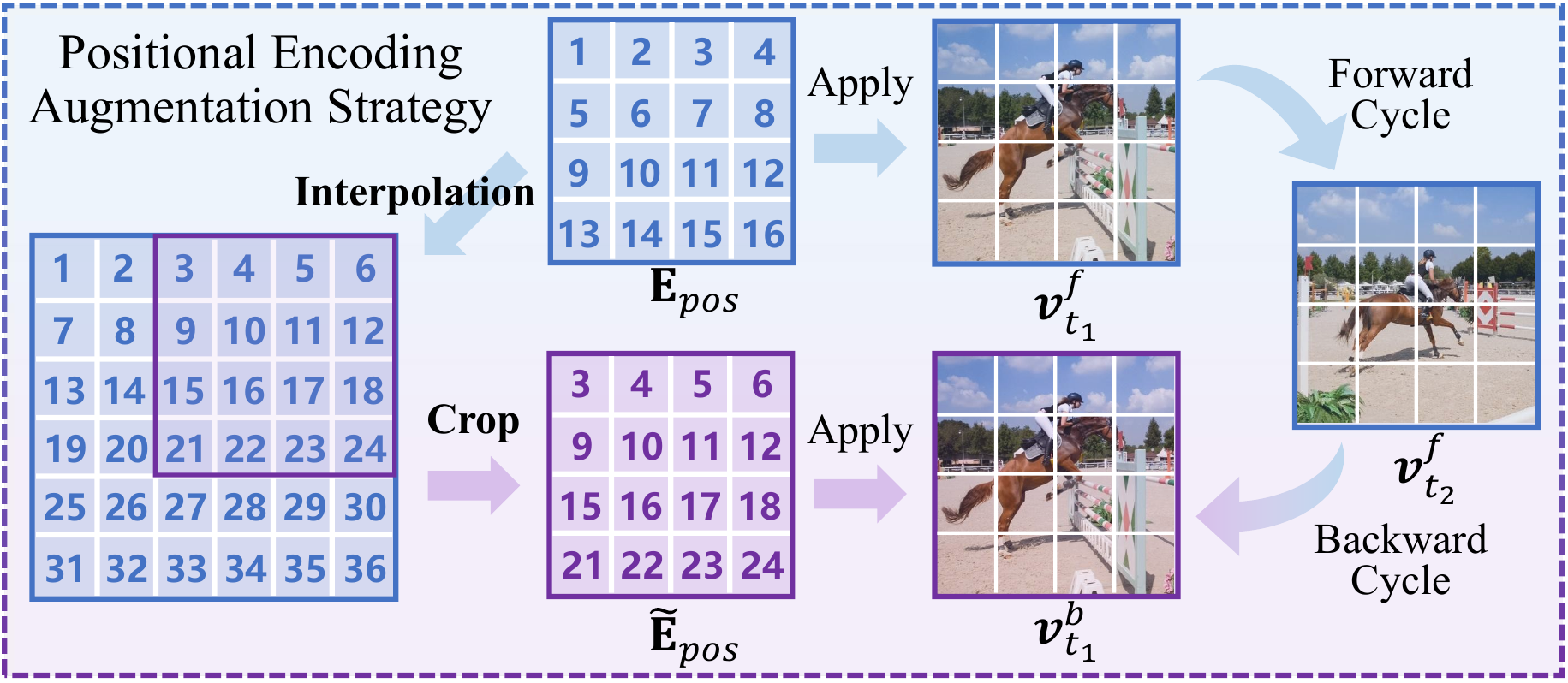}
        \label{fig:PEA_module}
    \end{subfigure}
    \begin{subfigure}{0.24\linewidth}
        \centering
        \includegraphics[width=\linewidth]{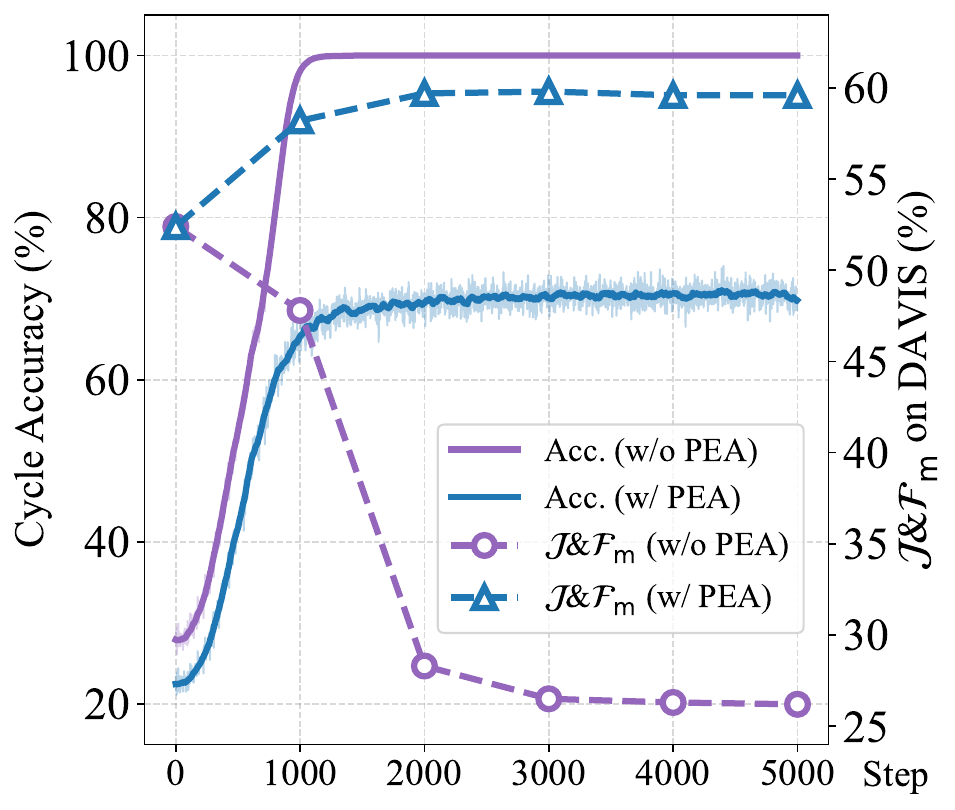}
        \label{fig:ablation_figure}
    \end{subfigure}
    
    \vspace{-14pt}
    \caption{
    \textbf{Left}: Observations on shortcuts. Patches with the same color box denote correspondence.
    \textbf{Middle}: Overview of our PEA strategy.
    \textbf{Right}: Cycle-consistent accuracy and downstream performance dynamics during training \textcolor{cvprblue}{with} or \textcolor{purple}{without} our PEA strategy on MAE encoder.
    }
    \label{fig:shortcut}
    \vspace{-14pt}
\end{figure*}

\subsection{Intra-video temporal consistency learning}
\label{subsec:temporal_consistency}
To provide direct and explicit guidance for learning temporal consistency, our core principle is to establish reliable temporal correspondences between patches.
Without supervision for precise alignment, prior works~\cite{cycle_walk_1,cycle_walk_2,cycle_walk_3} introduce \textit{Contrastive Random Walk} (CRW) structure to enhance cross-frame correspondences in videos.
CRW constructs a forward-backward sequence $\vta \xrightarrow{\text{forward}} \vtb \xrightarrow{\text{backward}} \vta$, where the first and last frames are identical.
For clarity, we denote the first forward frame as $\vta^f$ and the last backward frame as $\vta^b$.
The objective of CRW is to maximize the probability that each patch in $\vta^f$ returns to its original position in $\vta^b$ after traversing through the intermediate frame $\vtb$.

Let $\pta^f$, $\ptb$, and $\pta^b$ denote the projected representations of $\vta^f$, $\vtb$, and $\vta^b$, respectively. We then calculate the forward and backward correlation matrices as $\Atatb = \text{softmax}_\tau ( \pta^f \ptb^\top)$ and $\Atbta = \text{softmax}_\tau ( \ptb \pta^{b^\top} )$. 
Concretely, each element in $\Atatb$ can be computed as:
\begin{equation}
    \begin{aligned}
        \Atatb(i,j) = \frac{\text{exp}(d(\pta(i),\ptb(j))/\tau)}{\sum_{l=1}^N \text{exp}(d(\pta(i),\ptb(l))/\tau)},
    \end{aligned}
    \label{eq:correlation_matrix}
\end{equation}
where $\tau\in\mathbb{R}^+$ represents the temperature hyperparameter and $d(\cdot,\cdot)$ denotes the dot-product similarity.
Intuitively, the matrix $\Atatb$ quantifies the attention distribution from each patch in $\vta^f$ to all patches in $\vtb$, indicating the transition paths of each patch between the two frames.
Accordingly, the CRW objective can be formulated as aligning the correlation matrix product chain $\Atatb \Atbta$ with the identity matrix $\bm{I}$ via the Cross-Entropy loss $\mathcal{L}_{CRW}=\mathcal{L}_{CE}\left(\Atatb\Atbta, \bm{I}\right)$.

Despite its elegant formation, CRW underperforms on modern Vision Transformer (ViT) backbones~\cite{Vision_Transformer}. 
As shown in \cref{fig:shortcut} (Left), we test two cases: \textbf{1)} \textit{Irrelevant setting}: $\vta^f$ ($\vta^b$) and $\vtb$ are sampled from unrelated videos; \textbf{2)} \textit{Shuffled setting}: the patches of $\vtb$ are randomly shuffled.
In both cases, $\Atatb \Atbta$ rapidly converges to $\bm{I}$, minimizing $\mathcal{L}_{CRW}$ while failing to capture any meaningful visual correspondences.
This phenomenon reveals a shortcut solution~\cite{shortcut_2, cycle_walk_3}, where the model reduces the loss by exploiting exact positional cues instead of learning semantic relations.
We attribute this to the explicit positional encodings in ViTs, which are directly injected into the input tokens and propagated via the global attention, enabling patches to match by absolute location even under strong appearance changes.

To cope with this shortcut, we introduce \textit{\textbf{P}ositional \textbf{E}ncoding \textbf{A}ugmentation} (PEA), as illustrated in \cref{fig:shortcut} (Middle). 
PEA perturbs direct positional matching through an asymmetric design, inspired by information bottlenecks~\cite{bottleneck1,bottleneck2} in self-distillation frameworks~\cite{BYOL,SimSiam,DINO}.
Specifically, PEA interpolates the pretrained positional encoding $\mathbf{E}_{\text{pos}}$ with a controllable amplitude $\alpha\in\mathbb{R}^+$ and then applies a random crop to recover the original size, yielding an augmented version $\widetilde{\mathbf{E}}_{\text{pos}}$.
Afterward, we use $\widetilde{\mathbf{E}}_{\text{pos}}$ to encode the backward frame $\ztatilde^b = f(\vta^b; \widetilde{\mathbf{E}}_{\text{pos}})$ and obtain the corresponding projected representation $\ptatilde$.
This strategy breaks reliance on exact positional matches while preserving local positional relations, thereby leading to a stable learning process, as shown in \cref{fig:shortcut} (Right).

Given the forward correlation $\Atatb = \text{softmax}_\tau ( \qta^f \qtb^{\top})$ and the asymmetric backward correlation $\Atbtatilde = \text{softmax}_\tau ( \qtb \qtatilde^{b^\top})$, we define the cycle-consistency loss as $\mathcal{L}_{cyc} = \mathcal{L}_{CE}\left(\Atatb\Atbtatilde, \bm{I}\right)$ or formally:
\begin{equation}
    \begin{aligned}
        \mathcal{L}_{cyc}
        & = -\sum_{i=1}^{N}\text{log}P\left(X_d=\qtatilde^b(i)|X_s=\qta^f(i)\right),
    \end{aligned}
    \label{eq:cyc_loss}
\end{equation}
where $X_s$ and $X_d$ denote the start and destination patches of the random walker.
In this way, the cycle consistency objective can promote the model to learn effective temporal correspondences, facilitating the transfer from static image representations to dynamic video contexts.

\subsection{Inter-video semantic separability constraint}
\label{subsec:semantic_separability}

While image-pretrained models exhibit favorable semantic separability, optimizing $g$ solely for temporal alignment on video data with limited category diversity can lead to catastrophic forgetting of this property, and may further induce dimensional collapse~\cite{collapse1,collapse2,collapse3,dai2025exploring}.

To mitigate this issue, we introduce a distributional regularization based on Kullback-Leibler (KL) divergence to preserve feature diversity across the projection layer.
For each pair $(\bm{p}, \bm{z}) \in \mathcal{S} = \{ (\pta^f, \zta^f), (\ptb, \ztb), (\ptatilde^b, \ztatilde^b) \}$, we compute normalized distributions via softmax along the feature dimension: $P=\text{softmax}(\bm{p}), Z=\text{softmax}(\bm{z})$.
The regularization loss is then formulated as:
\begin{equation}
    \begin{aligned}
        \mathcal{L}_{reg}
        & = \frac{1}{\left|\mathcal{S}\right|}\sum_{(\bm{p}, \bm{z}) \in \mathcal{S}} \sum_{i=1}^{d} P(i) \log \frac{P(i)}{Z(i)},
    \end{aligned}
    \label{eq:reg_loss}
\end{equation}
where $d$ denotes the feature dimension. This distributional regularization enforces a constraint on the global semantic structure while allowing flexible refinement.

By jointly optimizing the cycle consistency loss and the constraint term, we can achieve a balanced learning between intra-video temporal consistency and inter-video semantic separability.
The final training objective for $g$ is formulated as \cref{eq:total_loss}, where $\lambda$ controls the strength of the constraint, and we will provide a justification for $\lambda$ in the \cref{sec:theory}. 
\begin{equation}
    \mathcal{L}_{total} = \mathcal{L}_{cyc} + \lambda\mathcal{L}_{reg}.
\label{eq:total_loss}
\end{equation}

\section{Theoretical Analysis}
\label{sec:theory}

In this section, we explore how the proposed method achieves our target, \ie, improving the \textbf{intra-video temporal consistency} without largely affecting the \textbf{inter-video semantic separability}.
Generally, our analysis leads to two main conclusions: 
\textbf{a)} Within our proposed method, both linear-based and MLP projection rebalance different dimensions of the representation space in a similar mechanism (\cref{thm:thm1_main}).
\textbf{b)} This rebalance yields a better trade-off between the two properties under appropriate conditions (\cref{thm:thm2_main}).
This section is self-contained and may be skipped without affecting the overall understanding of our framework. The detailed derivations are provided in the \textit{Supplementary Material}.

Formally, given the original representation of a patch $\tzi\in\mathbb R^d$, we aim to learn a projection $g$ that maps $\tzi$ to $\bm{p}_i = g(\tzi) \in \mathbb{R}^{d}$.
Since directly analyzing the original objectives in \cref{eq:cyc_loss} and \cref{eq:reg_loss} is challenging, we introduce simplified yet equivalent surrogates to facilitate the analysis.

\textbf{Objective 1 (Temporal Cycle Consistency).} 
This term encourages alignment between temporally corresponding patches. We quantify it with the metric in \cref{eq:M_cyc_main}.
Note that minimizing $M_{\text{cyc}}$ is equivalent to minimizing the cycle-consistency loss $L_{\text{cyc}}$, since both decrease as temporal consistency improves and share the same optimality conditions.
\begin{equation}
M_{\text{cyc}} = \frac12\,\mathbb{E}_{\tza,\tzb}\!\bigl[\lVert g(\tza)-g(\tzb)\rVert^{2}\bigr].
\label{eq:M_cyc_main}
\end{equation}

\textbf{Objective 2 (Semantic Separability Constraint)}: 
The KL divergence constraint in \cref{eq:reg_loss} preserves the distance relationships between patches before and after the projection, which is equivalent to constraining the projection to be isometric. This property can be measured by the orthogonality of the Jacobian matrix~\cite{vae1,vae2,vae3,vae4} of $g$, as formulated in \cref{eq:M_reg_main}. Therefore, we use it as an approximation of $L_{\text{reg}}$.
\begin{equation}
M_{\text{reg}} = \frac{1}{2}\,\mathbb{E}_{\tzi}\!\bigl[\lVert \tJ_g(\tzi)\tJ_g(\tzi)^{\top}-\bm{I}\rVert_{F}^{2}\bigr].
\label{eq:M_reg_main}
\end{equation}

Combining the two surrogates yields the overall objective:
\begin{equation}
\min_g \; M(g) = M_{\text{cyc}} + \lambda M_{\text{reg}}.
\label{eq:theory_obj}
\end{equation}

We now consider two representative cases for $g$:
\textbf{i)} A linear projection: $g(\bm{z}) = \tW \bm{z}$;
\textbf{ii)} A two-layer MLP: $g(\bm{z}) = \tWlb \, \phi(\tWla \bm{z})$ with activation function $\phi(\cdot)=tanh(\cdot)$, and this case represents more complex modules.
The following theorem analyzes the spectral properties of the optimal solution under both cases, illustrating how the projection affects the quality of the transferred representation.

\begin{thm}[Spectral Properties of Optimal Projections, Informal]
    \label{thm:thm1_main}
    Denote the eigenvalues of intra-video covariance matrix $\tSigma$ are $\{\sigma_i\}_{i=1}^d$.
    For case i), let $\{\mu_i\}_{i=1}^d$ be the eigenvalues of $\tW$.
    Assume $\tW$ and $\tSigma$ are positive semi-definite. For case ii), let $\{\mu_{1,i}\}_{i=1}^d$ and $\{\mu_{2,i}\}_{i=1}^d$ be the eigenvalues of $\tWla$ and $\tWlb$, respectively. Assume $\tzi\sim\mathcal{N}(\bm{0},\tSigma)$, $\tWla$, $\tWlb$ and $\tSigma$ are positive semi-definite. 
    Then the eigenvalues of the optimal projection $\tW^\star$, ${\tWla}^\star$ and ${\tWlb}^\star$ obey:
    \begin{equation}
        \mu_i^\star=
        \mu_{1,i}^\star\cdot\mu_{2,i}^\star=
        \begin{cases}
            0, & \sigma_i > 2\lambda,\\
            \sqrt{1-\frac{\sigma_i}{2\lambda}}, & \sigma_i\leq 2\lambda.
        \end{cases}
    \end{equation}
\end{thm}
\begin{rem}
\cref{thm:thm1_main} reveals a \emph{soft thresholding} behavior for both cases. Directions with large temporal variance ($\sigma_i > 2\lambda$) are suppressed ($\mu_i^\star = 0$ or $\mu_{1,i}^\star\cdot\mu_{2,i}^\star = 0$), those with smaller variance are gradually scaled toward unit norm.
Intuitively, components that already exhibit strong temporal consistency require less adjustment, whereas low-variance directions are amplified until the orthogonality penalty balances the consistency gain.
For the two-layer MLP, when the inputs fall within the approximate linear region of the \texttt{tanh} function, \textbf{the overall transformation presents similar representation scaling behavior as the linear layer}.
Similar conclusions hold for more complex architectures.
\label{rem:rem1}
\end{rem}

Since both cases yield similar spectral effects, we focus on examining whether a single linear layer is sufficient to improve the trade-off between temporal consistency and semantic separability.
To this end, we first define two metrics to quantify these two competing objectives:
\textbf{1) Intra-video distance}: $D_{intra}(\tza,\tzb) = \mathbb{E}_{\tza, \tzb}\left[\lVert \tza - \tzb\rVert^2\right]$, which measures the average distance between temporally corresponding patches within a video.
\textbf{2) Inter-video distance}: $D_{inter}(\tza,\tzb) = \mathbb{E}_{\tbarza, \tbarzb}\left[\lVert \tbarza - \tbarzb\rVert^2\right]$, calculating the average distance between video-level representations, where $\tbarzi = \mathbb{E}_{\bm{z} \in f(\bm{V}_i)} \left[\bm{z}\right]$ is the mean representation of the video $\bm{V}_i$.
Then we define the margin of these two metrics as $D(\tza,\tzb) = D_{inter}(\tza,\tzb)-\gamma D_{intra}(\tza,\tzb)$, reflecting the degree of separation between the two properties, where a larger value indicates a better trade-off.
In the following theorem, we present how this margin metric evolves.

\begin{thm}[Trade-off Improvement, Informal]
\label{thm:thm2_main}
Let 
$\tSigma = \mathbb{E}_{\tza,\tzb}\!\bigl[(\tza-\tzb)(\tza-\tzb)^\top\bigr]$, $\tbarSigma = \mathbb{E}_{\tbarza,\tbarzb}\!\bigl[(\tbarza-\tbarzb)(\tbarza-\tbarzb)^\top\bigr]$ denote the intra-video and inter-video covariance matrices, with eigenvalues $\{\sigma_i\}_{i=1}^d$ and $\{\tau_i\}_{i=1}^d$, respectively. 
Assume $\forall j, \tau_j = \frac 1 d \sum_{i=1}^d\sigma_i = \tau$. For the linear projection $\bm{p}=g(\bm{z}) = \tW \bm{z}$ with eigenvalues of the optimal $\tW^{\star}$ are $\mu_i^\star = \sqrt{1 - \frac{\sigma_i}{2\lambda}}$ (when $\sigma_i\leq 2\lambda$), the improvement in the margin metric $\Delta = D(\bm{p}_1,\bm{p}_2)-D(\tza,\tzb)$ is given by:
\begin{equation}
\Delta = \sum_{\sigma_i \le 2\lambda} (\tau - \sigma_i)\left(1 - \frac{\sigma_i}{2\lambda}\right) > 0.
\end{equation}
\end{thm}
\begin{rem}
\cref{thm:thm2_main} 
indicates that the margin metric between inter- and intra-video distance can be enhanced under the optimal linear projection, thus yielding a better trade-off between temporal consistency and semantic separability.
\label{rem:rem2}
\end{rem}

\section{Experiments}
\label{sec:experiments}

\begin{table*}[htbp]
  \centering
    \caption{Comparison with image and video representation learning methods on three dense-level video downstream tasks. Results marked with $^\dagger$ are reported from prior works; $^\ddagger$ denotes evaluations based on the official pretrained weights. The best results are highlighted in \textbf{bold}. INet, K400, WIT, and LAION represent ImageNet-1k~\cite{Imagenet}, Kinetics-400~\cite{Kinetics}, WIT-400M~\cite{CLIP}, and LAION-400M~\cite{LAION} datasets.}  
    \vspace{-4pt}
    \setlength{\tabcolsep}{9pt}
  \resizebox{0.92\textwidth}{!}{%
    \begin{tabular}{ccl|ccc|cccccc}
    \toprule
    \multicolumn{2}{c}{\multirow{2}[2]{*}{Type}} & \multirow{2}[2]{*}{Method} & \multirow{2}[2]{*}{Backbone} & \multirow{2}[2]{*}{Dataset} & \multirow{2}[2]{*}{Epoch} & VIP & \multicolumn{3}{c}{DAVIS-2017}   & \multicolumn{2}{c}{JHMDB} \\
        &  &   &      &       &   & mIoU & $\mathcal{J}\&\mathcal{F}_{\mathrm{m}}$ & $\mathcal{J}_{\mathrm{m}}$   & $\mathcal{F}_{\mathrm{m}}$   & PCK@0.1 & PCK@0.2 \\

    \midrule
     \multicolumn{2}{c}{\multirow{6}[1]{*}{\centering \shortstack{Video \\ Pretrained}}} 
       & VideoMAE$^\ddagger$~\cite{VideoMAE} & ViT-L/16 & K400 & 1600 & \cellcolor[rgb]{ 1,  .973,  .969}25.6  & \cellcolor[rgb]{ .961,  .98,  .992}45.0  & \cellcolor[rgb]{ .961,  .98,  .992}43.6  & \cellcolor[rgb]{ .961,  .98,  .992}46.5  & \cellcolor[rgb]{ .8,  .894,  .737}43.3  & \cellcolor[rgb]{ .851,  .922,  .808}70.5  \\
       & & MAE-ST$^\dagger$~\cite{MAE_ST} & ViT-L/16 & K400  & 1600  & \cellcolor[rgb]{ .984,  .875,  .804}33.2  & \cellcolor[rgb]{ .863,  .918,  .965}54.6  & \cellcolor[rgb]{ .792,  .875,  .949}55.5  & \cellcolor[rgb]{ .925,  .957,  .98}53.6  & \cellcolor[rgb]{ .776,  .878,  .706}44.4  & \cellcolor[rgb]{ .804,  .894,  .741}72.5  \\
       & & DropMAE$^\dagger$~\cite{DropMAE}  & ViT-B/16 & K400 & 1600  & \cellcolor[rgb]{ .996,  .929,  .894}31.1  & \cellcolor[rgb]{ .898,  .941,  .976}53.4  & \cellcolor[rgb]{ .902,  .945,  .976}51.8  & \cellcolor[rgb]{ .89,  .937,  .973}55.0  & \cellcolor[rgb]{ .82,  .902,  .765}42.3  & \cellcolor[rgb]{ .882,  .937,  .851}69.2  \\
       & & SiamMAE~\cite{SiamMAE} & ViT-B/16  & K400  & 400  &  \cellcolor[rgb]{ .973,  .796,  .678}36.1  & \cellcolor[rgb]{ .675,  .8,  .918}60.9  & \cellcolor[rgb]{ .675,  .8,  .918}59.4  & \cellcolor[rgb]{ .698,  .816,  .925}62.4  & \cellcolor[rgb]{ .624,  .8,  .502}47.0  & \cellcolor[rgb]{ .671,  .824,  .565}74.9  \\
       & & CropMAE$^\dagger$~\cite{CropMAE}   &  ViT-B/16 & K400   & 400  & \cellcolor[rgb]{ .988,  .878,  .812}33.0  & \cellcolor[rgb]{ .769,  .859,  .941}57.8  & \cellcolor[rgb]{ .753,  .847,  .937}56.9  & \cellcolor[rgb]{ .792,  .875,  .949}58.7  & \cellcolor[rgb]{ .725,  .851,  .639}45.3  & \cellcolor[rgb]{ .776,  .878,  .706}73.3  \\
       & & RSP$^\dagger$~\cite{RSP}   & ViT-B/16 & K400 & 400   & \cellcolor[rgb]{ .984,  .851,  .769}34.0  & \cellcolor[rgb]{ .69,  .808,  .922}60.5  & \cellcolor[rgb]{ .725,  .831,  .933}57.8  & \cellcolor[rgb]{ .675,  .8,  .918}63.2  & \cellcolor[rgb]{ .682,  .831,  .58}46.0  & \cellcolor[rgb]{ .69,  .835,  .592}74.6  \\

    \midrule\midrule
    \multicolumn{1}{c}{\multirow{16}[1]{*}{\centering \shortstack{Image \\ Pretrained \\ +\textbf{\textit{Ours}}}}} 
       &  \multicolumn{1}{c}{\multirow{4}[1]{*}{\centering \shortstack{Mask \\ Modeling }}}
       & MAE$^\ddagger$~\cite{MAE}  & ViT-B/16 & INet & 800  &  \cellcolor[rgb]{ 1,  .973,  .969}29.3  & \cellcolor[rgb]{ .925,  .957,  .98}52.4  & \cellcolor[rgb]{ .925,  .957,  .98}51.0  & \cellcolor[rgb]{ .922,  .953,  .98}53.9  & \cellcolor[rgb]{ .831,  .91,  .78}41.6  & \cellcolor[rgb]{ .882,  .937,  .851}69.3  \\
       & & MAE +\textbf{\textit{Ours}} & ViT-B/16 & K400 &  +5 & \cellcolor[rgb]{ .984,  .859,  .78}\textbf{33.8} & \cellcolor[rgb]{ .714,  .827,  .929}\textbf{59.6} & \cellcolor[rgb]{ .718,  .827,  .929}\textbf{58.0} & \cellcolor[rgb]{ .729,  .835,  .933}\textbf{61.2} & \cellcolor[rgb]{ .537,  .753,  .392}\textbf{48.4} & \cellcolor[rgb]{ .545,  .757,  .404}\textbf{76.7} \\
       \cmidrule{3-12}
       
       & & I-JEPA~\cite{I_JEPA}  & ViT-B/16 & INet & 800  & \cellcolor[rgb]{ .992,  .918,  .878}31.5  & \cellcolor[rgb]{ .882,  .929,  .973}53.9  & \cellcolor[rgb]{ .871,  .922,  .969}52.9  & \cellcolor[rgb]{ .898,  .941,  .976}54.8  & \cellcolor[rgb]{ .812,  .898,  .757}42.6  & \cellcolor[rgb]{ .831,  .91,  .78}71.4  \\
       & & I-JEPA +\textbf{\textit{Ours}} & ViT-B/16 & K400 &  +5 & \cellcolor[rgb]{ .976,  .82,  .714}\textbf{35.3} & \cellcolor[rgb]{ .741,  .843,  .937}\textbf{58.7} & \cellcolor[rgb]{ .737,  .839,  .933}\textbf{57.4} & \cellcolor[rgb]{ .765,  .855,  .941}\textbf{59.9} & \cellcolor[rgb]{ .776,  .878,  .706}\textbf{44.4} & \cellcolor[rgb]{ .784,  .886,  .718}\textbf{73.2} \\
       \cmidrule{2-12}
       
       &  \multicolumn{1}{c}{\multirow{6}[1]{*}{\centering \shortstack{Contrastive \\ Learning }}}
       & CLIP$^\ddagger$~\cite{CLIP}  & ViT-B/16 & WIT  & 32  & \cellcolor[rgb]{ .973,  .773,  .639}38.1  & \cellcolor[rgb]{ .855,  .914,  .965}54.9  & \cellcolor[rgb]{ .855,  .914,  .965}53.4  & \cellcolor[rgb]{ .855,  .914,  .965}56.4  & \cellcolor[rgb]{ .949,  .973,  .933}36.9  & \cellcolor[rgb]{ .949,  .973,  .933}67.7  \\
       & & CLIP +\textbf{\textit{Ours}} & ViT-B/16 & K400 &  +5 & \cellcolor[rgb]{ .961,  .714,  .549}\textbf{39.2} & \cellcolor[rgb]{ .753,  .851,  .937}\textbf{58.3} & \cellcolor[rgb]{ .749,  .847,  .937}\textbf{57.0} & \cellcolor[rgb]{ .769,  .859,  .941}\textbf{59.7} & \cellcolor[rgb]{ .851,  .922,  .808}\textbf{40.6} & \cellcolor[rgb]{ .831,  .91,  .78}\textbf{71.4} \\
       \cmidrule{3-12}

       & & BLIP$^\ddagger$~\cite{BLIP} & ViT-B/16 &  LAION & 20 & \cellcolor[rgb]{ .973,  .796,  .678}37.6  & \cellcolor[rgb]{ .737,  .839,  .933}58.8  & \cellcolor[rgb]{ .741,  .843,  .937}57.3  & \cellcolor[rgb]{ .753,  .851,  .937}60.3  & \cellcolor[rgb]{ .957,  .976,  .945}35.1  & \cellcolor[rgb]{ .957,  .976,  .945}65.9  \\
       & & BLIP +\textbf{\textit{Ours}} & ViT-B/16 & K400 &  +5  & \cellcolor[rgb]{ .961,  .694,  .518}\textbf{39.6} & \cellcolor[rgb]{ .675,  .8,  .918}\textbf{62.0} & \cellcolor[rgb]{ .663,  .796,  .918}\textbf{60.2} & \cellcolor[rgb]{ .675,  .8,  .918}\textbf{63.8} & \cellcolor[rgb]{ .882,  .937,  .851}\textbf{38.9} & \cellcolor[rgb]{ .859,  .925,  .82}\textbf{70.2} \\
       \cmidrule{3-12}
       
       & & MoCo v3$^\ddagger$~\cite{MoCov3}  & ViT-B/16 & INet & 300  & \cellcolor[rgb]{ .965,  .737,  .58}38.8  & \cellcolor[rgb]{ .639,  .776,  .91}62.6  & \cellcolor[rgb]{ .675,  .8,  .918}60.0  & \cellcolor[rgb]{ .592,  .753,  .898}65.1  & \cellcolor[rgb]{ .792,  .89,  .729}43.6  & \cellcolor[rgb]{ .776,  .878,  .706}73.5  \\
       & & MoCo v3 +\textbf{\textit{Ours}} & ViT-B/16 & K400 &  +5 & \cellcolor[rgb]{ .957,  .686,  .498}\textbf{39.8} & \cellcolor[rgb]{ .62,  .765,  .906}\textbf{62.9} & \cellcolor[rgb]{ .639,  .78,  .91}\textbf{60.6} & \cellcolor[rgb]{ .584,  .745,  .898}\textbf{65.2} & \cellcolor[rgb]{ .725,  .851,  .639}\textbf{45.3} & \cellcolor[rgb]{ .663,  .82,  .557}\textbf{75.0} \\
       \cmidrule{2-12}

       &  \multicolumn{1}{c}{\multirow{6}[1]{*}{\centering \shortstack{Self- \\ Distillation }}} 
       & iBOT$^\ddagger$~\cite{iBOT}  & ViT-B/16 & INet & 400  & \cellcolor[rgb]{ .961,  .694,  .518}39.6  & \cellcolor[rgb]{ .51,  .702,  .878}64.6  & \cellcolor[rgb]{ .494,  .69,  .875}63.0  & \cellcolor[rgb]{ .529,  .714,  .882}66.1  & \cellcolor[rgb]{ .702,  .839,  .604}45.7  & \cellcolor[rgb]{ .643,  .808,  .529}75.3  \\
       & & iBOT +\textbf{\textit{Ours}} & ViT-B/16 & K400 &  +5 & \cellcolor[rgb]{ .949,  .631,  .416}\textbf{40.8} & \cellcolor[rgb]{ .475,  .678,  .867}\textbf{65.1} & \cellcolor[rgb]{ .475,  .678,  .867}\textbf{63.3} & \cellcolor[rgb]{ .475,  .678,  .867}\textbf{66.9} & \cellcolor[rgb]{ .678,  .827,  .573}\textbf{46.1} & \cellcolor[rgb]{ .608,  .792,  .482}\textbf{75.8} \\
       \cmidrule{3-12}
       
       & & DINO$^\ddagger$~\cite{DINO}  & ViT-B/16 & INet & 300  & \cellcolor[rgb]{ .965,  .722,  .557}39.1  & \cellcolor[rgb]{ .6,  .753,  .898}63.2  & \cellcolor[rgb]{ .624,  .769,  .906}60.9  & \cellcolor[rgb]{ .569,  .737,  .89}65.5  & \cellcolor[rgb]{ .776,  .878,  .706}44.4  & \cellcolor[rgb]{ .718,  .847,  .627}74.2  \\
       & & DINO +\textbf{\textit{Ours}} & ViT-B/16 & K400 &  +5 & \cellcolor[rgb]{ .957,  .686,  .498}\textbf{39.8} & \cellcolor[rgb]{ .533,  .714,  .882}\textbf{64.2} & \cellcolor[rgb]{ .537,  .718,  .882}\textbf{62.3} & \cellcolor[rgb]{ .533,  .718,  .882}\textbf{66.0} & \cellcolor[rgb]{ .671,  .824,  .565}\textbf{46.2} & \cellcolor[rgb]{ .647,  .812,  .537}\textbf{75.2} \\
       \cmidrule{3-12}
       
       & & DINO v2~\cite{DINOv2}   & ViT-B/16 & INet & 100   & \cellcolor[rgb]{ .969,  .757,  .616}38.4  & \cellcolor[rgb]{ .604,  .757,  .902}63.1  & \cellcolor[rgb]{ .58,  .741,  .894}61.6  & \cellcolor[rgb]{ .631,  .773,  .91}64.5  & \cellcolor[rgb]{ .647,  .812,  .533}46.6  & \cellcolor[rgb]{ .573,  .773,  .439}76.3  \\
       & & DINO v2 +\textbf{\textit{Ours}} & ViT-B/16 &  K400 &  +5 & \cellcolor[rgb]{ .957,  .678,  .49}\textbf{39.9} & \cellcolor[rgb]{ .569,  .733,  .89}\textbf{63.7} & \cellcolor[rgb]{ .561,  .733,  .89}\textbf{61.9} & \cellcolor[rgb]{ .573,  .737,  .894}\textbf{65.4} & \cellcolor[rgb]{ .604,  .788,  .482}\textbf{47.3} & \cellcolor[rgb]{ .537,  .753,  .392}\textbf{76.8} \\

    \bottomrule
    \end{tabular}%
    }
    
  \label{tab:dense_results}%
  \vspace{-6pt}
\end{table*}%

\begin{figure*}[htbp]
  \centering
     \includegraphics[width=0.238\linewidth]{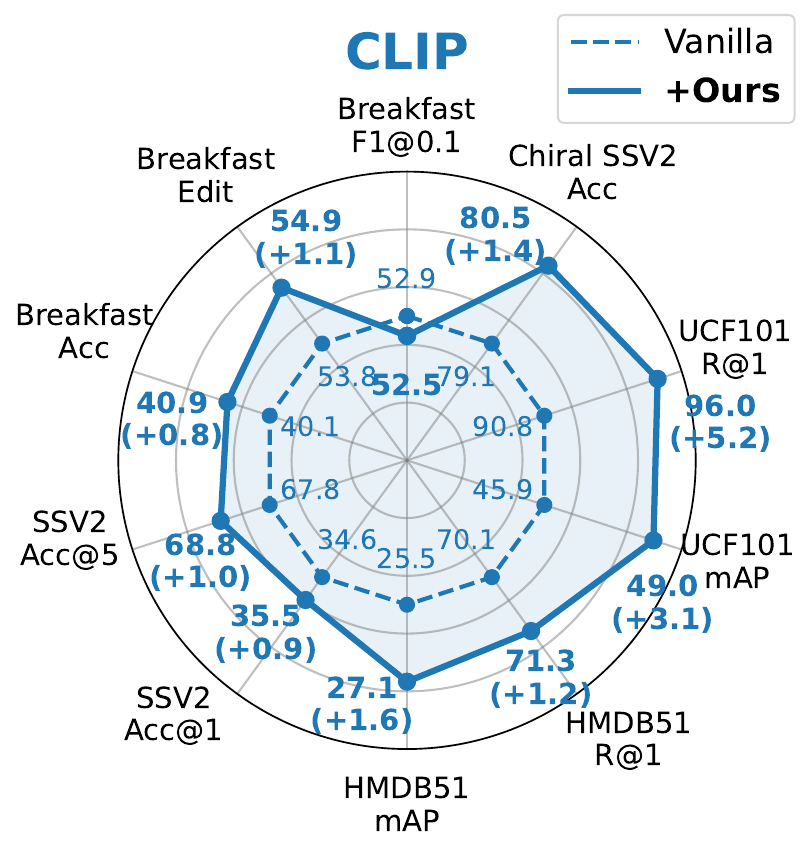}
     \includegraphics[width=0.238\linewidth]{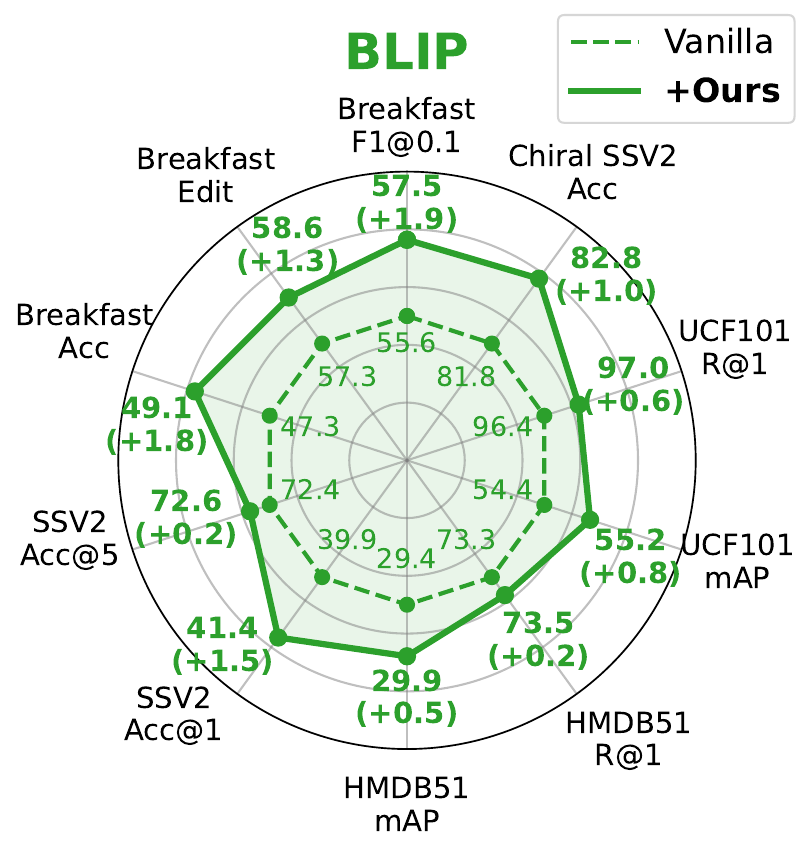}
     \includegraphics[width=0.238\linewidth]{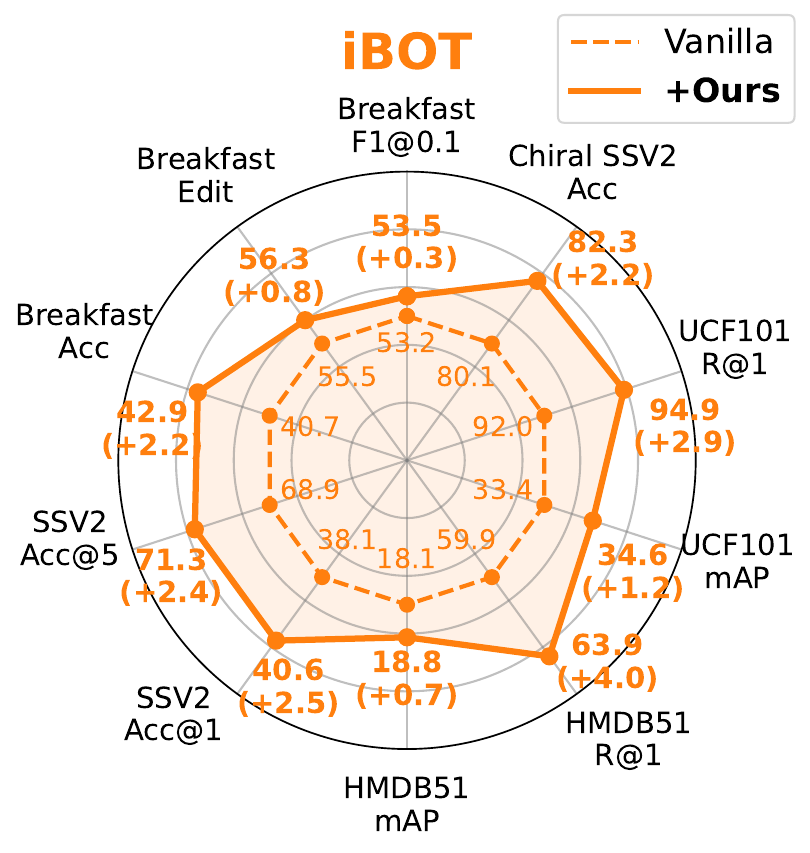}
     \includegraphics[width=0.238\linewidth]{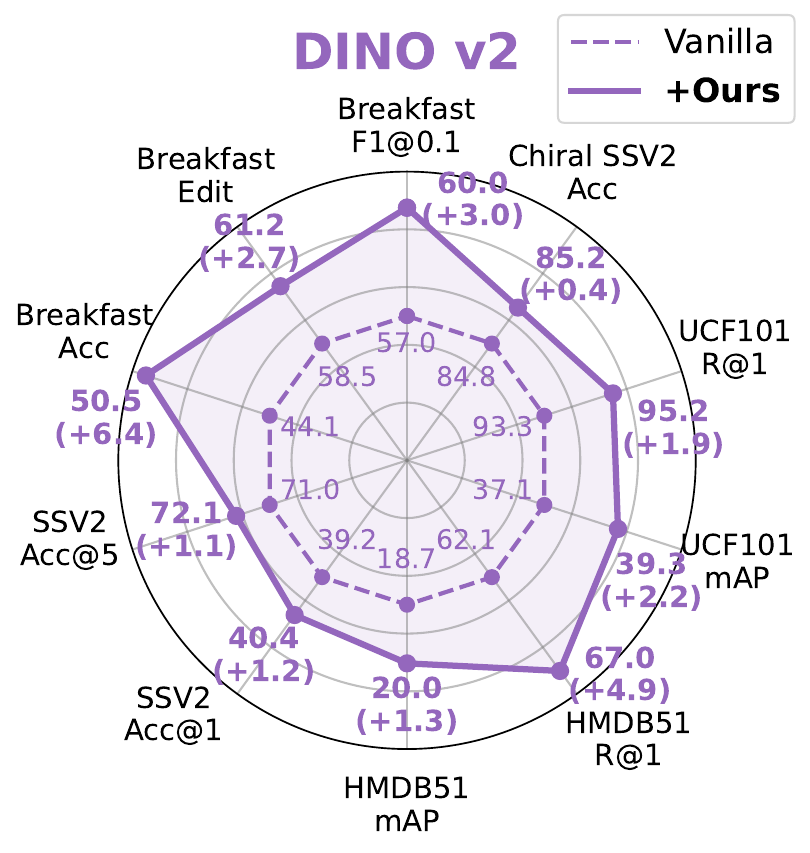}
     \vspace{-6pt}
     \caption{Evaluation results on frame-level and video-level tasks based on four representative image models.}
    \label{fig:frame_video_results}
  \vspace{-14pt}
\end{figure*}

In this section, we first introduce the implementation details in \cref{subsec:implementation_details}. Then, we evaluate the effectiveness and efficiency of our framework on multiple downstream tasks in \cref{subsec:main_results} and justify with distance metrics in \cref{subsec:distance_metrics}.
Finally, we perform ablation studies in \cref{subsec:ablation_study} to assess the impact of key configurations. 
Additional implementation details and results are provided in the \textit{Supplementary Material}.

\subsection{Implementation details}
\label{subsec:implementation_details}

\noindent\textbf{Fundamental image models.}
We evaluate our method across eight pretrained image encoders grouped into three categories: 1) \textit{Masked modeling}: MAE~\cite{MAE}, I-JEPA~\cite{I_JEPA}; 2) \textit{Contrastive learning}: CLIP~\cite{CLIP}, BLIP~\cite{BLIP}, MoCo v3~\cite{MoCov3}; 3) \textit{Self-distillation}: iBOT~\cite{iBOT}, DINO~\cite{DINO}, DINO v2~\cite{DINOv2}.

\noindent\textbf{Optimizing strategies.} 
During training, we only update the projection layer $g$ with feature dimension $d=768$ while the pretrained image encoders are kept frozen. The layer $g$ is trained on the Kinetics-400 dataset for 5 epochs with a total batch size of $512$ on 4 RTX4090 GPUs.
Parameters optimization is performed by AdamW~\cite{AdamW} with a basic learning rate of $1 \times 10^{-4}$ and a cosine decay schedule.

\subsection{Main results on different types of video tasks}
\label{subsec:main_results}

\subsubsection{Evaluation on dense-level benchmarks}
\label{subsubsec:evaluation_dense}
\noindent\textbf{Experiment setup.} 
We evaluate the representations on three dense-level video benchmarks: video object segmentation on DAVIS-2017~\cite{DAVIS17}, human part segmentation on VIP~\cite{VIP}, and pose propagation on JHMDB~\cite{JHMDB}. 
All evaluations are under a zero-shot semi-supervised protocol~\cite{SiamMAE,CropMAE,RSP,tcore}.

\noindent\textbf{Quantitative results.}
The evaluation results on three dense-level benchmarks are reported in \cref{tab:dense_results}, from which we make the following conclusions:
\textbf{1)} Our method consistently improves performance across all eight fundamental image models, demonstrating its effectiveness in self-supervised image-to-video representation transfer. Specifically, the image-pretrained encoders yield average improvements of $1.98\%$ mIoU on VIP, $2.63\%$ $\mathcal{J}\&\mathcal{F}_{\mathrm{m}}$ on DAVIS, and $2.59\%$ PCK@0.1 on JHMDB. 
\textbf{2)} The improvements hold across three categories of image models and, notably, also extend to supervised multimodal models (e.g., CLIP, BLIP), showing broad applicability to ViT-based image encoders under different pretraining paradigms.
\textbf{3)} Under comparable parameter scales, the transferred models match or surpass recent self-supervised video representation models while using simpler modules and lower training overhead, thus offering a lightweight alternative for dense-level video understanding.

\subsubsection{Evaluation on frame-level and video-level tasks}
\label{subsubsec:evaluation_frame_video}

\noindent\textbf{Experiment setup.} 
We evaluate the models on several frame- and video-level tasks: temporal action localization on Breakfast~\cite{Breakfast} with the FACT~\cite{FACT} backbone, zero-shot video retrieval on UCF101 and HMDB51~\cite{UCF101,HMDB51}, action classification on Something-Something-v2 (SSV2)~\cite{ssv2}, and temporal order discrimination on Chiral SSV2~\cite{Chirality}. For SSV2, the models are fine-tuned on the training set for 25 epochs before evaluation on the validation set with single-clip sampling. For Chiral SSV2, we concatenate frame embeddings along the temporal dimension and train a linear probe.

\noindent\textbf{Quantitative results.} 
\cref{fig:frame_video_results} depicts the performance of transferred representations from four representative image models on both frame- and video-level downstream tasks. Our method consistently improves performance across these tasks. For instance, on the frame-level Breakfast task, it achieves an average gain of $2.80\%~Acc$, indicating enhanced temporal awareness in the image models. On video-level tasks, it achieves a $2.58\%~R@1$ improvement on HMDB51, a $1.53\%~Acc@1$ gain on SSV2, and a $1.25\%~Acc$ gain on Chiral SSV2. These results reveal that our method generalizes well across different task granularities, highlighting its potential as a versatile solution for image-to-video transfer.

\begin{table}[t]
  \centering
  \caption{Results of integration into video dense tracking pipeline.}
  \label{tab:dense_tracking}
  \vspace{-8pt}
  \begin{minipage}{0.49\textwidth}
    \centering
    \begin{minipage}{0.31\textwidth}
      \centering
      \setlength{\tabcolsep}{4pt}
      \resizebox{1.0\textwidth}{!}{%
      \begin{tabular}{c|cc}
      \toprule
      \multirow{2}[2]{*}{Features}  & \multicolumn{2}{c}{BADJA} \\
           & $\delta^{seg}$ & $\delta^{3px}$ \\
      \midrule
      DINO v2     & 62.73 & 8.85  \\
      +\textbf{\textit{Ours}}   & \textbf{70.52} & \textbf{8.86} \\
      \midrule
      \midrule
      \multirow{2}[2]{*}{Features} & \multicolumn{2}{c}{TAP-DAVIS}  \\
           & $\delta_{avg}^x$ & OA  \\
      \midrule
      DINO v2    & \textbf{64.68} & 81.98  \\
      +\textbf{\textit{Ours}}    & 64.02 & \textbf{85.40}  \\
      \bottomrule
      \end{tabular}%
      }
    \end{minipage}%
    \hspace{0.02\textwidth}
    \begin{minipage}{0.66\textwidth}
      \includegraphics[width=0.98\linewidth]{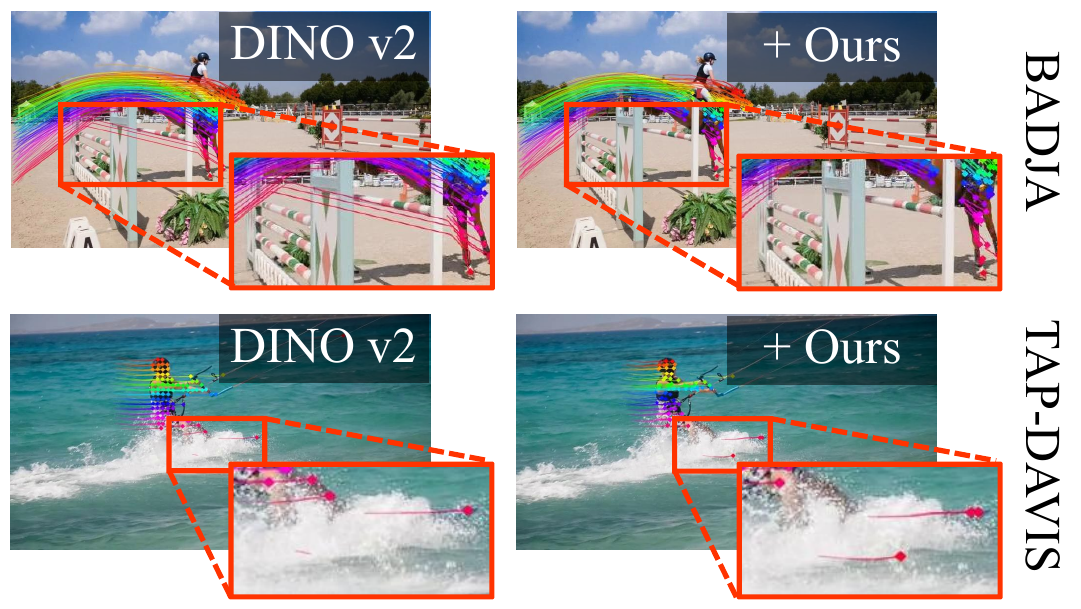}
    \end{minipage}
  \end{minipage}
  \vspace{-6pt}
\end{table}

\begin{table}[t]
  \centering
  \caption{Efficiency and performance comparison with image-to-video transfer methods. 
  $^\dagger$ marks the costs from prior studies.}
  \label{tab:efficiency}
  \vspace{-6pt}

  \setlength{\tabcolsep}{2pt}
  \resizebox{0.96\linewidth}{!}{
  \begin{tabular}{c|ccc|ccc}
  
    \toprule
    \multirow{2}[2]{*}{Method} & \multirow{2}[2]{*}{ \centering \shortstack{Tunable \\ Params $(\downarrow)$}} & \multirow{2}[2]{*}{\centering \shortstack{Peak GPU \\  Mem. $(\downarrow)$}} & \multirow{2}[2]{*}{ \centering \shortstack{ Tuning GPU \\  Hours $(\downarrow)$ }} & \multicolumn{1}{c}{VIP} & \multicolumn{1}{c}{DAVIS17} &  \multicolumn{1}{c}{JHMDB} \\
      &       &       &       & $\mathcal{J}\&\mathcal{F}_{\mathrm{m}}$ & mIoU  & PCK@0.1 \\
    \midrule
    MAE  & /     & /      & / & 29.3  & 52.4 &  41.6   \\
    Full Fine-tuning &    111.66 M   &   13.9 GB     & 20.1$\times$RTX4090       & 29.8  & 53.9  & 42.6   \\
    Partial Fine-tuning &   7.09 M    &    9.2 GB    &    15.6$\times$RTX4090  &  29.2  & 52.7  & 41.3  \\  
    I2V Adapter &  14.22 M   &   20.4 GB    & 21.2$\times$RTX4090      &  29.5  & 53.0 &  41.5   \\
    \textbf{\textit{Ours}} &   0.59 M    &   4.8 GB      & 1.2$\times$RTX4090    & \textbf{33.8}  & \textbf{59.6}  &  \textbf{48.4}   \\
    \midrule
    \midrule
    CLIP  & /     & /       & / & 38.1 & 54.9 &  36.9  \\
    AIM$^\dagger$~\cite{adapter3} &     11.00 M  &    8.7 GB    &   120$\times$V100    & 34.2  & 51.6 &  35.8  \\
    ST-Adapter$^\dagger$~\cite{adapter2} &   7.42 M     &   6.9 GB   &   23$\times$V100   & 36.5 &   54.4  & 37.5\\
    ZeroI2V$^\dagger$~\cite{adapter4} &   14.00 M    &  7.6 GB   &   100$\times$V100   &  37.2  & 54.8  & 37.2 \\
    \textbf{\textit{Ours}} &   0.59 M    &   5.2 GB  & 1.2$\times$RTX4090  & \textbf{39.2} & \textbf{58.3} &  \textbf{40.6}  \\
    \bottomrule
    
  \end{tabular}
  }
  
  \vspace{-14pt}
\end{table}

\subsubsection{Integration into existing video pipeline}
\label{subsubsec:evaluation_complex}
\noindent\textbf{Experiment setup.} 
We further explore the integration of the transferred representations into existing video analysis pipelines. Specifically, for the complex task of dense point tracking, we replace the original DINOv2 features in the DINO-Tracker~\cite{track_4} framework with our transferred DINOv2 representations and evaluate the performance on the BADJA~\cite{BADJA} and TAP-DAVIS~\cite{tap_davis} tracking benchmarks.

\noindent\textbf{Evaluation results.} 
As shown in \cref{tab:dense_tracking}, our method achieves overall improvements on both datasets. The tracking trajectories suggest that our representations provide enhanced spatiotemporal coherence and better occlusion handling. These findings provide preliminary evidence of its potential for integration into existing video processing frameworks for its application in more real-world scenarios.


\begin{table}[t]
  \centering
  \caption{Validation results on distance-based trade-off metrics.}
  \vspace{-6pt}
  \setlength{\tabcolsep}{8pt}
  \resizebox{0.43\textwidth}{!}{%
    \begin{tabular}{cc|ccccc}
    \toprule
    Type & {Method} & {$D_{inter}$} & {$D_{intra}$} & {$D (\uparrow)$} &{\centering \shortstack{Cyc. Acc. $(\uparrow)$}} \\
    \midrule
     \multicolumn{1}{c}{\multirow{3}[1]{*}{\centering \shortstack{Video \\ Pretrained}}}  & SiamMAE & \cellcolor[rgb]{ .906,  .875,  .98}0.5067  & \cellcolor[rgb]{ .945,  .925,  .988}0.1330  & \cellcolor[rgb]{ .886,  .843,  .976}0.4668  & \cellcolor[rgb]{ 1,  .847,  .847}0.4100  \\
    & CropMAE  & \cellcolor[rgb]{ .902,  .867,  .98}0.5216  & \cellcolor[rgb]{ .89,  .851,  .976}0.1736  & \cellcolor[rgb]{ .886,  .843,  .976}0.4695  & \cellcolor[rgb]{ 1,  .702,  .702}0.6522  \\
    &  RSP &  \cellcolor[rgb]{ .914,  .886,  .98}0.4662  & \cellcolor[rgb]{ .839,  .776,  .965}0.2130  & \cellcolor[rgb]{ .922,  .898,  .984}0.4023  & \cellcolor[rgb]{ 1,  .737,  .737}0.5990  \\
    \midrule\midrule
    \multicolumn{1}{c}{\multirow{16}[1]{*}{\centering \shortstack{Image \\ Pretrained \\ +\textbf{\textit{Ours}}}}}  & MAE & \cellcolor[rgb]{ .969,  .961,  .992}0.3122  & \cellcolor[rgb]{ .969,  .961,  .992}0.1131  & \cellcolor[rgb]{ .929,  .914,  .984}0.2783  & \cellcolor[rgb]{ 1,  .937,  .937}0.1366  \\
     & MAE +\textbf{\textit{Ours}} &  
     \cellcolor[rgb]{ .878,  .831,  .973}0.5073  & \cellcolor[rgb]{ .91,  .882,  .984}0.1834  & \cellcolor[rgb]{ .906,  .871,  .98}\textbf{0.4523}  & \cellcolor[rgb]{ 1,  .686,  .686}\textbf{0.7203}  \\
     \cmidrule{2-6}
    & I-JEPA  & \cellcolor[rgb]{ .969,  .961,  .992}0.2572  & \cellcolor[rgb]{ .933,  .91,  .984}0.1425  & \cellcolor[rgb]{ .969,  .961,  .992}0.2145  & \cellcolor[rgb]{ 1,  .937,  .937}0.1192  \\
    & I-JEPA +\textbf{\textit{Ours}}  & \cellcolor[rgb]{ .851,  .792,  .969}0.5904  & \cellcolor[rgb]{ .922,  .898,  .984}0.1745  & \cellcolor[rgb]{ .827,  .757,  .961}\textbf{0.5380}  & \cellcolor[rgb]{ 1,  .741,  .741}\textbf{0.5906}  \\
    \cmidrule{2-6}
    & CLIP  & \cellcolor[rgb]{ .894,  .855,  .976}0.5603  & \cellcolor[rgb]{ .831,  .765,  .961}0.2186  & \cellcolor[rgb]{ .871,  .82,  .973}0.4947  & \cellcolor[rgb]{ 1,  .867,  .867}0.3002  \\
    & CLIP +\textbf{\textit{Ours}}    & \cellcolor[rgb]{ .843,  .78,  .965}0.6162  & \cellcolor[rgb]{ .776,  .69,  .949}0.2626  & \cellcolor[rgb]{ .827,  .761,  .961}\textbf{0.5374}  & \cellcolor[rgb]{ 1,  .82,  .82}\textbf{0.4608}  \\
    \cmidrule{2-6}
    & BLIP   &  \cellcolor[rgb]{ .886,  .843,  .976}0.5858  & \cellcolor[rgb]{ .91,  .875,  .98}0.1598  & \cellcolor[rgb]{ .827,  .761,  .961}\textbf{0.5378}  & \cellcolor[rgb]{ 1,  .937,  .937}0.2245  \\
    & BLIP +\textbf{\textit{Ours}}   &  \cellcolor[rgb]{ .843,  .78,  .965}0.6102  & \cellcolor[rgb]{ .804,  .725,  .957}0.2457  & \cellcolor[rgb]{ .827,  .761,  .961}0.5365  & \cellcolor[rgb]{ 1,  .824,  .824}\textbf{0.4527}  \\
    \cmidrule{2-6}
    & MoCo v3   & \cellcolor[rgb]{ .894,  .855,  .976}0.5547  & \cellcolor[rgb]{ .835,  .773,  .965}0.2164  & \cellcolor[rgb]{ .875,  .824,  .973}0.4898  & \cellcolor[rgb]{ 1,  .867,  .867}0.3770  \\
    & MoCo v3  +\textbf{\textit{Ours}}  &  \cellcolor[rgb]{ .867,  .812,  .973}0.5503  & \cellcolor[rgb]{ .878,  .831,  .973}0.1909  & \cellcolor[rgb]{ .878,  .831,  .973}\textbf{0.4930}  & \cellcolor[rgb]{ 1,  .737,  .737}\textbf{0.5981}  \\
    \cmidrule{2-6}
    & iBOT   & \cellcolor[rgb]{ .878,  .831,  .973}0.6143  & \cellcolor[rgb]{ .875,  .827,  .973}0.1862  & \cellcolor[rgb]{ .831,  .765,  .961}0.5584  & \cellcolor[rgb]{ 1,  .827,  .827}0.4488  \\
    & iBOT +\textbf{\textit{Ours}}  & \cellcolor[rgb]{ .831,  .765,  .961}0.6399  & \cellcolor[rgb]{ .855,  .796,  .969}0.2092  & \cellcolor[rgb]{ .78,  .694,  .953}\textbf{0.5772}  & \cellcolor[rgb]{ 1,  .753,  .753}\textbf{0.5731}  \\
    \cmidrule{2-6}
    & DINO  & \cellcolor[rgb]{ .89,  .847,  .976}0.5756  & \cellcolor[rgb]{ .839,  .776,  .965}0.2144  & \cellcolor[rgb]{ .859,  .808,  .969}0.5112  & \cellcolor[rgb]{ 1,  .82,  .82}0.4590  \\
    & DINO +\textbf{\textit{Ours}}   & \cellcolor[rgb]{ .839,  .773,  .965}0.6246  & \cellcolor[rgb]{ .824,  .753,  .961}0.2316  & \cellcolor[rgb]{ .808,  .729,  .957}\textbf{0.5551} & \cellcolor[rgb]{ 1,  .741,  .741}\textbf{0.5927}  \\
    \cmidrule{2-6}
    & DINO v2   &  \cellcolor[rgb]{ .886,  .843,  .976}0.5926  & \cellcolor[rgb]{ .882,  .835,  .973}0.1808  & \cellcolor[rgb]{ .843,  .784,  .965}0.5384  & \cellcolor[rgb]{ 1,  .839,  .839}0.4276  \\
    & DINO v2 +\textbf{\textit{Ours}}  & \cellcolor[rgb]{ .835,  .769,  .965}0.6373  & \cellcolor[rgb]{ .871,  .82,  .973}0.1976  & \cellcolor[rgb]{ .776,  .69,  .949}\textbf{0.5780}  & \cellcolor[rgb]{ 1,  .773,  .773}\textbf{0.5383}  \\
    \bottomrule
    \end{tabular}%
    }
  \label{tab:distance_metrics}%
  \vspace{-12pt}
\end{table}%

\begin{table*}[t]
  \centering
  \caption{Ablation results on the components, structures, and training configurations of our method. The best and secondary best results are highlighted in \textbf{bold} and \underline{underlined}. Default settings are marked with \colorbox[rgb]{ .867,  .922,  .969}{blue}.}
  \vspace{-16pt}
  \begin{tabular}{cc}
  
    \begin{subtable}[t]{0.47\textwidth}
      \centering
      \caption{Ablation on components of $\mathcal{L}_{cyc}$, $\mathcal{L}_{reg}$ and PEA.}
      \setlength{\tabcolsep}{8pt}
      \resizebox{\textwidth}{!}{%
        \begin{tabular}{cccc|ccc}
        \toprule
        \multirow{2}{*}{\centering \shortstack{Base \\ Model}} & \multirow{2}{*}{$\mathcal{L}_{cyc}$}  & \multirow{2}{*}{$\mathcal{L}_{reg}$}  & \multirow{2}{*}{PEA} & \multicolumn{1}{c}{VIP} & \multicolumn{1}{c}{DAVIS17} &  \multicolumn{1}{c}{JHMDB} \\
         &   &   &  &   mIoU & $\mathcal{J}\&\mathcal{F}_{\mathrm{m}}$ & PCK@0.1 \\
        \midrule
        \multirow{4}{*}{MAE}    & \ding{51}  &   &    &  16.2    &   26.2   &     38.5   \\  
                   & \ding{51}    & \ding{51}  &    &   23.1    &  42.3   &    42.4 \\
                & \ding{51}    &    & \ding{51}     & \underline{33.3}    & \underline{59.3}     &    \underline{48.1} \\
              &  \cellcolor[rgb]{ .867,  .922,  .969}\ding{51}    & \cellcolor[rgb]{ .867,  .922,  .969}\ding{51}    & \cellcolor[rgb]{ .867,  .922,  .969}\ding{51}    & \cellcolor[rgb]{ .867,  .922,  .969}\textbf{33.8}   &    \cellcolor[rgb]{ .867,  .922,  .969}\textbf{59.6}   &   \cellcolor[rgb]{ .867,  .922,  .969}\textbf{48.4}   \\
        \midrule
        \multirow{4}{*}{DINO}    & \ding{51}  &  &    &  17.5  &  30.6  & 39.3   \\
               & \ding{51}    & \ding{51}  &        &  33.6    &  58.9   &   \textbf{46.4} \\
                 & \ding{51}   &  & \ding{51}  &    \underline{38.0}    &   \underline{61.8}   &    46.1  \\
              & \cellcolor[rgb]{ .867,  .922,  .969}\ding{51}    & \cellcolor[rgb]{ .867,  .922,  .969}\ding{51}    &  \cellcolor[rgb]{ .867,  .922,  .969}\ding{51}    &  \cellcolor[rgb]{ .867,  .922,  .969}\textbf{39.8}    &  \cellcolor[rgb]{ .867,  .922,  .969}\textbf{64.2}    &    \cellcolor[rgb]{ .867,  .922,  .969}\underline{46.2} \\
        \bottomrule
        \end{tabular}

      }
      \label{tab:ablation_components}
    \end{subtable}

    \begin{subtable}[t]{0.48\textwidth}
      \centering
      \caption{Ablation on the projection layer structures.}
      \setlength{\tabcolsep}{10pt}
      \resizebox{\textwidth}{!}{%
        \begin{tabular}{cc|ccc}
        \toprule
        \multirow{2}{*}{\centering \shortstack{Base \\ Model}} & \multirow{2}{*}{\centering \shortstack{Projection \\ Structure}} & \multicolumn{1}{c}{VIP} & \multicolumn{1}{c}{DAVIS17} &  JHMDB \\
              &     &  mIoU & $\mathcal{J}\&\mathcal{F}_{\mathrm{m}}$ & PCK@0.1 \\
        \midrule
        \multirow{4}{*}{MAE} & Vanilla &  29.3    & 52.4    &    41.6  \\
              & \cellcolor[rgb]{ .867,  .922,  .969}Linear-based layer &   
              \cellcolor[rgb]{ .867,  .922,  .969}\textbf{33.8}  & \cellcolor[rgb]{ .867,  .922,  .969}\textbf{59.6}  &    \cellcolor[rgb]{ .867,  .922,  .969}\textbf{47.9} \\
              & MLP (2 layers) &  \underline{33.2}    &   \underline{59.2}   &   \underline{47.5}  \\
              & MLP (3 layers) & 32.6   &  58.4    &     47.4  \\
        \midrule
        \multirow{4}{*}{DINO} & Vanilla &  39.1    &   63.2   &    44.4  \\
              & \cellcolor[rgb]{ .867,  .922,  .969}Linear-based layer &    \cellcolor[rgb]{ .867,  .922,  .969}\underline{39.8}   &  \cellcolor[rgb]{ .867,  .922,  .969}\textbf{64.2}    &  \cellcolor[rgb]{ .867,  .922,  .969}\textbf{46.2} \\
              & MLP (2 layers) &  39.7    &  \underline{64.1}   &     \underline{45.9} \\
              & MLP (3 layers) &  \textbf{39.9}    &  63.8    & 45.7  \\
        \bottomrule
        \end{tabular}
      }
      \label{tab:ablation_projection}
    \end{subtable}

    \\

    \begin{subtable}[t]{0.305\textwidth}
      \centering
      \vspace{-4pt}
      \caption{Ablation on training dataset.}
      \setlength{\tabcolsep}{4pt}
      \resizebox{\textwidth}{!}{%
        \begin{tabular}{cc|ccc}
        \toprule
        \multirow{2}{*}{\centering \shortstack{Base \\ Model}} & \multirow{2}{*}{\centering \shortstack{Training \\ Dataset}} & \multicolumn{1}{c}{VIP}  & \multicolumn{1}{c}{DAVIS17} &  \multicolumn{1}{c}{JHMDB} \\
         &  & mIoU & $\mathcal{J}\&\mathcal{F}_{\mathrm{m}}$ &  PCK@0.1 \\
        \midrule
        \multirow{2}{*}{MAE} &  \cellcolor[rgb]{ .867,  .922,  .969}K400  &  \cellcolor[rgb]{ .867,  .922,  .969}33.8    & \cellcolor[rgb]{ .867,  .922,  .969}59.6    &     \cellcolor[rgb]{ .867,  .922,  .969}48.4  \\  
        &  SSV2  &  33.6    &  60.4   &    48.2 \\
        \midrule
        \multirow{2}{*}{BLIP}  &  \cellcolor[rgb]{ .867,  .922,  .969}K400  & \cellcolor[rgb]{ .867,  .922,  .969}39.6   &   \cellcolor[rgb]{ .867,  .922,  .969}62.0   &   \cellcolor[rgb]{ .867,  .922,  .969}38.9  \\  
        &  SSV2  &  38.2   & 60.0    &    37.7  \\
        \midrule
        \multirow{2}{*}{DINO}  &  \cellcolor[rgb]{ .867,  .922,  .969}K400  &   \cellcolor[rgb]{ .867,  .922,  .969}39.8   &    \cellcolor[rgb]{ .867,  .922,  .969}64.2   & \cellcolor[rgb]{ .867,  .922,  .969}46.2  \\  
        &  SSV2  &   39.4   & 63.8    &   45.8  \\
        \bottomrule
        \end{tabular}
        }
      \label{tab:ablation_training_dataset}
    \end{subtable}

    \begin{subtable}[t]{0.32\textwidth}
      \centering
      \vspace{-4pt}
      \caption{Ablation on the backbone scales.}
      \setlength{\tabcolsep}{4pt}
            \resizebox{\textwidth}{!}{%
        \begin{tabular}{cc|ccc}
        \toprule
        \multirow{2}{*}{\centering \shortstack{Backbone \\ Model}} & \multirow{2}{*}{\centering \shortstack{Method}}  &  \multicolumn{1}{c}{VIP} & \multicolumn{1}{c}{DAVIS17} & \multicolumn{1}{c}{JHMDB} \\
         &  & mIoU & $\mathcal{J}\&\mathcal{F}_{\mathrm{m}}$ & PCK@0.1 \\
        \midrule
        \multirow{2}{*}{ViT-B/16}  &  MAE  & 29.3   & 52.4    &      41.6   \\  
        &  \cellcolor[rgb]{ .867,  .922,  .969}+\textbf{\textit{Ours}}  &   \cellcolor[rgb]{ .867,  .922,  .969}33.8   &  \cellcolor[rgb]{ .867,  .922,  .969}59.6    &   \cellcolor[rgb]{ .867,  .922,  .969}48.4    \\  
        \midrule
        \multirow{2}{*}{ViT-L/16}  &  MAE   &  29.9   & 55.8   &     44.6\\
        &  +\textbf{\textit{Ours}}  & 33.4    &   59.9    &  48.9    \\  
        \midrule
        \multirow{2}{*}{ViT-H/14}  &  MAE  &  29.5  & 55.8   &    / \\
        &  +\textbf{\textit{Ours}}  & 33.4    &  60.1    &    /  \\  
        \bottomrule
        \end{tabular}
      }
      \label{tab:ablation_backbone_scales}
    \end{subtable}

    \begin{subtable}[t]{0.318\textwidth}
      \centering
      \vspace{-4pt}
      \caption{Ablation on the training epochs.}
      \setlength{\tabcolsep}{4pt}
      \resizebox{\textwidth}{!}{%
        \begin{tabular}{cc|ccc}
        \toprule
        \multirow{2}{*}{\centering \shortstack{Base \\ Model}} & \multirow{2}{*}{\centering \shortstack{Training \\ Epochs}}  & \multicolumn{1}{c}{VIP} &  \multicolumn{1}{c}{DAVIS17} & \multicolumn{1}{c}{JHMDB} \\
         &  & mIoU & $\mathcal{J}\&\mathcal{F}_{\mathrm{m}}$ &  PCK@0.1 \\
        \midrule
        \multirow{3}{*}{MAE}   &  \cellcolor[rgb]{ .867,  .922,  .969}5  &    \cellcolor[rgb]{ .867,  .922,  .969}33.8    &   \cellcolor[rgb]{ .867,  .922,  .969}59.6   &    \cellcolor[rgb]{ .867,  .922,  .969}48.4   \\  
                  &  10  & 33.6   & 59.4  &  48.3 \\
                  &  20  &  33.5   & 59.1  &   48.1 \\
        \midrule
        \multirow{3}{*}{DINO}   &  \cellcolor[rgb]{ .867,  .922,  .969}5  & \cellcolor[rgb]{ .867,  .922,  .969}39.8    &   \cellcolor[rgb]{ .867,  .922,  .969}64.2  &      \cellcolor[rgb]{ .867,  .922,  .969}46.2  \\  
                  &  10  & 39.9   &  64.1   &     46.3 \\
                  &  20  & 40.0   & 64.0  &   46.3  \\
        \bottomrule
        \end{tabular}
      }
      \label{tab:ablation_training_epochs}
    \end{subtable}

  \end{tabular}
  \vspace{-15pt}
\end{table*}

\subsubsection{Efficiency analysis}
\label{subsubsec:efficiency_analysis}

\noindent\textbf{Experiment setup.}
We compare our method with several image-to-video transfer methods, all of which involve post-training or fine-tuning on Kinetics-400. We first compare with three common strategies: full fine-tuning, partial fine-tuning (i.e., updating only the final Transformer block), and the I2V adapter~\cite{adapter2} applied to each block. Then we compare our method with CLIP-based supervised adaptation methods.

\noindent\textbf{Evaluation results.} 
As shown in \cref{tab:efficiency}, the baseline methods yield suboptimal performance, primarily due to degraded discriminability caused by the limited semantic diversity of the video dataset.
In contrast, our method delivers superior results with a $\sim13\times$ speed-up while updating only $0.59~M$ parameters. 
Additionally, our method outperforms CLIP-based adaptive baselines with lower training cost. These baselines emphasize global semantic separability but lack explicit temporal correspondence for dense video understanding.
Overall, on dense-level tasks, our method delivers improved performance over common image-to-video transfer baselines while maintaining high computational efficiency. These outcomes also indicate that the performance gain chiefly arises from a more favorable trade-off between intra-video temporal consistency and inter-video semantic separability instead of image-to-video domain adaptation.

\subsection{Distance-based trade-off metrics validation}
\label{subsec:distance_metrics}

\noindent\textbf{Experiment setup.} 
To provide an interpretable assessment of the method’s effectiveness, we validate the distance-based metrics proposed in \cref{sec:theory}. Specifically, we randomly sample 1000 videos from the Kinetics-400 validation set and measure four metrics for the original image-pretrained models and our transferred models. The metrics includes: \textbf{1)} Inter-video distance $D_{inter}={D_{inter}^{ori}}/{2R_{inter}}$ (normalized by its diameter); \textbf{2)} Intra-video distance $D_{intra}={D_{intra}^{ori}}/{2R_{intra}}$ (normalized by its diameter); \textbf{3)} Distance margin $D=D_{inter}-\gamma D_{intra}$ where the scale factor $\gamma={\mathbb{E}_{\mathcal{M}}\left[{D_{intra}^{ori}}\right]}/{\mathbb{E}_{\mathcal{M}}\left[{D_{inter}^{ori}}\right]}$ is the average ratio of the original intra-/inter-video distance for each model $\mathcal{M}$; \textbf{4)} Cycle consistency accuracy ($Cyc. Acc.$), defined as the proportion of patches that return to their original positions under a palindrome sequence constructed from two frames.

\noindent\textbf{Evaluation results.} 
As shown in \cref{tab:distance_metrics}, our method generally increases the margin $D$ by widening the gap between inter- and intra-video distances. This yields a better trade-off between intra-video temporal consistency and inter-video semantic separability, supporting the conclusion of \cref{thm:thm2_main}. In addition, our method significantly improves cycle consistency accuracy to a level comparable to video-pretrained models, indicating the effectiveness of introducing dense temporal correspondence into image-pretrained models.

\subsection{Ablation Study}
\label{subsec:ablation_study}

In this part, we conduct ablation studies to assess the contribution of each component in our method.
All settings are kept consistent across variants, except for the ablated factors.

\noindent\textbf{Analysis of components.}
We first examine the effect of cycle consistency loss, regularization loss, and PEA strategy. 
As shown in \cref{tab:ablation_components}, applying cycle consistency loss without the PEA strategy leads to obvious performance degradations, as the model exploits positional shortcuts (\textit{line 1-2}).  
With the PEA strategy enabled (\textit{line 3}), the model learns meaningful temporal correspondences, showing that cycle consistency becomes effective only when shortcuts are suppressed.  
Additionally, incorporating the regularization term (\textit{line 4}) yields further improvements, highlighting the importance of preserving semantic separability during transfer learning.

\noindent\textbf{Effect of projection structure.}
We compare different designs of the projection structure, including a linear-based layer and MLPs with 2 or 3 layers in \cref{tab:ablation_projection}.
Empirically, the linear-based layer achieves equal or superior performance relative to the deeper MLPs, likely because more complex projections are prone to perturbing the semantic structure of image-pretrained representations.
These results indicate that a simple linear projection is often sufficient, which is consistent with the theoretical analysis in \cref{thm:thm1_main}.

\noindent\textbf{Generalization on training configurations.}
We evaluate the generalization ability of our method under different training configurations as shown in \cref{tab:ablation_training_dataset,tab:ablation_backbone_scales,tab:ablation_training_epochs}.
In addition to K400, training on the SSV2 dataset yields consistent improvements, demonstrating the robustness of the proposed method on datasets with stronger temporal dynamics.
Scaling up the backbone from ViT-B to ViT-L and ViT-H still enhances downstream performance, indicating its adaptability with larger models.
Moreover, our model converges within 5 epochs, which is selected as the default setting.

\section{Conclusion}
\label{sec:conclusion}

This work explores self-supervised image-to-video transfer learning for an effective trade-off between intra-video temporal consistency and inter-video semantic separability. 
We propose \methodname{} framework to project the representation space via a lightweight layer and provide a theoretical analysis.
Experimental results with eight image models present the effectiveness of \methodname{} across multiple video tasks.

\section*{Acknowledgments}
This work was supported in part by National Natural Science Foundation of China: 62525212, U23B2051, 62236008, 62441232, 62521007 and U21B2038, in part by Youth Innovation Promotion Association CAS, in part by the Strategic Priority Research Program of the Chinese Academy of Sciences, Grant No. XDB0680201, in part by the project ZR2025ZD01 supported by Shandong Provincial Natural Science Foundation, in part by the China National Postdoctoral Program for Innovative Talents under Grant BX20250377, and in part by the Beijing Major Science and Technology Project under Contract No. Z251100008125059. This work was supported by Beijing Academy of Artificial Intelligence (BAAI).

{
    \small
    \bibliographystyle{ieeenat_fullname}
    \bibliography{main}

@String(PAMI = {IEEE Transactions on Pattern Analysis and Machine Intelligence})

@String(CVPR = {IEEE/CVF Conference on Computer Vision and Pattern Recognition})

@String(ICCV = {International Conference on Computer Vision})

@String(ECCV = {European Conference on Computer Vision})

@String(NIPS = {Advances in Neural Information Processing Systems})

@String(ICML = {International Conference on Machine Learning})

@String(ACMMM = {ACM International Conference on Multimedia})

@String(ACCV = {Asian Conference on Computer Vision})

@String(ICLR = {International Conference on Learning Representations})

@inproceedings{volume_1,
  title={Motion segmentation and tracking using normalized cuts},
  author={Shi, Jianbo and Malik, Jitendra},
  booktitle=ICCV,
  pages={1154--1160},
  year={1998},
  organization={IEEE}
}

@inproceedings{volume_2,
  title={Event-based analysis of video},
  author={Zelnik-Manor, Lihi and Irani, Michal},
  booktitle=CVPR,
  volume={2},
  pages={II--II},
  year={2001},
  organization={IEEE}
}

@inproceedings{MoCov1,
  title={Momentum contrast for unsupervised visual representation learning},
  author={He, Kaiming and Fan, Haoqi and Wu, Yuxin and Xie, Saining and Girshick, Ross},
  booktitle=CVPR,
  pages={9729--9738},
  year={2020}
}

@inproceedings{MoCov3,
  title={An empirical study of training self-supervised vision transformers},
  author={Chen, Xinlei and Xie, Saining and He, Kaiming},
  booktitle=ICCV,
  pages={9640--9649},
  year={2021}
}

@inproceedings{SimCLRv1,
  title={A simple framework for contrastive learning of visual representations},
  author={Chen, Ting and Kornblith, Simon and Norouzi, Mohammad and Hinton, Geoffrey},
  booktitle=ICML,
  pages={1597--1607},
  year={2020},
  organization={PMLR}
}

@article{BYOL,
  title={Bootstrap your own latent-a new approach to self-supervised learning},
  author={Grill, Jean-Bastien and Strub, Florian and Altch{\'e}, Florent and Tallec, Corentin and Richemond, Pierre and Buchatskaya, Elena and Doersch, Carl and Avila Pires, Bernardo and Guo, Zhaohan and Gheshlaghi Azar, Mohammad and others},
  journal=NIPS,
  volume={33},
  pages={21271--21284},
  year={2020}
}

@article{SwAV,
  title={Unsupervised learning of visual features by contrasting cluster assignments},
  author={Caron, Mathilde and Misra, Ishan and Mairal, Julien and Goyal, Priya and Bojanowski, Piotr and Joulin, Armand},
  journal=NIPS,
  volume={33},
  pages={9912--9924},
  year={2020}
}

@inproceedings{SimSiam,
  title={Exploring simple siamese representation learning},
  author={Chen, Xinlei and He, Kaiming},
  booktitle=CVPR,
  pages={15750--15758},
  year={2021}
}

@inproceedings{huangtowards,
  title={Towards the Generalization of Contrastive Self-Supervised Learning},
  author={Huang, Weiran and Yi, Mingyang and Zhao, Xuyang and Jiang, Zihao},
  booktitle=ICLR,
  year={2023}
}

@inproceedings{SlowFast,
  title={A large-scale study on unsupervised spatiotemporal representation learning},
  author={Feichtenhofer, Christoph and Fan, Haoqi and Xiong, Bo and Girshick, Ross and He, Kaiming},
  booktitle=CVPR,
  pages={3299--3309},
  year={2021}
}

@inproceedings{SimMIM,
  title={Simmim: A simple framework for masked image modeling},
  author={Xie, Zhenda and Zhang, Zheng and Cao, Yue and Lin, Yutong and Bao, Jianmin and Yao, Zhuliang and Dai, Qi and Hu, Han},
  booktitle=CVPR,
  pages={9653--9663},
  year={2022}
}

@inproceedings{I_JEPA,
  title={Self-supervised learning from images with a joint-embedding predictive architecture},
  author={Assran, Mahmoud and Duval, Quentin and Misra, Ishan and Bojanowski, Piotr and Vincent, Pascal and Rabbat, Michael and LeCun, Yann and Ballas, Nicolas},
  booktitle=CVPR,
  pages={15619--15629},
  year={2023}
}

@inproceedings{StoP,
  title={Stochastic positional embeddings improve masked image modeling},
  author={Bar, Amir and Bordes, Florian and Shocher, Assaf and Assran, Mido and Vincent, Pascal and Ballas, Nicolas and Darrell, Trevor and Globerson, Amir and LeCun, Yann},
  booktitle=ICML,
  year={2024}
}

@article{BEiTv1,
  title={Beit: Bert pre-training of image transformers},
  author={Bao, Hangbo and Dong, Li and Piao, Songhao and Wei, Furu},
  booktitle=ICLR,
  year={2022}
}

@article{BEiTv2,
  title={Beit v2: Masked image modeling with vector-quantized visual tokenizers},
  author={Peng, Zhiliang and Dong, Li and Bao, Hangbo and Ye, Qixiang and Wei, Furu},
  journal={arXiv preprint arXiv:2208.06366},
  year={2022}
}

@inproceedings{TimeSformer,
  title={Is space-time attention all you need for video understanding?},
  author={Bertasius, Gedas and Wang, Heng and Torresani, Lorenzo},
  booktitle=ICML,
  year={2021}
}

@inproceedings{ViViT,
  title={Vivit: A video vision transformer},
  author={Arnab, Anurag and Dehghani, Mostafa and Heigold, Georg and Sun, Chen and Lu{\v{c}}i{\'c}, Mario and Schmid, Cordelia},
  booktitle=ICCV,
  pages={6836--6846},
  year={2021}
}

@article{XViT,
  title={Space-time mixing attention for video transformer},
  author={Bulat, Adrian and Perez Rua, Juan Manuel and Sudhakaran, Swathikiran and Martinez, Brais and Tzimiropoulos, Georgios},
  journal=NIPS,
  volume={34},
  pages={19594--19607},
  year={2021}
}

@article{MAE_ST,
  title={Masked autoencoders as spatiotemporal learners},
  author={Feichtenhofer, Christoph and Li, Yanghao and He, Kaiming and others},
  journal=NIPS,
  volume={35},
  pages={35946--35958},
  year={2022}
}

@article{VideoMAE,
  title={Videomae: Masked autoencoders are data-efficient learners for self-supervised video pre-training},
  author={Tong, Zhan and Song, Yibing and Wang, Jue and Wang, Limin},
  journal=NIPS,
  volume={35},
  pages={10078--10093},
  year={2022}
}

@inproceedings{VideoMAEv2,
  title={Videomae v2: Scaling video masked autoencoders with dual masking},
  author={Wang, Limin and Huang, Bingkun and Zhao, Zhiyu and Tong, Zhan and He, Yinan and Wang, Yi and Wang, Yali and Qiao, Yu},
  booktitle=CVPR,
  pages={14549--14560},
  year={2023}
}

@inproceedings{MaskViT,
  title={MaskViT: Masked Visual Pre-Training for Video Prediction},
  author={Gupta, Agrim and Tian, Stephen and Zhang, Yunzhi and Wu, Jiajun and Mart{\'\i}n-Mart{\'\i}n, Roberto and Fei-Fei, Li},
  booktitle=ICLR,
  year={2023}
}

@inproceedings{DropMAE,
  title={Dropmae: Masked autoencoders with spatial-attention dropout for tracking tasks},
  author={Wu, Qiangqiang and Yang, Tianyu and Liu, Ziquan and Wu, Baoyuan and Shan, Ying and Chan, Antoni B},
  booktitle=CVPR,
  pages={14561--14571},
  year={2023}
}

@article{STP,
  title={Spatiotemporal Predictive Pre-training for Robotic Motor Control},
  author={Yang, Jiange and Liu, Bei and Fu, Jianlong and Pan, Bocheng and Wu, Gangshan and Wang, Limin},
  journal={arXiv preprint arXiv:2403.05304},
  year={2024}
}

@inproceedings{SiameseIM,
  title={Siamese image modeling for self-supervised vision representation learning},
  author={Tao, Chenxin and Zhu, Xizhou and Su, Weijie and Huang, Gao and Li, Bin and Zhou, Jie and Qiao, Yu and Wang, Xiaogang and Dai, Jifeng},
  booktitle=CVPR,
  pages={2132--2141},
  year={2023}
}

@inproceedings{VideoMAC,
  title={VideoMAC: Video Masked Autoencoders Meet ConvNets},
  author={Pei, Gensheng and Chen, Tao and Jiang, Xiruo and Liu, Huafeng and Sun, Zeren and Yao, Yazhou},
  booktitle=CVPR,
  pages={22733--22743},
  year={2024}
}

@inproceedings{centering,
  title={Weighted Ensemble Self-Supervised Learning},
  author={Ruan, Yangjun and Singh, Saurabh and Morningstar, Warren Richard and Alemi, Alexander A and Ioffe, Sergey and Fischer, Ian and Dillon, Joshua V},
  booktitle=ICLR,
  year={2023}
}

@inproceedings{KoLeo,
  title={Spreading vectors for similarity search},
  author={Sablayrolles, Alexandre and Douze, Matthijs and Schmid, Cordelia and J{\'e}gou, Herv{\'e}},
  booktitle=ICLR,
  year={2019}
}

@article{adapting_resolution,
  title={Fixing the train-test resolution discrepancy},
  author={Touvron, Hugo and Vedaldi, Andrea and Douze, Matthijs and J{\'e}gou, Herv{\'e}},
  journal=NIPS,
  volume={32},
  year={2019}
}

@article{contrast1,
  title={Self-supervised co-training for video representation learning},
  author={Han, Tengda and Xie, Weidi and Zisserman, Andrew},
  journal=NIPS,
  volume={33},
  pages={5679--5690},
  year={2020}
}

@article{contrast2,
  title={Compressed video contrastive learning},
  author={Huo, Yuqi and Ding, Mingyu and Lu, Haoyu and Fei, Nanyi and Lu, Zhiwu and Wen, Ji-Rong and Luo, Ping},
  journal=NIPS,
  volume={34},
  pages={14176--14187},
  year={2021}
}

@inproceedings{contrast3,
  title={Self-supervised learning of compressed video representations},
  author={Yu, Youngjae and Lee, Sangho and Kim, Gunhee and Song, Yale},
  booktitle=ICLR,
  year={2020}
}

@inproceedings{forgetting1,
  title={FROSTER: Frozen CLIP is A Strong Teacher for Open-Vocabulary Action Recognition},
  author={Huang, Xiaohu and Zhou, Hao and Yao, Kun and Han, Kai},
  booktitle=ICLR,
  year={2024},
}

@inproceedings{forgetting2,
  title={Clipood: Generalizing clip to out-of-distributions},
  author={Shu, Yang and Guo, Xingzhuo and Wu, Jialong and Wang, Ximei and Wang, Jianmin and Long, Mingsheng},
  booktitle=ICML,
  pages={31716--31731},
  year={2023},
  organization={PMLR}
}

@inproceedings{forgetting3,
  title={Fine-tuned clip models are efficient video learners},
  author={Rasheed, Hanoona and Khattak, Muhammad Uzair and Maaz, Muhammad and Khan, Salman and Khan, Fahad Shahbaz},
  booktitle=CVPR,
  pages={6545--6554},
  year={2023}
}

@inproceedings{collapse1,
  title={Rankme: Assessing the downstream performance of pretrained self-supervised representations by their rank},
  author={Garrido, Quentin and Balestriero, Randall and Najman, Laurent and Lecun, Yann},
  booktitle=ICML,
  pages={10929--10974},
  year={2023},
  organization={PMLR}
}

@article{collapse2,
  title={Towards a unified theoretical understanding of non-contrastive learning via rank differential mechanism},
  author={Zhuo, Zhijian and Wang, Yifei and Ma, Jinwen and Wang, Yisen},
  journal=ICLR,
  year={2023}
}

@inproceedings{collapse3,
  title={On the importance of asymmetry for siamese representation learning},
  author={Wang, Xiao and Fan, Haoqi and Tian, Yuandong and Kihara, Daisuke and Chen, Xinlei},
  booktitle=CVPR,
  pages={16570--16579},
  year={2022}
}

@article{cluster,
  title={The hidden uniform cluster prior in self-supervised learning},
  author={Assran, Mahmoud and Balestriero, Randall and Duval, Quentin and Bordes, Florian and Misra, Ishan and Bojanowski, Piotr and Vincent, Pascal and Rabbat, Michael and Ballas, Nicolas},
  journal=ICLR,
  year={2023}
}

@inproceedings{two_stage_1,
  title={Spatial-then-temporal self-supervised learning for video correspondence},
  author={Li, Rui and Liu, Dong},
  booktitle=CVPR,
  pages={2279--2288},
  year={2023}
}

@inproceedings{two_stage_2,
  title={Semantic-aware fine-grained correspondence},
  author={Hu, Yingdong and Wang, Renhao and Zhang, Kaifeng and Gao, Yang},
  booktitle=ECCV,
  pages={97--115},
  year={2022},
  organization={Springer}
}

@inproceedings{two_stage_3,
  title={Learning fine-grained features for pixel-wise video correspondences},
  author={Li, Rui and Zhou, Shenglong and Liu, Dong},
  booktitle=ICCV,
  pages={9632--9641},
  year={2023}
}

@inproceedings{two_stage_4,
  title={Bootstap: Bootstrapped training for tracking-any-point},
  author={Doersch, Carl and Luc, Pauline and Yang, Yi and Gokay, Dilara and Koppula, Skanda and Gupta, Ankush and Heyward, Joseph and Rocco, Ignacio and Goroshin, Ross and Carreira, Jo{\~a}o and others},
  booktitle=ACCV,
  pages={3257--3274},
  year={2024}
}

@article{adapter1,
  title={Adaptformer: Adapting vision transformers for scalable visual recognition},
  author={Chen, Shoufa and Ge, Chongjian and Tong, Zhan and Wang, Jiangliu and Song, Yibing and Wang, Jue and Luo, Ping},
  journal=NIPS,
  volume={35},
  pages={16664--16678},
  year={2022}
}

@article{adapter2,
  title={St-adapter: Parameter-efficient image-to-video transfer learning},
  author={Pan, Junting and Lin, Ziyi and Zhu, Xiatian and Shao, Jing and Li, Hongsheng},
  journal=NIPS,
  volume={35},
  pages={26462--26477},
  year={2022}
}

@article{adapter3,
  title={Aim: Adapting image models for efficient video action recognition},
  author={Yang, Taojiannan and Zhu, Yi and Xie, Yusheng and Zhang, Aston and Chen, Chen and Li, Mu},
  journal=ICLR,
  year={2023}
}

@inproceedings{adapter4,
  title={Zeroi2v: Zero-cost adaptation of pre-trained transformers from image to video},
  author={Li, Xinhao and Zhu, Yuhan and Wang, Limin},
  booktitle=ECCV,
  pages={425--443},
  year={2024},
  organization={Springer}
}

@inproceedings{adapter5,
  title={Frozen clip models are efficient video learners},
  author={Lin, Ziyi and Geng, Shijie and Zhang, Renrui and Gao, Peng and De Melo, Gerard and Wang, Xiaogang and Dai, Jifeng and Qiao, Yu and Li, Hongsheng},
  booktitle=ECCV,
  pages={388--404},
  year={2022},
  organization={Springer}
}

@inproceedings{dual1,
  title={Dual-path adaptation from image to video transformers},
  author={Park, Jungin and Lee, Jiyoung and Sohn, Kwanghoon},
  booktitle=CVPR,
  pages={2203--2213},
  year={2023}
}

@inproceedings{dual2,
  title={Disentangling spatial and temporal learning for efficient image-to-video transfer learning},
  author={Qing, Zhiwu and Zhang, Shiwei and Huang, Ziyuan and Zhang, Yingya and Gao, Changxin and Zhao, Deli and Sang, Nong},
  booktitle=ICCV,
  pages={13934--13944},
  year={2023}
}

@article{slot_attenion,
  title={Object-centric learning with slot attention},
  author={Locatello, Francesco and Weissenborn, Dirk and Unterthiner, Thomas and Mahendran, Aravindh and Heigold, Georg and Uszkoreit, Jakob and Dosovitskiy, Alexey and Kipf, Thomas},
  journal=NIPS,
  volume={33},
  pages={11525--11538},
  year={2020}
}

@inproceedings{ssv2,
  title={The" something something" video database for learning and evaluating visual common sense},
  author={Goyal, Raghav and Ebrahimi Kahou, Samira and Michalski, Vincent and Materzynska, Joanna and Westphal, Susanne and Kim, Heuna and Haenel, Valentin and Fruend, Ingo and Yianilos, Peter and Mueller-Freitag, Moritz and others},
  booktitle=ICCV,
  pages={5842--5850},
  year={2017}
}

@article{smooth,
  title={Some informational aspects of visual perception.},
  author={Attneave, Fred},
  journal={Psychological review},
  volume={61},
  number={3},
  pages={183},
  year={1954},
  publisher={American Psychological Association}
}

@article{continuity,
  title={Natural image statistics and neural representation},
  author={Simoncelli, Eero P and Olshausen, Bruno A},
  journal={Annual review of neuroscience},
  volume={24},
  number={1},
  pages={1193--1216},
  year={2001},
  publisher={Annual Reviews 4139 El Camino Way, PO Box 10139, Palo Alto, CA 94303-0139, USA}
}

@inproceedings{cycle_patch_1,
  title={Learning correspondence from the cycle-consistency of time},
  author={Wang, Xiaolong and Jabri, Allan and Efros, Alexei A},
  booktitle=CVPR,
  pages={2566--2576},
  year={2019}
}

@article{cycle_patch_2,
  title={Joint-task self-supervised learning for temporal correspondence},
  author={Li, Xueting and Liu, Sifei and De Mello, Shalini and Wang, Xiaolong and Kautz, Jan and Yang, Ming-Hsuan},
  journal=NIPS,
  volume={32},
  year={2019}
}

@inproceedings{cycle_patch_3,
  title={Boosting video object segmentation via space-time correspondence learning},
  author={Zhang, Yurong and Li, Liulei and Wang, Wenguan and Xie, Rong and Song, Li and Zhang, Wenjun},
  booktitle=CVPR,
  pages={2246--2256},
  year={2023}
}

@inproceedings{cycle_patch_4,
  title={Self-Supervised Cross-View Correspondence with Predictive Cycle Consistency},
  author={Baade, Alan and Chen, Changan},
  booktitle=CVPR,
  pages={16753--16763},
  year={2025}
}

@inproceedings{cycle_match_1,
  title={Representation learning via global temporal alignment and cycle-consistency},
  author={Hadji, Isma and Derpanis, Konstantinos G and Jepson, Allan D},
  booktitle=CVPR,
  pages={11068--11077},
  year={2021}
}

@inproceedings{cycle_match_2,
  title={Temporal cycle-consistency learning},
  author={Dwibedi, Debidatta and Aytar, Yusuf and Tompson, Jonathan and Sermanet, Pierre and Zisserman, Andrew},
  booktitle=CVPR,
  pages={1801--1810},
  year={2019}
}

@inproceedings{cycle_match_3,
  title={Contrastive learning of image representations with cross-video cycle-consistency},
  author={Wu, Haiping and Wang, Xiaolong},
  booktitle=ICCV,
  pages={10149--10159},
  year={2021}
}

@article{cycle_walk_1,
  title={Space-time correspondence as a contrastive random walk},
  author={Jabri, Allan and Owens, Andrew and Efros, Alexei},
  journal=NIPS,
  volume={33},
  pages={19545--19560},
  year={2020}
}

@inproceedings{cycle_walk_2,
  title={Learning pixel trajectories with multiscale contrastive random walks},
  author={Bian, Zhangxing and Jabri, Allan and Efros, Alexei A and Owens, Andrew},
  booktitle=CVPR,
  pages={6508--6519},
  year={2022}
}

@inproceedings{cycle_walk_3,
  title={Self-supervised Any-Point Tracking by Contrastive Random Walks},
  author={Shrivastava, Ayush and Owens, Andrew},
  booktitle=ECCV,
  pages={267--284},
  year={2024},
  organization={Springer}
}

@article{shortcut_2,
  title={Breaking shortcut: Exploring fully convolutional cycle-consistency for video correspondence learning},
  author={Tang, Yansong and Jiang, Zhenyu and Xie, Zhenda and Cao, Yue and Zhang, Zheng and Torr, Philip HS and Hu, Han},
  journal={arXiv preprint arXiv:2105.05838},
  year={2021}
}

@inproceedings{track_4,
  title={Dino-tracker: Taming dino for self-supervised point tracking in a single video},
  author={Tumanyan, Narek and Singer, Assaf and Bagon, Shai and Dekel, Tali},
  booktitle=ECCV,
  year={2024}
}

@inproceedings{compression_2,
  title={Deep learning and the information bottleneck principle},
  author={Tishby, Naftali and Zaslavsky, Noga},
  booktitle={2015 ieee information theory workshop (itw)},
  pages={1--5},
  year={2015},
  organization={IEEE}
}

@inproceedings{compression_3,
  title={How does information bottleneck help deep learning?},
  author={Kawaguchi, Kenji and Deng, Zhun and Ji, Xu and Huang, Jiaoyang},
  booktitle=ICML,
  pages={16049--16096},
  year={2023},
  organization={PMLR}
}

@inproceedings{CLIP,
  title={Learning transferable visual models from natural language supervision},
  author={Radford, Alec and Kim, Jong Wook and Hallacy, Chris and Ramesh, Aditya and Goh, Gabriel and Agarwal, Sandhini and Sastry, Girish and Askell, Amanda and Mishkin, Pamela and Clark, Jack and others},
  booktitle=ICML,
  pages={8748--8763},
  year={2021},
  organization={PMLR}
}

@inproceedings{BLIP,
  title={Blip: Bootstrapping language-image pre-training for unified vision-language understanding and generation},
  author={Li, Junnan and Li, Dongxu and Xiong, Caiming and Hoi, Steven},
  booktitle=ICML,
  pages={12888--12900},
  year={2022},
  organization={PMLR}
}

@article{LAION,
  title={Laion-400m: Open dataset of clip-filtered 400 million image-text pairs},
  author={Schuhmann, Christoph and Vencu, Richard and Beaumont, Romain and Kaczmarczyk, Robert and Mullis, Clayton and Katta, Aarush and Coombes, Theo and Jitsev, Jenia and Komatsuzaki, Aran},
  journal={arXiv preprint arXiv:2111.02114},
  year={2021}
}

@inproceedings{DINO,
  title={Emerging properties in self-supervised vision transformers},
  author={Caron, Mathilde and Touvron, Hugo and Misra, Ishan and J{\'e}gou, Herv{\'e} and Mairal, Julien and Bojanowski, Piotr and Joulin, Armand},
  booktitle=ICCV,
  pages={9650--9660},
  year={2021}
}

@inproceedings{MAE,
  title={Masked autoencoders are scalable vision learners},
  author={He, Kaiming and Chen, Xinlei and Xie, Saining and Li, Yanghao and Doll{\'a}r, Piotr and Girshick, Ross},
  booktitle=CVPR,
  pages={16000--16009},
  year={2022}
}

@article{SiamMAE,
  title={Siamese masked autoencoders},
  author={Gupta, Agrim and Wu, Jiajun and Deng, Jia and Li, Fei-Fei},
  journal=NIPS,
  volume={36},
  pages={40676--40693},
  year={2023}
}

@InProceedings{RSP,
  title={Visual Representation Learning with Stochastic Frame Prediction},
  author={Jang, Huiwon and Kim, Dongyoung and Kim, Junsu and Shin, Jinwoo and Abbeel, Pieter and Seo, Younggyo},
  booktitle=ICML,
  volume={235},
  pages={21289--21305},
  year={2024}
}

@inproceedings{CropMAE,
  title={Efficient Image Pre-Training with Siamese Cropped Masked Autoencoders},
  author={Eyma{\"e}l, Alexandre and Vandeghen, Renaud and Cioppa, Anthony and Giancola, Silvio and Ghanem, Bernard and Van Droogenbroeck, Marc},
  booktitle=ECCV,
  year={2024}
}

@article{iBOT,
  title={ibot: Image bert pre-training with online tokenizer},
  author={Zhou, Jinghao and Wei, Chen and Wang, Huiyu and Shen, Wei and Xie, Cihang and Yuille, Alan and Kong, Tao},
  journal=ICLR,
  year={2022}
}

@article{DINOv2,
  title={Dinov2: Learning robust visual features without supervision},
  author={Oquab, Maxime and Darcet, Timoth{\'e}e and Moutakanni, Th{\'e}o and Vo, Huy and Szafraniec, Marc and Khalidov, Vasil and Fernandez, Pierre and Haziza, Daniel and Massa, Francisco and El-Nouby, Alaaeldin and others},
  journal={Transactions on Machine Learning Research},
  year={2024},
}

@article{DINOv3,
  title={Dinov3},
  author={Sim{\'e}oni, Oriane and Vo, Huy V and Seitzer, Maximilian and Baldassarre, Federico and Oquab, Maxime and Jose, Cijo and Khalidov, Vasil and Szafraniec, Marc and Yi, Seungeun and Ramamonjisoa, Micha{\"e}l and others},
  journal={arXiv preprint arXiv:2508.10104},
  year={2025}
}

@inproceedings{tcore,
  title={When the future becomes the past: Taming temporal correspondence for self-supervised video representation learning},
  author={Liu, Yang and Xu, Qianqian and Wen, Peisong and Dai, Siran and Huang, Qingming},
  booktitle=CVPR,
  pages={24033--24044},
  year={2025}
}

@inproceedings{Vision_Transformer,
  title={An Image is Worth 16x16 Words: Transformers for Image Recognition at Scale},
  author={Dosovitskiy, Alexey and Beyer, Lucas and Kolesnikov, Alexander and Weissenborn, Dirk and Zhai, Xiaohua and Unterthiner, Thomas and Dehghani, Mostafa and Minderer, Matthias and Heigold, Georg and Gelly, Sylvain and others},
  booktitle=ICLR,
  year={2021}
}

@article{UCF101,
  title={UCF101: A dataset of 101 human actions classes from videos in the wild},
  author={Soomro, K},
  journal={arXiv preprint arXiv:1212.0402},
  year={2012}
}

@inproceedings{HMDB51,
  title={HMDB: a large video database for human motion recognition},
  author={Kuehne, Hildegard and Jhuang, Hueihan and Garrote, Est{\'\i}baliz and Poggio, Tomaso and Serre, Thomas},
  booktitle=ICCV,
  pages={2556--2563},
  year={2011},
  organization={IEEE}
}

@article{DAVIS17,
  title={The 2017 davis challenge on video object segmentation},
  author={Pont-Tuset, Jordi and Perazzi, Federico and Caelles, Sergi and Arbel{\'a}ez, Pablo and Sorkine-Hornung, Alex and Van Gool, Luc},
  journal={arXiv preprint arXiv:1704.00675},
  year={2017}
}

@inproceedings{JHMDB,
  title={Towards understanding action recognition},
  author={Jhuang, Hueihan and Gall, Juergen and Zuffi, Silvia and Schmid, Cordelia and Black, Michael J},
  booktitle={Proceedings of the IEEE international conference on computer vision},
  pages={3192--3199},
  year={2013}
}

@inproceedings{VIP,
  title={Adaptive temporal encoding network for video instance-level human parsing},
  author={Zhou, Qixian and Liang, Xiaodan and Gong, Ke and Lin, Liang},
  booktitle={Proceedings of the 26th ACM international conference on Multimedia},
  pages={1527--1535},
  year={2018}
}

@inproceedings{FACT,
  title={Fact: Frame-action cross-attention temporal modeling for efficient action segmentation},
  author={Lu, Zijia and Elhamifar, Ehsan},
  booktitle=CVPR,
  pages={18175--18185},
  year={2024}
}

@inproceedings{Breakfast,
  title={The language of actions: Recovering the syntax and semantics of goal-directed human activities},
  author={Kuehne, Hilde and Arslan, Ali and Serre, Thomas},
  booktitle=CVPR,
  pages={780--787},
  year={2014}
}

@article{Kinetics,
  title={The kinetics human action video dataset},
  author={Kay, Will and Carreira, Joao and Simonyan, Karen and Zhang, Brian and Hillier, Chloe and Vijayanarasimhan, Sudheendra and Viola, Fabio and Green, Tim and Back, Trevor and Natsev, Paul and others},
  journal={arXiv preprint arXiv:1705.06950},
  year={2017}
}

@inproceedings{Kinetics_track,
  title={Quo vadis, action recognition? a new model and the kinetics dataset},
  author={Carreira, Joao and Zisserman, Andrew},
  booktitle=CVPR,
  pages={6299--6308},
  year={2017}
}

@inproceedings{BADJA,
  title={Creatures great and smal: Recovering the shape and motion of animals from video},
  author={Biggs, Benjamin and Roddick, Thomas and Fitzgibbon, Andrew and Cipolla, Roberto},
  booktitle=ACCV,
  pages={3--19},
  year={2019},
  organization={Springer}
}

@article{tap_davis,
  title={Tap-vid: A benchmark for tracking any point in a video},
  author={Doersch, Carl and Gupta, Ankush and Markeeva, Larisa and Recasens, Adria and Smaira, Lucas and Aytar, Yusuf and Carreira, Joao and Zisserman, Andrew and Yang, Yi},
  journal={Advances in Neural Information Processing Systems},
  volume={35},
  pages={13610--13626},
  year={2022}
}

@inproceedings{Imagenet,
  title={Imagenet: A large-scale hierarchical image database},
  author={Deng, Jia and Dong, Wei and Socher, Richard and Li, Li-Jia and Li, Kai and Fei-Fei, Li},
  booktitle=CVPR,
  pages={248--255},
  year={2009},
  organization={Ieee}
}

@inproceedings{AdamW,
  title={Decoupled Weight Decay Regularization},
  author={Loshchilov, Ilya and Hutter, Frank},
  booktitle=ICLR,
  year={2019}
}

@article{bottleneck1,
  title={The information bottleneck method},
  author={Tishby, Naftali and Pereira, Fernando C and Bialek, William},
  journal={arXiv preprint physics/0004057},
  year={2000}
}

@inproceedings{bottleneck2,
  title={Deep learning and the information bottleneck principle},
  author={Tishby, Naftali and Zaslavsky, Noga},
  booktitle={2015 ieee information theory workshop (itw)},
  pages={1--5},
  year={2015},
  organization={Ieee}
}

@misc{vae1,
  title={Auto-encoding variational bayes},
  author={Kingma, Diederik P and Welling, Max and others},
  year={2013},
  publisher={Banff, Canada}
}

@inproceedings{vae2,
  title={Regularized autoencoders for isometric representation learning},
  author={Lee, Yonghyeon and Yoon, Sangwoong and Son, MinJun and Park, Frank C},
  booktitle=ICLR,
  year={2022}
}

@inproceedings{vae3,
  title={Learning flat latent manifolds with VAEs},
  author={Chen, Nutan and Klushyn, Alexej and Ferroni, Francesco and Bayer, Justin and Van Der Smagt, Patrick},
  booktitle=ICML,
  pages={1587--1596},
  year={2020}
}

@article{vae4,
  title={Isometric quotient variational auto-encoders for structure-preserving representation learning},
  author={Huh, In and Choe, Jae Myung and KIM, YOUNGGU and Kim, Daesin and others},
  journal=NIPS,
  volume={36},
  pages={39075--39087},
  year={2023}
}

@article{sam2,
  title={Sam 2: Segment anything in images and videos},
  author={Ravi, Nikhila and Gabeur, Valentin and Hu, Yuan-Ting and Hu, Ronghang and Ryali, Chaitanya and Ma, Tengyu and Khedr, Haitham and R{\"a}dle, Roman and Rolland, Chloe and Gustafson, Laura and others},
  journal={arXiv preprint arXiv:2408.00714},
  year={2024}
}

@article{Chirality,
  title={Chirality in action: Time-aware video representation learning by latent straightening},
  author={Bagad, Piyush and Zisserman, Andrew},
  journal={arXiv preprint arXiv:2509.08502},
  year={2025}
}

@inproceedings{dense_adapter,
  title={Rethinking image-to-video adaptation: An object-centric perspective},
  author={Qian, Rui and Ding, Shuangrui and Lin, Dahua},
  booktitle=ECCV,
  pages={329--348},
  year={2024},
  organization={Springer}
}

@inproceedings{clip_adapter,
  title={CLIP's Visual Embedding Projector is a Few-shot Cornucopia},
  author={Fahes, Mohammad and Vu, Tuan-Hung and Bursuc, Andrei and P{\'e}rez, Patrick and De Charette, Raoul},
  booktitle=WACV,
  pages={3254--3264},
  year={2026}
}

@article{wen2025semantic,
  title={Semantic Concentration for Self-Supervised Dense Representations Learning},
  author={Wen, Peisong and Xu, Qianqian and Dai, Siran and Cong, Runmin and Huang, Qingming},
  journal=PAMI,
  year={2025},
  publisher={IEEE}
}

@article{dai2025exploring,
  title={Exploring Structural Degradation in Dense Representations for Self-supervised Learning},
  author={Dai, Siran and Xu, Qianqian and Wen, Peisong and Liu, Yang and Huang, Qingming},
  journal={arXiv preprint arXiv:2510.17299},
  year={2025}
}

@article{dai2025exploring2,
  title={Exploring non-contrastive self-supervised representation learning for image-based profiling},
  author={Dai, Siran and Xu, Qianqian and Wen, Peisong and Liu, Yang and Huang, Qingming},
  journal={arXiv e-prints},
  pages={arXiv--2506},
  year={2025}
}

@inproceedings{s2vs,
  title={Self-supervised video similarity learning},
  author={Kordopatis-Zilos, Giorgos and Tolias, Giorgos and Tzelepis, Christos and Kompatsiaris, Ioannis and Patras, Ioannis and Papadopoulos, Symeon},
  booktitle=CVPR,
  pages={4756--4766},
  year={2023}
}

@inproceedings{liu2024not,
  title={Not All Pairs are Equal: Hierarchical Learning for Average-Precision-Oriented Video Retrieval},
  author={Liu, Yang and Xu, Qianqian and Wen, Peisong and Dai, Siran and Huang, Qingming},
  booktitle=ACMMM,
  pages={3828--3837},
  year={2024}
}

@article{wang2025unified,
  title={A unified perspective for loss-oriented imbalanced learning via localization},
  author={Wang, Zitai and Xu, Qianqian and Yang, Zhiyong and Xu, Zhikang and Zhang, Linchao and Cao, Xiaochun and Huang, Qingming},
  journal=PAMI,
  year={2025},
  publisher={IEEE}
}

@article{wang2022openauc,
  title={Openauc: Towards auc-oriented open-set recognition},
  author={Wang, Zitai and Xu, Qianqian and Yang, Zhiyong and He, Yuan and Cao, Xiaochun and Huang, Qingming},
  journal=NIPS,
  volume={35},
  pages={25033--25045},
  year={2022}
}

@article{han2024aucseg,
  title={Aucseg: Auc-oriented pixel-level long-tail semantic segmentation},
  author={Han, Boyu and Xu, Qianqian and Yang, Zhiyong and Bao, Shilong and Wen, Peisong and Jiang, Yangbangyan and Huang, Qingming},
  journal=NIPS,
  volume={37},
  pages={126863--126907},
  year={2024}
}

@article{li2024size,
  title={Size-invariance matters: Rethinking metrics and losses for imbalanced multi-object salient object detection},
  author={Li, Feiran and Xu, Qianqian and Bao, Shilong and Yang, Zhiyong and Cong, Runmin and Cao, Xiaochun and Huang, Qingming},
  journal={arXiv preprint arXiv:2405.09782},
  year={2024}
}

@article{pytorch,
  title={Pytorch: An imperative style, high-performance deep learning library},
  author={Paszke, Adam and Gross, Sam and Massa, Francisco and Lerer, Adam and Bradbury, James and Chanan, Gregory and Killeen, Trevor and Lin, Zeming and Gimelshein, Natalia and Antiga, Luca and others},
  journal=NIPS,
  volume={32},
  year={2019}
}

@inproceedings{he2015delving,
  title={Delving deep into rectifiers: Surpassing human-level performance on imagenet classification},
  author={He, Kaiming and Zhang, Xiangyu and Ren, Shaoqing and Sun, Jian},
  booktitle=ICCV,
  pages={1026--1034},
  year={2015}
}

@inproceedings{glorot2010understanding,
  title={Understanding the difficulty of training deep feedforward neural networks},
  author={Glorot, Xavier and Bengio, Yoshua},
  booktitle={Proceedings of the thirteenth international conference on artificial intelligence and statistics},
  pages={249--256},
  year={2010},
  organization={JMLR Workshop and Conference Proceedings}
}

@article{RoPE,
  title={Roformer: Enhanced transformer with rotary position embedding},
  author={Su, Jianlin and Ahmed, Murtadha and Lu, Yu and Pan, Shengfeng and Bo, Wen and Liu, Yunfeng},
  journal={Neurocomputing},
  volume={568},
  pages={127063},
  year={2024},
  publisher={Elsevier}
}

@article{STCN,
  title={Rethinking space-time networks with improved memory coverage for efficient video object segmentation},
  author={Cheng, Ho Kei and Tai, Yu-Wing and Tang, Chi-Keung},
  journal=NIPS,
  volume={34},
  pages={11781--11794},
  year={2021}
}

@article{SwinB,
  title={Associating objects with transformers for video object segmentation},
  author={Yang, Zongxin and Wei, Yunchao and Yang, Yi},
  journal=NIPS,
  volume={34},
  pages={2491--2502},
  year={2021}
}

@inproceedings{SimVOS,
  title={Scalable video object segmentation with simplified framework},
  author={Wu, Qiangqiang and Yang, Tianyu and Wu, Wei and Chan, Antoni B},
  booktitle=ICCV,
  pages={13879--13889},
  year={2023}
}

@inproceedings{Cutie,
  title={Putting the object back into video object segmentation},
  author={Cheng, Ho Kei and Oh, Seoung Wug and Price, Brian and Lee, Joon-Young and Schwing, Alexander},
  booktitle=CVPR,
  pages={3151--3161},
  year={2024}
}

@inproceedings{thoker2023tubelet,
  title={Tubelet-contrastive self-supervision for video-efficient generalization},
  author={Thoker, Fida Mohammad and Doughty, Hazel and Snoek, Cees GM},
  booktitle=ICCV,
  pages={13812--13823},
  year={2023}
}
}

\maketitlesupplementary

\appendix

\section*{{\Large{Appendix Contents}}}
\startcontents[appendices]
\printcontents[appendices]{l}{1}{\setcounter{tocdepth}{3}}
\newpage

\section{Symbol Definitions}
\label{sec:symbol}

We summarize the key notations used in the Method and Theoretical Analysis sections in \cref{tab:symbols_method} and \cref{tab:symbols_theory}.

\begin{table}[htbp]
\vspace{6pt}
  \centering
  \caption{A summary of key notations and descriptions used in the Method Section.}
  \vspace{-6pt}
  \resizebox{0.48\textwidth}{!}{%
    \begin{tabular}{ll}
    \toprule
    Notations & Descriptions \\
    \midrule
    $\mathcal{D}$ & Training dataset. \\
    $H/W/C$ & The height/width/channel dimension of a frame. \\
    $T$ & The number of frames in a video. \\
    $\bm{V}$ & A video containing $T$ frames. \\
    $\bm{v}_{t}$ & The frame at the moment $t$ in a video, $\bm{v}_{t} \in \mathbb{R}^{H \times W \times C}$. \\
    $\bm{v}_{t}(i)$ & A frame patch of $\bm{v}_{t}$, $\bm{v}_{t} \in \mathbb{R}^{H \times W \times C}$. \\

    $\delta$ & The temporal offset between two frames, $\delta \in (0,1)$. \\
    $p$ & The size of a frame patch. \\
    $N_H$ & The patch number on the height dimension, $N_H=H/p$. \\
    $N_W$ & The patch number on the width dimension, $N_W=W/p$. \\
    $N$ & The patch number of a frame, $N=N_H \times N_W$. \\
    $d$ & The embedding dimension of a frame patch. \\
    $f(\cdot)$ & The image-pretrained encoder, $f: \mathbb{R}^{H \times W \times C} \rightarrow \mathbb{R}^{N \times d}$. \\
    $g(\cdot)$ & The projection layer, $g: \mathbb{R}^{N \times d} \rightarrow \mathbb{R}^{N \times d}$. \\

    $\alpha$ & The amplitude of the positional encoding interpolation. \\
    $\mathbf{E}_{\text{pos}}$ & The positional encoding of $f$. \\
    $\widetilde{\mathbf{E}}_{\text{pos}}$ & The augmented version of $\mathbf{E}_{\text{pos}}$. \\
    $\bm{z}_t$ & The original representation of $\bm{v}_t$, where $\bm{z}_t = f(\bm{v}_t)$. \\
    $\bm{p}_t$ & The projected representation of $\bm{z}_t$, where $\bm{p}_t = g(\bm{z}_t)$. \\
    $\Atatb$ & The transition matrix between representations $\bm{p}_{t1}$ and $\bm{p}_{t2}$. \\
    $\lambda$ & The strength of the constraint term. \\
    \bottomrule
    \end{tabular}%
    }
  \label{tab:symbols_method}%
\end{table}%

\begin{table}[htbp]
\vspace{4pt}
  \centering
  \caption{A summary of key notations and descriptions used in the Theoretical Analysis Section.}
  \vspace{-6pt}
  \resizebox{0.48\textwidth}{!}{%
    \begin{tabular}{ll}
    \toprule
    Notations & Descriptions \\
    \midrule
     $d$ & The embedding dimension of a frame patch. \\
     $\lambda$ & The strength of the constraint term. \\
     $f(\cdot)$ & The image-pretrained encoder, $f: \mathbb{R}^{H \times W \times C} \rightarrow \mathbb{R}^{N \times d}$. \\
     $g(\cdot)$ & The projection layer, $g: \mathbb{R}^{N \times d} \rightarrow \mathbb{R}^{N \times d}$. \\
     $\tzi$ & The latent representation of an input patch. \\
     $\tbarzi$ & The mean representation of video $\bm{V}_i$, $\tbarzi = \mathbb{E}_{\bm{z} \in f(\bm{V}_i)} \left[\bm{z}\right]$ . \\
     $\bm{p}_i$ & The projected representation of $\tzi$. \\
     $\tW$ & The projection weight of the linear layer. \\
     $\tWla/\tWlb$ & The projection weight of the two-layer MLP. \\
     $\phi(\cdot)$ & The \texttt{tanh} activation function. \\
     $\tJ_g(\cdot)$ & The Jacobian matrix of $g$.  \\
     $\tSigma$ & The intra-video covariance matrix between two patches. \\
     $\tbarSigma$ & The inter-video covariance matrix between two videos.\\
     $\tU$ & The orthogonal basis for spectral decomposition. \\
     $\tLambdaW/\tLambdaWla/\tLambdaWlb$ & The eigenvalue matrices of $\tW/\tWla/\tWlb$.  \\
     $\tLambdaSigma/\tLambdabarSigma$ & The eigenvalue matrices of $\tSigma/\tbarSigma$.  \\
     $\mu_{i}/\mu_{1,i}/\mu_{2,i}$ & The eigenvalues of $\tW/\tWla/\tWlb$. \\
     $\sigma_{i}/\tau_{i}$ & The eigenvalues of $\tSigma/\tbarSigma$.  \\
     $D_{intra}$ & The intra-video distance between two patches. \\
     $D_{inter}$ & The inter-video distance between two videos. \\
     $\gamma$ & The scale factor between $D_{intra}$ and $D_{inter}$. \\
     $D$ & The margin of inter-/intra-video distances.  \\
     $\Delta$ & The improvement of $D(\tza,\tzb)$.  \\  
    \bottomrule
    \end{tabular}%
    }
  \label{tab:symbols_theory}%
\end{table}%

\section{Formal Theorems and Proofs}
\label{sec:formal_theory}

In this section, we provide detailed proofs for the theoretical analysis of how the proposed method achieves our target, \ie, improving the \textbf{intra-video temporal consistency} without largely affecting the \textbf{inter-video semantic separability}.
Generally, our analysis leads to two main conclusions: 
\textbf{a)} Within our proposed method, both linear-based and MLP projection rebalance different dimensions of the representation space in a similar mechanism (\cref{thm:thm1}).
\textbf{b)} This rebalance yields a better trade-off between the two properties under appropriate conditions (\cref{thm:thm2}).

Formally, given the original representation of a patch $\tzi\in\mathbb R^d$, we aim to learn a projection $g$ that maps $\tzi$ to $\bm{p}_i = g(\tzi) \in \mathbb{R}^{d}$.
Since directly analyzing the original objectives $\mathcal{L}_{\text{cyc}}$ and $\mathcal{L}_{\text{reg}}$ is challenging, we introduce simplified yet equivalent surrogates to facilitate the analysis.

\textbf{Objective 1 (Temporal Cycle Consistency).} 
This term encourages alignment between temporally corresponding patches. We quantify it with the metric in \cref{eq:M_cyc}.
Note that minimizing $M_{\text{cyc}}$ is equivalent to minimizing the cycle-consistency loss $L_{\text{cyc}}$, since both decrease as temporal consistency improves and share the same optimality conditions.
\begin{equation}
M_{\text{cyc}} = \frac12\,\mathbb{E}_{\tza,\tzb}\!\bigl[\lVert g(\tza)-g(\tzb)\rVert^{2}\bigr].
\label{eq:M_cyc}
\end{equation}

\textbf{Objective 2 (Semantic Separability Constraint)}: 
The KL divergence constraint $\mathcal{L}_{\text{reg}}$ preserves the distance relationships between patches before and after the projection,
which is equivalent to constraining the projection to be isometric. This property can be measured by the orthogonality of the Jacobian matrix~\cite{vae1,vae2,vae3,vae4} of $g$, as formulated in \cref{eq:M_reg}. Therefore, we use it as an approximation of $L_{\text{reg}}$.
\begin{equation}
M_{\text{reg}} = \frac{1}{2}\,\mathbb{E}_{\tzi}\!\bigl[\lVert \tJ_g(\tzi)\tJ_g(\tzi)^{\top}-\bm{I}\rVert_{F}^{2}\bigr].
\label{eq:M_reg}
\end{equation}

Combining the two surrogates yields the overall objective:
\begin{equation}
\min_g \; M(g) = M_{\text{cyc}} + \lambda M_{\text{reg}}.
\label{eq:theory_obj}
\end{equation}

We now consider two representative cases for $g$:
\textbf{i)} A linear projection: $g(\bm{z}) = \tW \bm{z}$;
\textbf{ii)} A two-layer MLP: $g(\bm{z}) = \tWlb \, \phi(\tWla \bm{z})$ with activation function $\phi(\cdot)=tanh(\cdot)$, and this case represents more complex modules.
The following theorem analyzes the spectral properties of the optimal solution under both cases, illustrating how the projection affects the quality of the transferred representation.

\subsection{Spectral Properties of Optimal Projections}
\label{subsec:theorem1}

\subsubsection{Settings for Linear Projection}
\label{subsubsec:linear_projection}

For the linear projection $g(\bm{z}) = \tW \bm{z}$, the Jacobian matrix can be expressed as $\tJ_g(\tzi) = \dfrac{\partial g}{\partial \tzi} = \tW$ , thereby the optimization objective can be reformulated as:

\begin{equation}
\begin{aligned}
\min_{\tW}\;M(\tW)
&=\frac12\mathbb{E}_{\tza,\tzb}\bigl[\lVert \tW\tza-\tW\tzb \rVert^{2}\bigr] \\
&\quad +\frac{\lambda}{2}\lVert \tW\tW^{\top}-\bm{I}\rVert_{F}^{2}.
\label{eq:theory_obj_1}
\end{aligned}
\end{equation}

To facilitate the analysis, we begin by introducing several definitions and assumptions.

\begin{defn}[Intra-video Covariance Matrix]
\label{defn:intra_covariance} 
Define the intra-video covariance matrix as the covariance of the patch representations that exhibit corresponding relationships between different frames in a single video: $\tSigma = \mathbb{E}_{\tza,\tzb}\bigl[(\tza-\tzb)(\tza-\tzb)^\top\bigr]$, where $(\tza,\tzb)$ denotes a pair of temporally aligned patch representations.
\end{defn}

\begin{defn}[Inter-video Covariance Matrix]
\label{defn:inter_covariance} 
Define the inter-video covariance matrix as the covariance of the video-level representations across the dataset: $\tbarSigma = \mathbb{E}_{\tbarza,\tbarzb}\bigl[(\tbarza-\tbarzb)(\tbarza-\tbarzb)^\top\bigr]$, where $\tbarzi = \mathbb{E}_{\bm{z} \in f(\bm{V}_i)} \left[\bm{z}\right]$ denotes the mean representation of video $\bm{V}_i$.
\end{defn}

\begin{assmp}[Symmetric Operator]\label{assmp:sym_1} Without loss of generality, $\tW$, $\tSigma$ and $\tbarSigma$ are constrained to be symmetric ($\tW^\top = \tW, \tSigma^\top = \tSigma, \tbarSigma^\top = \tbarSigma$). This is justified because any optimal $\tW$ can be symmetrized without increasing~\eqref{eq:theory_obj}.
\end{assmp}

\begin{assmp}[Positive Semi-definite]\label{assmp:psd_1} The transformation operator $\tW$, the intra-video covariance matrix
$\tSigma$, and the inter-video covariance matrix
$\tbarSigma$ are positive semi-definite: $\tW\succeq0, \tSigma\succeq0, \tbarSigma\succeq0$. This ensures all eigenvalues are non-negative.
\end{assmp}

\begin{assmp}[Commutative Minimizer]\label{assmp:commuting_1} we restrict the analysis to real symmetric commuting pairs $(\tW,\tSigma)$ and $(\tW,\tbarSigma)$, \ie, $\tSigma\tW=\tW\tSigma$ and  $\tbarSigma\tW=\tW\tbarSigma$. This allows simultaneous diagonalization with a common orthogonal basis $\tU$, yielding $\tW=\tU\tLambdaW\tU^\top$, $\tSigma=\tU\tLambdaSigma\tU^\top$, and $\tbarSigma=\tU\tLambdabarSigma\tU^\top$, where $\tLambdaW, \tLambdaSigma, \tLambdabarSigma$ denote corresponding eigenvalue matrices.
\end{assmp}

\subsubsection{Settings for MLP Projection}
\label{subsubsec:mlp_projection}

For the MLP projection $g(\bm{z}) = \tWlb  \phi(\tWla \bm{z})$, the optimization objective can be reformulated as:

\begin{equation}
\begin{aligned}
\min_{g}\;M(g)
&=\frac12\mathbb{E}_{\tza,\tzb}\bigl[\lVert g(\tza)-g(\tzb)\rVert^{2}\bigr] \\
&\quad +\frac{\lambda}{2}\mathbb{E}_{\tzi}\bigl[\lVert \tJ_g(\tzi)\tJ_g(\tzi)^{\top}-\bm{I}\rVert_{F}^{2}\bigr].
\label{eq:theory_obj_2}
\end{aligned}
\end{equation}

To facilitate the analysis, we begin by introducing a set of assumptions analogous to those in Case~\textbf{i)}. Specifically, we replace the matrix $\bm{W}$ in \cref{assmp:sym_1,assmp:psd_1,assmp:commuting_1} with $\bm{W}_1$ and $\bm{W}_2$, respectively. In addition to these modifications, we introduce the following additional assumptions:

\begin{assmp}[Gaussian Distribution]\label{assmp:distribution_2}
Without loss of generality, we assume that each patch representation $\bm{z}_i \in \mathbb{R}^{d}$ is independently drawn from a multivariate Gaussian distribution: $\bm{z} \sim \mathcal{N}(\bm{0}, \tSigma)$. This assumption is justified by the observation that patch-level features extracted from natural videos tend to exhibit approximately Gaussian behavior due to the high-dimensional embedding and the central limit effect. Consequently, $\mathbb{E}_{\bm{z}_i}[\bm{z}_i\bm{z}_i^\top] = \tSigma$.
\end{assmp}

\begin{assmp}[Linear Approximation]\label{assmp:approx_2}
Assuming most values fall within the near-linear region of the \texttt{tanh} activation, we adopt the approximation $\phi(\tU\bm{x}) \approx \tU\phi(\bm{x})$, where $\tU$ is an orthogonal matrix and $\phi(\cdot) = \tanh(\cdot)$ is applied element-wise.
\end{assmp}

\subsubsection{Proof for Theorem 1}
\label{subsubsec:thm_1_proof}

Based on the settings above, we establish the following theorem, which characterizes the spectral properties of the optimal solution in both cases and illustrates how the projection affects the quality of the transferred representation.

\begin{thm}[Spectral Properties of Optimal Projections, Formal]
\label{thm:thm1}
Denote the intra-video covariance matrix as $\tSigma = \mathbb{E}_{\tza,\tzb}\bigl[(\tza-\tzb)(\tza-\tzb)^\top\bigr]$. Let $\{\sigma_{i}\}_{i=1}^d$ be the eigenvalues of $\tSigma$.

For case i), assume symmetric matrices $\tW$ and $\tSigma$ are positive semi-definite and mutually commuting. Let $\{\mu_{i}\}_{i=1}^d$ be the eigenvalues of $\tW$.
Then the eigenvalues of the optimal projection $\tW^\star$ obey:
\begin{equation}
    \mu_i^\star=
    \begin{cases}
        0, & \sigma_i > 2\lambda,\\
        \sqrt{1-\dfrac{\sigma_i}{2\lambda}}, & \sigma_i\leq 2\lambda.
    \end{cases}
\end{equation}
For case ii), assume $\phi(u\tzi) \approx u\phi(\tzi)$ holds for $\tzi\sim\mathcal{N}(\bm{0},\tSigma)$, and that symmetric matrices $\tWla$, $\tWlb$, $\tSigma$ are positive semi-definite and mutually commuting. Let $\{\mu_{1,i}\}_{i=1}^d$ and $\{\mu_{2,i}\}_{i=1}^d$ be the eigenvalues of $\tWla$ and $\tWlb$, respectively.
Then the eigenvalues of the optimal projections ${\tWla}^\star$, ${\tWlb}^\star$ satisfy:
\begin{equation}
    \mu_{1,i}^\star\mu_{2,i}^\star=
    \begin{cases}
        0, & \sigma_i > 2\lambda,\\
        \sqrt{1-\dfrac{\sigma_i}{2\lambda}}, & \sigma_i\leq 2\lambda.
    \end{cases}
\end{equation}
\end{thm}

\begin{proof}
We first derive the \textbf{case i)} for the optimization objective of linear projection:
\begin{equation}
\begin{aligned}
M(\tW)
&=\underbrace{\frac12\mathbb{E}_{\tza,\tzb}\bigl[\lVert \tW\tza-\tW\tzb \rVert^{2}\bigr]}_{\text{Term A}} \\
&\quad +\underbrace{\frac{\lambda}{2}\lVert \tW\tW^{\top}-\bm{I}\rVert_{F}^{2}}_{\text{Term B}}.
\label{eq:theory_obj_1}
\end{aligned}
\end{equation}

The $\text{Term A}$ can be derived as:
\begin{equation}
\begin{aligned}
&\text{Term A} \\
&= \frac12\mathbb{E}_{\tza,\tzb}\bigl[\lVert \tW\tza-\tW\tzb \rVert^{2}\bigr] \\
&=\frac12\mathbb{E}_{\tza,\tzb}\bigl[ (\tW\tza-\tW\tzb)^\top(\tW\tza-\tW\tzb) \bigr]  \\
&=\frac12\mathbb{E}_{\tza,\tzb}\bigl[ \Tr( (\tW\tza-\tW\tzb)(\tW\tza-\tW\tzb)^\top ) \bigr]  \\
&=\frac12\mathbb{E}_{\tza,\tzb}\bigl[ \Tr(\tW(\tza-\tzb)(\tza-\tzb)^\top\tW^\top) \bigr]  \\
&=\frac12 \Tr(\tW^\top\tW\mathbb{E}_{\tza,\tzb}\bigl[(\tza-\tzb)(\tza-\tzb)^\top) \bigr] ) \\
&=\frac12 \Tr(\tW^\top\tW \tSigma ) \\
&=\frac12 \Tr((\tU\tLambdaW\tU^\top)^\top(\tU\tLambdaW\tU^\top) (\tU\tLambdaSigma\tU^\top)) \\
&=\frac12 \Tr(\tLambdaW^2\tLambdaSigma).
\label{eq:term_A}
\end{aligned}
\end{equation}

The derivation in \cref{eq:term_A} converts the squared $\ell_2$ norm into a matrix trace and further reduces it to a product of eigenvalues via orthogonal decomposition.

The $\text{Term B}$ can be derived as:
\begin{equation}
\begin{aligned}
& \text{Term B} \\
&= \frac{\lambda}{2}\lVert \tW\tW^{\top}-\bm{I}\rVert_{F}^{2} \\
 &= \frac{\lambda}{2} \Tr( (\tW\tW^{\top}-\bm{I})(\tW\tW^{\top}-\bm{I})^\top ) \\
 &= \frac{\lambda}{2} \Tr( (\tW^4 -2\tW^2+\bm{I} )) \\
&= \frac{\lambda}{2} \Tr( (\tU\tLambdaW\tU^\top)^4 -2(\tU\tLambdaW\tU^\top)^2+\bm{I} ) \\
&= \frac{\lambda}{2} \Tr( \tLambdaW^4 -2\tLambdaW^2)+\frac{\lambda d}{2}.
\end{aligned}
\end{equation}

Similarly, this derivation transforms the Frobenius norm into a matrix trace and reduces it to a function of eigenvalues via orthogonal decomposition.

Then the original objective can be rewritten as:
\begin{equation}
\begin{aligned}
M(\tW) 
&=\frac12\mathbb{E}_{\tza,\tzb}\bigl[\lVert \tW\tza-\tW\tzb \rVert^{2}\bigr] \\
&\quad +\frac{\lambda}{2}\lVert \tW\tW^{\top}-\bm{I}\rVert_{F}^{2} \\
&=\frac12 \Tr(\tLambdaW^2\tLambdaSigma) + \frac{\lambda}{2} \Tr( \tLambdaW^4 -2\tLambdaW^2)+\frac{\lambda d}{2} \\
&=\frac12 \sum_{i=1}^{d}({\mu_i}^2\sigma_i) + \frac{\lambda}{2} \sum_{i=1}^{d}( {\mu_i}^4 -2{\mu_i}^2)+\frac{\lambda d}{2}.
\end{aligned}
\label{eq:intermedia_1}
\end{equation}

Differentiating \cref{eq:intermedia_1} \wrt{} $\mu_i$ gives:
\begin{equation}
\begin{aligned}
\frac{\partial M}{\partial\mu_i}
&=\mu_i\sigma_i + 2\lambda\mu_i^3 - 2\lambda\mu_i \\
&=\mu_i(2\lambda\mu_i^2 - (2\lambda-\sigma_i )).
\end{aligned}
\end{equation}

Setting the derivative of the objective function to zero yields two possible solutions for each eigenvalue $\mu_i$:

\begin{itemize}[leftmargin=*]
    \item $\mu_i=0$. This is always a solution. It is optimal whenever the cubic term renders the quartic penalization unnecessary, \ie, when $\sigma_i>2\lambda$.
    \item $2\lambda\mu_i^2 - (2\lambda-\sigma_i )=0$.
    Solving for $\mu_i$ yields the non-zero stationary points, which exist precisely when $\sigma_i\le 2\lambda$.
\end{itemize}

In summary, the objective \cref{eq:theory_obj_1} reaches its minimum at
$
    \tW^\star=
      \tU
      \operatorname{diag}\bigl(\mu_1^\star,\dots,\mu_d^\star\bigr)
      \tU^\top,
$
where the eigenvalues of the optimal projection $\tW^\star$ obey:
\begin{equation}
    \mu_i^\star=
    \begin{cases}
        0, & \sigma_i \geq 2\lambda,\\
        \sqrt{1-\dfrac{\sigma_i}{2\lambda}}, & \sigma_i < 2\lambda.
    \end{cases}
\end{equation}

Afterward, we derive the \textbf{case ii)} for the optimization objective of MLP projection:

\begin{equation}
\begin{aligned}
\min_{g}\;M(g)
&=\underbrace{\frac12\mathbb{E}_{\tza,\tzb}\bigl[\lVert g(\tza)-g(\tzb)\rVert^{2}\bigr]}_{\text{Term C}} \\
&\quad +\underbrace{\frac{\lambda}{2}\mathbb{E}_{\tzi}\bigl[\lVert \tJ_g(\tzi)\tJ_g(\tzi)^{\top}-\bm{I}\rVert_{F}^{2}\bigr]}_{\text{Term D}}.
\label{eq:theory_obj_2}
\end{aligned}
\end{equation}

The \text{Term C} can be derived as:
\begin{equation}
\begin{split}
&\text{Term C} \\
&= \frac12\mathbb{E}_{\tza,\tzb}\bigl[\lVert g(\tza)-g(\tzb) \bigr \rVert^{2}\bigr] \\
&= \frac12\mathbb{E}_{\tza,\tzb}\bigl[\lVert \tWlb \bigl( \phi(\tWla\tza)-\phi(\tWlb\tzb) \bigr) \rVert^{2} \bigr] \\
&= \frac12\mathbb{E}_{\tza,\tzb}\bigl[\lVert \bigl( \tU\tLambdaWlb\tU^\top \bigr) \\
&\phantom{=\frac12\mathbb{E}_{\tza,\tzb}\bigl[\lVert \bigl(}
\bigl( \phi(\tU\tLambdaWla\tU^\top\tza)-\phi(\tU\tLambdaWla\tU^\top\tzb) \bigr) \rVert^{2} \bigr].
\end{split}
\end{equation}
Following \cref{assmp:approx_2}, the MLP projection becomes $g(\bm{z}) = \tWlb  \phi(\tWla \bm{z}) \approx \tU\tLambdaWlb\phi(\tLambdaWla\tU^\top\bm{z})$, we can continue to derive \text{Term C}:
\begin{equation}
\begin{aligned}
&\text{Term C} \\
&\approx \frac12\mathbb{E}_{\tza,\tzb}\bigl[\lVert \bigl( \tU\tLambdaWlb\tU^\top \bigr) \tU \\
&\phantom{\approx \frac12\mathbb{E}_{\tza,\tzb}\bigl[\lVert \bigl(}
\bigl( \phi(\tLambdaWla\tU^\top\tza)-\phi(\tLambdaWla\tU^\top\tzb) \bigr) \rVert^{2} \bigr] \\
&= \frac12\mathbb{E}_{\tza,\tzb}\bigl[\lVert \tLambdaWlb \bigl( \phi(\tLambdaWla\tU^\top\tza)-\phi(\tLambdaWla\tU^\top\tzb) \bigr) \rVert^{2} \bigr].
\end{aligned}
\end{equation}

Let $\bm{Q}=\tU^\top\bm{z}$. Since $\tU$ is orthogonal and $\bm{z} \sim \mathcal{N}(\bm{0}, \tSigma)$, it follows that $\bm{Q} \sim \mathcal{N}(\bm{0}, \tSigma)$.
Based on this property, we have:
\begin{equation}
\begin{aligned}
& \text{Term C} \\
&= \frac12\mathbb{E}_{\tza,\tzb}\bigl[\lVert \tLambdaWlb \bigl( \phi(\tLambdaWla\tQa)-\phi(\tLambdaWla\tQb) \bigr) \rVert^{2} \bigr] \\
&= \frac12\sum_{i=1}^{d} \mu_{2,i}^2  \mathbb{E}_{\tza,\tzb}\bigl[\bigl(  \phi(\mu_{1,i}\bm{Q}_{1,i})-\phi(\mu_{1,i}\bm{Q}_{2,i})  \bigr)^2 \bigr] \\
&\approx \frac12\sum_{i=1}^{d} \mu_{2,i}^2 \mu_{1,i}^2 \mathbb{E}_{\tza,\tzb}\bigl[\bigl( \bm{Q}_{1,i}-\bm{Q}_{2,i}  \bigr)^2 \bigr] \\
&= \frac12\sum_{i=1}^{d} \mu_{2,i}^2 \mu_{1,i}^2 \sigma_i. 
\end{aligned}
\end{equation}

Next, we consider the derivation of the \text{Term D}.
Note that $\tJ_g$ is the Jacobian matrix of $g(\tzi)$, it can be formulated as:
\begin{equation}
\begin{aligned}
\tJ_g(\tzi) 
&\approx \tU\tLambdaWlb\phi(\tLambdaWla\tU^\top\tzi) \\
&=\tU\tLambdaWlb\phi'(\tLambdaWla\tU^\top\tzi)\tLambdaWla\tU^\top \\
&=\tU\tLambdaWlb\bigl(1-\phi^2(\tLambdaWla\tU^\top\tzi)\bigr)\tLambdaWla\tU^\top.
\end{aligned}
\end{equation}

Therefore, the \text{Term D} can be derived as:
\begin{equation}
\begin{aligned}
&\text{Term D} \\
&= \frac{\lambda}{2}\mathbb{E}_{\tzi}\bigl[\lVert \tJ_g(\tzi)\tJ_g(\tzi)^{\top}-\bm{I}\rVert_{F}^{2}\bigr] \\
&= \frac{\lambda}{2} \mathbb{E}_{\tzi}\bigl[ \lVert \bigl( \tU\tLambdaWlb\bigl(1-\phi^2(\tLambdaWla\tU^\top\tzi)\bigr)\tLambdaWla\tU^\top \bigr)  \\
& \phantom{= \frac{\lambda}{2} \mathbb{E}_{\tzi}\bigl[ \lVert }
\bigl( \tU\tLambdaWlb\bigl(1-\phi^2(\tLambdaWla\tU^\top\tzi)\bigr)\tLambdaWla\tU^\top \bigr)^\top -\bm{I} \rVert_{F}^{2} \bigr] \\
&= \frac{\lambda}{2} \mathbb{E}_{\tzi}\bigl[ \lVert  \tLambdaWlb\bigl(1-\phi^2(\tLambdaWla\tU^\top\tzi)\bigr)\tLambdaWla \tLambdaWla^\top \\
& \phantom{= \frac{\lambda}{2} \mathbb{E}_{\tzi}\bigl[ \lVert }
\bigl(1-\phi^2(\tLambdaWla\tU^\top\tzi)\bigr)^\top\tLambdaWlb^\top -\bm{I} \rVert_{F}^{2} \bigr] \\
&= \frac{\lambda}{2} \sum_{i=1}^d \bigl( \mu_{1,i}^2\mu_{2,i}^2 \mathbb{E}_{\tzi}\bigl[  \bigl(1-\phi^2(\mu_{1,i}\bm{Q}_{i})\bigr)^2 \bigr] -1 \bigr)^2.
\end{aligned}
\end{equation}

Similarly, based on \cref{assmp:approx_2}, we can approximately move the coefficient of $\tzi$ outside the activation function $\phi(\cdot)$, leading to the following transformation into the function of eigenvalues:

\begin{equation}
\begin{aligned}
&\text{Term D} \\
&\approx \frac{\lambda}{2} \sum_{i=1}^d \bigl( \mu_{1,i}^2\mu_{2,i}^2  \bigl(1-\mu_{1,i}^2 \mathbb{E}_{\tzi}\bigl[ \bm{Q}_{i}^\top\bm{Q}_{i} \bigr]\bigr)^2  -1 \bigr)^2 \\
&= \frac{\lambda}{2} \sum_{i=1}^d \bigl( \mu_{1,i}^2\mu_{2,i}^2  \bigl(1-2\mu_{1,i}^2\mathbb{E}_{\tzi}\bigl[ \bm{Q}_{i}^\top\bm{Q}_{i} \bigr] \\
& \phantom{\frac{\lambda}{2} \sum_{i=1}^d \bigl( \mu_{1,i}^2\mu_{2,i}^2  \bigl(}
+ \mathbb{E}_{\tzi}\bigl[ (\bm{Q}_{i}^\top\bm{Q}_{i})^\top(\bm{Q}_{i}^\top\bm{Q}_{i}) \bigr]  \bigr)  -1 \bigr)^2 \\
&= \frac{\lambda}{2} \sum_{i=1}^d \bigl( \mu_{1,i}^2\mu_{2,i}^2  \bigl(1-2\sigma_i\mu_{1,i}^2 + 3\sigma_i^2\mu_{1,i}^4 \bigr)  -1 \bigr)^2 .
\end{aligned}
\end{equation}

Then the original objective can be rewritten as:
\begin{equation}
\begin{aligned}
M(g)
&=\frac12\sum_{i=1}^{d} \mu_{2,i}^2 \mu_{1,i}^2 \sigma_i \\
&+\frac{\lambda}{2} \sum_{i=1}^d \bigl( \mu_{1,i}^2\mu_{2,i}^2  \bigl(1-2\sigma_i\mu_{1,i}^2 + 3\sigma_i^2\mu_{1,i}^4 \bigr)  -1 \bigr)^2 .
\end{aligned}
\end{equation}

For simplicity, we assume $\mu_{1,i} \ll 1$, which is a reasonable approximation at initialization when using Kaiming~\cite{he2015delving} or Xavier~\cite{glorot2010understanding} schemes. Under this setting, the term $1 - 2\sigma_i\mu_{1,i}^2 + 3\sigma_i^2\mu_{1,i}^4$ approaches $1$, leading to the following simplification:

\begin{equation}
\begin{aligned}
M(g)
&=\frac12\sum_{i=1}^{d} \mu_{2,i}^2 \mu_{1,i}^2 \sigma_i
+\frac{\lambda}{2} \sum_{i=1}^d \bigl( \mu_{1,i}^2\mu_{2,i}^2  -1 \bigr)^2 .
\end{aligned}
\label{eq:intermedia_2}
\end{equation}

Differentiating \cref{eq:intermedia_2} \wrt{} $\mu_{1,i}$ gives:
\begin{equation}
\begin{aligned}
\frac{\partial M}{\partial\mu_{1,i}}
&=\mu_{1,i}\mu_{2,i}^2\sigma_i + 2\lambda \bigl(\mu_{1,i}^2\mu_{2,i}^2-1\bigr)\mu_{1,i}\mu_{2,i}^2 \\
&=\mu_{1,i}\mu_{2,i}^2  \bigl( \sigma_i + 2\lambda \bigl(\mu_{1,i}^2\mu_{2,i}^2-1\bigr) \bigr).
\end{aligned}
\end{equation}

Setting the derivative to zero produces two cases:
\begin{itemize}
    \item $\mu_{1,i}=0$ when $\sigma_i>2\lambda$.
    \item $ \mu_{1,i} = \dfrac{1}{\mu_{2,i}} \sqrt{1-\dfrac{\sigma_i}{2\lambda}}$.
    Solving for $\mu_{1,i}$ yields the non-zero stationary points, which exist precisely when $\sigma_i\le 2\lambda$.
\end{itemize}

Differentiating \cref{eq:intermedia_2} \wrt{} $\mu_{2,i}$ gives:
\begin{equation}
\begin{aligned}
\frac{\partial M}{\partial\mu_{2,i}}
&=\mu_{2,i}\mu_{1,i}^2\sigma_i + 2\lambda \bigl(\mu_{2,i}^2\mu_{1,i}^2-1\bigr)\mu_{2,i}\mu_{1,i}^2 \\
&=\mu_{2,i}\mu_{1,i}^2  \bigl( \sigma_i + 2\lambda \bigl(\mu_{2,i}^2\mu_{1,i}^2-1\bigr) \bigr).
\end{aligned}
\end{equation}

Setting the derivative to zero produces two cases:
\begin{itemize}
    \item $\mu_{2,i}=0$ when $\sigma_i>2\lambda$.
    \item $ \mu_{2,i} = \dfrac{1}{\mu_{1,i}} \sqrt{1-\dfrac{\sigma_i}{2\lambda}}$.
    Solving for $\mu_{2,i}$ yields the non-zero stationary points, which exist precisely when $\sigma_i\le 2\lambda$.
\end{itemize}

In summary, the objective \cref{eq:theory_obj_2} reaches its minimum at
$
    \tWla^\star= \tU\operatorname{diag}\bigl(\mu_{1,1}^\star,\dots,\mu_{1,d}^\star\bigr)\tU^\top,
    \quad 
    \tWlb^\star= \tU\operatorname{diag}\bigl(\mu_{2,1}^\star,\dots,\mu_{2,d}^\star\bigr)\tU^\top,
$
where the eigenvalues of the optimal projection $\tWla^\star, \tWlb^\star$ obey:
\begin{equation}
    \mu_{1,i}^\star\mu_{2,i}^\star=
    \begin{cases}
        0, & \sigma_i > 2\lambda,\\
        \sqrt{1-\dfrac{\sigma_i}{2\lambda}}, & \sigma_i\leq 2\lambda.
    \end{cases}
\end{equation}

This completes the proof.
\end{proof}

\subsection{Trade-off Improvement}
\label{subsec:theorem1}

Since both cases in \cref{subsec:theorem1} yield similar spectral effects, we conduct the theoretical analysis based on a linear layer.
To this end, we first provide the justification of the existence of the trade-off between temporal consistency and semantic separability.
Subsequently, we define the distance-based metrics to quantify the two competing objectives and then present a theorem that reveals how the margin evolves after applying the optimal projection.

\begin{lem}[Trade-off between temporal consistency and semantic separability] For the objective $M(\tW)$ consisting of a temporal consistency term and a semantic separability term, the gradients of these two terms induce opposing directions in a certain parameter space. This misalignment indicates an inherent trade-off between temporal consistency and semantic separability when optimizing $M(\tW)$.
\label{lem:lemma1}
\end{lem}

\begin{proof}
According to \cref{eq:intermedia_1}, the objective $M({\tW})$ of our method can be derived as:
\begin{equation}
\begin{aligned}
M({\tW})
&=\underbrace{\frac12\mathbb{E}_{{z}_1,{z}_2}\bigl[\lVert {\bm{W}}{z}_1-{\bm{W}}{z}_2 \rVert^{2}\bigr]}_{\text{Temporal Consistency}}
+\underbrace{\frac{\lambda}{2}\lVert {\bm{W}}{\bm{W}}^{\top}-{I}\rVert_{F}^{2}}_{\text{Semantic Separability}} \\
&=\frac12 \text{Tr}(\tLambdaW^2\tLambdaSigma) + \frac{\lambda}{2} \text{Tr}( {\tLambdaW}^4 -2\tLambdaW^2)+\frac{\lambda d}{2} \\
&=\frac12 \sum_{i=1}^{d}({\mu_i}^2\sigma_i)+ \frac{\lambda}{2} \sum_{i=1}^{d}( {\mu_i}^4 -2{\mu_i}^2)+\frac{\lambda d}{2}.
\end{aligned}
\end{equation}

According to the \cref{assmp:psd_1} , $\tW$ and $\tSigma$ are positive semi-definite, implying that $\mu_i$ and $\sigma_i$ are non-negative. Therefore, by differentiating $M(\bm{W})$ w.r.t. $\mu_i$ gives:
\begin{equation}
\begin{aligned}
\frac{\partial M}{\partial\mu_i}
&=\underbrace{\mu_i\sigma_i}_{\text{Temporal Consistency}} + \underbrace{2\lambda\mu_i^3 - 2\lambda\mu_i}_{\text{Semantic Separability}}.
\label{eq:grad}
\end{aligned}
\end{equation}

Based on the formulation in \cref{eq:grad}, the gradient of these two derived terms can be inferred as:

\begin{itemize}
  \item Temporal Consistency: $\mu_i\sigma_i \geq 0$.
  \item Semantic Separability: $2\lambda\mu_i^3 - 2\lambda\mu_i 
  = 2\lambda\mu_i(\mu_i+1)(\mu_i-1) < 0$ when $\mu_i < 1$.
\end{itemize}

Therefore, the two terms may change in opposite directions during optimization, since updates that increase temporal consistency tend to decrease semantic separability in the same eigen-direction, and vice versa. This behavior reflects an inherent trade-off between temporal consistency and semantic separability in our objective. A similar argument can be made for the trade-off in the MLP case.

\end{proof}

\begin{defn}[Intra-video Distance]
\label{defn:intra_distance} 
Define the intra-video distance as $D_{intra}(\tza,\tzb) = \mathbb{E}_{\tza, \tzb}\left[\lVert \tza - \tzb\rVert^2\right]$, which measures the average distance between temporally corresponding patches within a video.
\end{defn}

\begin{defn}[Inter-video Distance]
\label{defn:inter_distance} 
Define the inter-video distance as $D_{inter}(\tza,\tzb) = \mathbb{E}_{\tbarza, \tbarzb}\left[\lVert \tbarza - \tbarzb\rVert^2\right]$, calculating the average distance between video-level representations, where $\tbarzi = \mathbb{E}_{\bm{z} \in f(\bm{V}_i)} \left[\bm{z}\right]$ is the mean representation of the video $\bm{V}_i$. This reflects the average distance between different video-level representations.
\end{defn}

\begin{defn}[Distance Margin]
\label{defn:margin} 
Define the margin of these two metrics as $D(\tza,\tzb) = D_{inter}(\tza,\tzb)-\gamma D_{intra}(\tza,\tzb)$, reflecting the degree of separation between the two properties, where a larger value indicates a better trade-off between the two objectives.
\end{defn}

\begin{assmp}[Mean Eigenvalue Approximation]
\label{assmp:eigen_1} 
The eigenvalues of the inter-video covariance matrix approximate the average of those of the intra-video covariance matrix, \ie, $\forall j,\ \tau_j = \frac{1}{d} \sum_{i=1}^d \sigma_i$.
\end{assmp}

\begin{thm}[Trade-off Improvement, Formal]
\label{thm:thm2}
Let 
$\tSigma = \mathbb{E}_{\tza,\tzb}\bigl[(\tza-\tzb)(\tza-\tzb)^\top\bigr]$, $\tbarSigma = \mathbb{E}_{\tbarza,\tbarzb}\bigl[(\tbarza-\tbarzb)(\tbarza-\tbarzb)^\top\bigr]$ be the intra-video and inter-video covariance matrices, with eigenvalues $\{\sigma_i\}_{i=1}^d$ and $\{\tau_i\}_{i=1}^d$, respectively.  
Assume symmetric matrices $\tW$, $\tSigma$, and $\tbarSigma$ are positive semi-definite and mutually commuting, and that $\forall j,\ \tau_j = \frac{1}{d} \sum_{i=1}^d \sigma_i = \tau$. Let $\{\mu_{i}\}_{i=1}^d$ be the eigenvalues of $\tW$.
For the linear projection $g(\bm{z}) = \tW \bm{z}$, where the optimal eigenvalues are given by $\mu_i^\star = \sqrt{1 - \dfrac{\sigma_i}{2\lambda}}$ for $\sigma_i \leq 2\lambda$, the improvement in the margin metric is: given by:
\begin{equation}
\begin{aligned}
\Delta &= D(g(\tza),g(\tzb))-D(\tza,\tzb) \\
&= \sum_{\sigma_i \le 2\lambda} (\tau - \sigma_i)\left(1 - \dfrac{\sigma_i}{2\lambda}\right) > 0.
\end{aligned}
\end{equation}
\end{thm}

\begin{proof}

The margin metric of intra-video distance between the projected representations and the original representations can be derived as:
\begin{equation}
\begin{aligned}
\Delta_{intra} &= D_{intra}(\tW\tza,\tW\tzb) - \gamma D_{intra}(\tza,\tzb) \\
&= \mathbb{E}_{\tza, \tzb}\left[\lVert \tW\tza - \tW\tzb\rVert^2\right]  \\ 
&\quad - \gamma \mathbb{E}_{\tza, \tzb}\left[\lVert \tza - \tzb\rVert^2\right] \\
&= \mathbb{E}_{\tza, \tzb}\left[ (\tW\tza - \tW\tzb)^\top(\tW\tza - \tW\tzb) \right] \\
&\quad - \gamma  \mathbb{E}_{\tza, \tzb}\left[ (\tza - \tzb)^\top(\tza - \tzb) \right]  \\
&=  \mathbb{E}_{\tza, \tzb}\left[  \Tr ( \tW (\tza - \tzb)(\tza - \tzb)^\top \tW^\top )  \right]  \\
&\quad - \gamma   \mathbb{E}_{\tza, \tzb}\left[  \Tr ( (\tza - \tzb)(\tza - \tzb)^\top )  \right] \\
&=   \Tr ( \tW^\top\tW \mathbb{E}_{\tza, \tzb}\left[   (\tza - \tzb)(\tza - \tzb)^\top \right] ) \\
&\quad -  \gamma  \Tr ( \mathbb{E}_{\tza, \tzb}\left[ (\tza - \tzb)(\tza - \tzb)^\top \right] )   \\
&= \Tr(\tLambdaW^\top\tLambdaW \tLambdaSigma) - \gamma  \Tr(\tLambdaSigma) \\
&= \sum_{i=1}^{d} (\mu_i^{2}-\gamma) \sigma_i.
\end{aligned}
\end{equation}

The margin metric of inter-video distance between the projected representations and the original representations can be derived as:
\begin{equation}
\begin{aligned}
\Delta_{inter} &= D_{inter}(\tW\tza,\tW\tzb) -  D_{inter}(\tza,\tzb) \\
&= \mathbb{E}_{\tbarza, \tbarzb}\left[\lVert \tW\tbarza - \tW\tbarzb\rVert^2\right] \\
&\quad - \gamma  \mathbb{E}_{\tbarza, \tbarzb}\left[\lVert \tbarza - \tbarzb\rVert^2\right] \\
&= \mathbb{E}_{\tbarza, \tbarzb}\left[ (\tW\tbarza - \tW\tbarzb)^\top(\tW\tbarza - \tW\tbarzb) \right] \\
&\quad - \gamma  \mathbb{E}_{\tbarza, \tbarzb}\left[ (\tbarza - \tbarzb)^\top(\tbarza - \tbarzb) \right]  \\
&=  \mathbb{E}_{\tbarza, \tbarzb}\left[  \Tr ( \tW (\tbarza - \tbarzb)(\tbarza - \tbarzb)^\top \tW^\top )  \right] \\
&\quad - \gamma  \mathbb{E}_{\tbarza, \tbarzb}\left[  \Tr ( (\tbarza - \tbarzb)(\tbarza - \tbarzb)^\top )  \right] \\
&=   \Tr ( \tW^\top\tW \mathbb{E}_{\tbarza, \tbarzb}\left[   (\tbarza - \tbarzb)(\tbarza - \tbarzb)^\top \right] ) \\
&\quad - \gamma  \Tr ( \mathbb{E}_{\tbarza, \tbarzb}\left[ (\tbarza - \tbarzb)(\tbarza - \tbarzb)^\top \right] )   \\
&= \Tr(\tLambdaW^\top\tLambdaW \tLambdabarSigma) - \gamma  \Tr(\tLambdabarSigma) \\
&= \sum_{i=1}^{d} (\mu_i^{2}-\gamma) \tau_i.
\end{aligned}
\end{equation}

Then the improvement of the margin metrics can be formulated as:

\begin{equation}
\begin{aligned}
\Delta &= (D_{inter}(\tW\tza,\tW\tzb) - D_{intra}(\tW\tza,\tW\tzb)) \\
&\quad - (D_{inter}(\tza,\tzb) - D_{intra}(\tza,\tzb)) \\
&= (D_{inter}(\tW\tza,\tW\tzb) - D_{inter}(\tza,\tzb)) \\
&\quad - (D_{intra}(\tW\tza,\tW\tzb) - D_{intra}(\tza,\tzb)) \\
&= \sum_{i=1}^{d} (\mu_i^{2}-\gamma) \tau_i - \sum_{i=1}^{d} (\mu_i^{2}-\gamma) \sigma_i \\
&= \sum_{i=1}^{d} (\mu_i^{2}-\gamma) (\tau_i - \sigma_i).
\end{aligned}
\end{equation}

Under \cref{assmp:sym_1,assmp:psd_1,assmp:commuting_1}, the optimal linear projection $\tW^\star= \tU\operatorname{diag}\bigl(\mu_1^\star,\dots,\mu_d^\star\bigr)\tU^\top$ has eigenvalues given by:
\begin{equation}
    \mu_i^\star=
    \begin{cases}
        0, & \sigma_i \geq 2\lambda,\\
        \sqrt{1-\dfrac{\sigma_i}{2\lambda}}, & \sigma_i < 2\lambda.
    \end{cases}
\end{equation}

Under \cref{assmp:eigen_1}, due to $\forall j, \tau_j = \frac 1 d \sum_{i=1}^d\sigma_i = \tau$, we have $\sum_{i=1}^d (\tau - \sigma_i) = 0$.

By substituting $\mu_i^\star=\sqrt{1-\dfrac{\sigma_i}{2\lambda}}$ and $\sum_{i=1}^d (\tau - \sigma_i) = 0$ under the condition $\sigma_i < 2\lambda$, the change in the margin metric can be expressed as:
\begin{equation}
\Delta = \sum_{i=1}^{d} (\mu_i^{2}-\gamma) (\tau_i - \sigma_i) = \sum_{\sigma_i < 2\lambda} (\tau - \sigma_i)\left(1 - \frac{\sigma_i}{2\lambda}\right).
\label{eq:final_delta}
\end{equation}

When $\lambda < \dfrac{\tau}{2}$ (\ie, $2\lambda < \tau$), all $\sigma_i \le 2\lambda$ necessarily obey $\sigma_i < \tau$. In this case, each term in \cref{eq:final_delta} satisfies $(\tau - \sigma_i)\left(1 - \dfrac{\sigma_i}{2\lambda}\right)>0$ because:
\begin{itemize}[leftmargin=*]
    \item $\tau - \sigma_i > 0$ follows from $\sigma_i < \tau$,
    \item $1 - \dfrac{\sigma_i}{2\lambda} \ge 0$ since $\sigma_i \le 2\lambda$.
\end{itemize}
Therefore, $\Delta > 0$ holds whenever $\lambda < \dfrac{1}{2d} \sum_{i=1}^d \sigma_i$.
\end{proof}

In summary, this section provides a theoretical analysis of the trade-off between temporal consistency and semantic separability, leading to the following two key insights:

\begin{enumerate}[leftmargin=*]
    \item \textbf{\cref{thm:thm1}} shows that linear projection exhibits similar behavior to shallow MLPs in adjusting representations, yielding comparable effects in similar feature scaling behavior as the linear layer.
    \item \textbf{\cref{thm:thm2}} demonstrates that under optimal conditions, a linear projection is sufficient to improve the trade-off between temporal consistency and semantic separability.
\end{enumerate}

\section{Supplementary Explanation of Method}
\label{sec:method}

\subsection{Differences with Previous Methods}
\label{subsec:differences}

\begin{figure*}[htbp]
  \centering

  \begin{subfigure}[b]{0.46\linewidth}
    \includegraphics[width=\linewidth]{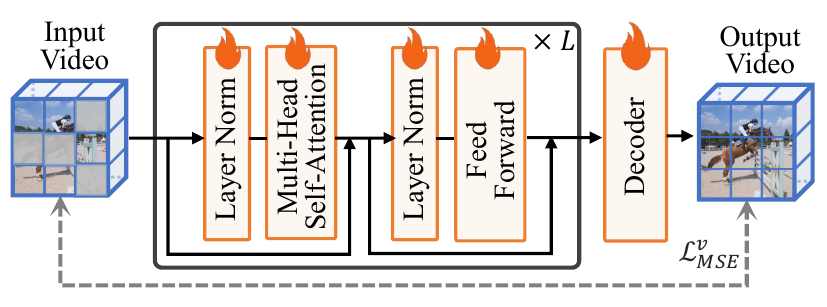}
    \caption{Video-pretrained methods~\cite{VideoMAE,MAE_ST}.}
  \end{subfigure}
  \begin{subfigure}[b]{0.46\linewidth}
    \includegraphics[width=\linewidth]{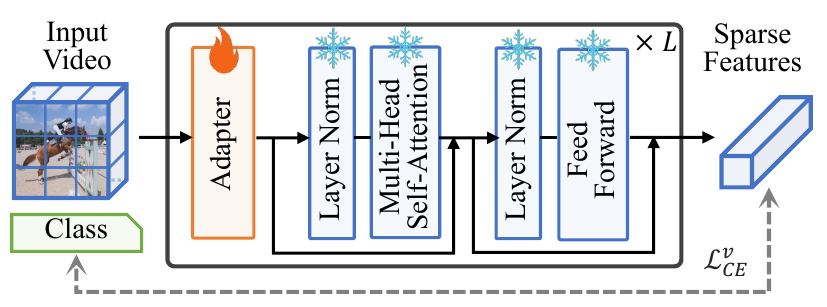}
    \caption{Supervised adaptation methods~\cite{adapter2,adapter3}.}
  \end{subfigure}

  \begin{subfigure}[b]{0.46\linewidth}
    \includegraphics[width=\linewidth]{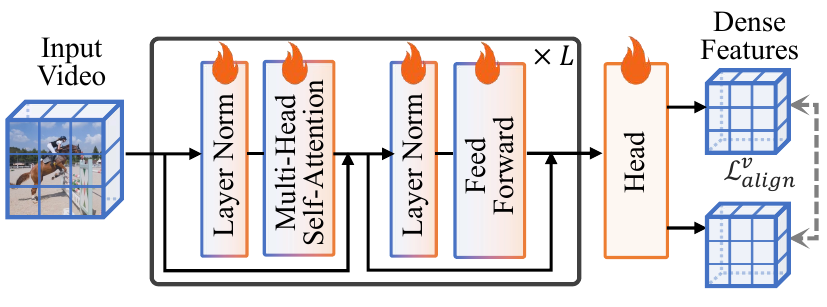}
    \caption{Video fine-tuning methods~\cite{two_stage_1,two_stage_2}.}
  \end{subfigure}
  \begin{subfigure}[b]{0.46\linewidth}
    \includegraphics[width=\linewidth]{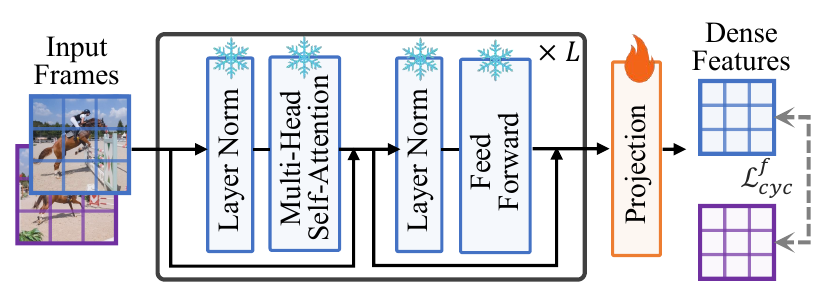}
    \caption{Image-to-video transfer method (Ours).}
  \end{subfigure}

  \caption{Comparison of several categories of video representation learning methods with ours.}
  \label{fig:compare_methods}
\end{figure*}

In \cref{fig:compare_methods}, we provide a comparative overview of several categories of video representation learning works alongside our method.

\textbf{1) Video-pretrained methods} extend the masked image modeling paradigm to the video domain by masking 3D volumes and reconstructing raw pixels for spatiotemporal learning~\cite{VideoMAE,VideoMAEv2,MAE_ST}. Subsequent variants incorporate conditional frames to enhance temporal modeling~\cite{SiamMAE,CropMAE,RSP,tcore}. These approaches typically require large-scale video pretraining from scratch, incurring substantial computational cost due to video redundancy and pixel-level reconstruction overhead.

\textbf{2) Supervised adaptation methods} adapt Vision Transformers pretrained with CLIP~\cite{CLIP} by inserting lightweight adapters in serial or parallel configurations~\cite{adapter1,adapter2,adapter3,adapter4,adapter5}. These adapters are usually trained on supervised action recognition datasets~\cite{Kinetics,ssv2}, making them highly task-dependent and less generalizable without additional task-specific fine-tuning.

\textbf{3) Video fine-tuning methods} follow a two-stage training scheme: models are first pretrained on task-specific datasets to learn static features for instance-level discrimination, then fine-tuned on video datasets with additional temporal branches introduced to handle motion reasoning~\cite{two_stage_1,two_stage_2,two_stage_3,two_stage_4}. Although it can perform well on specific video tasks, its increased model complexity and training cost make it difficult to perform fast cross-domain transfer.

\textbf{4) Our image-to-video transfer method} takes a different approach by leveraging pretrained image representations and adapting them to video tasks via structure-preserving projection. The main advantages are as follows:
\begin{itemize}[leftmargin=*]
    \item \textbf{Efficient transfer}: We sample two frames per video and insert a lightweight linear-based projection head after a frozen image encoder, enabling fast transfer with reduced temporal and spatial cost.
    \item \textbf{Joint optimization}: We simultaneously optimize temporal consistency and semantic separability via a temporal cycle-consistency objective and a semantic separability regularization term.
    \item \textbf{Label-free training}: Our method is fully self-supervised, requiring no manual annotations, which enhances scalability and promotes better generalization across diverse video tasks of different granularity.
\end{itemize}

\subsection{Algorithm of the Framework}
\label{subsec:algorithm}

The complete optimization procedure of our framework is summarized in \cref{alg:algo}. The batch-level for-loop can be implemented
via matrix operations to reduce computational burden.

\begin{algorithm}[tb]
\caption{Consistency-Separability Trade-off Transfer Learning Algorithm}
\label{alg:algo}
\LinesNumbered
\KwIn{Unlabeled dataset $\mathcal{D}$, number of iterations $L$, interpolation ratio $\alpha$, constraint weight $\lambda$.}
\KwOut{Parameters $\bm{\theta}_{L+1}$ of projection layer $g$.}
Initialize parameters $\bm{\theta}_{1}$ for $g$.

\For{$l=1$ {\bfseries to} $L$}{
    Sample a batch of videos $\{\bm{V}_i\}_{i=1}^B$.

    \For{$i=1$ {\bfseries to} $B$}{
        \textcolor{cvprblue}{$\triangleright$ Temporal Correspondence Establishment}
        
        Select frames $\vta^f$, $\vtb$, $\vta^b$ from $\bm{V}_m$ and prepare position encoding $\mathbf{E}_{\text{pos}}$ and $\widetilde{\mathbf{E}}_{\text{pos}}$.
        
        Extract representations $\zta^f$, $\ztb$, $\ztatilde^b$ with $f$ and projections $\pta^f$, $\ptb$, $\ptatilde^b$ via $g$.

        Calculate correlation matrices $\Atatb$ and $\Atbtatilde$.

        \textcolor{cvprblue}{$\triangleright$ Temporal Consistency and Semantic Separability Trade-off}

        Enhance temporal consistency of $\pta^f$, $\ptb$, $\ptatilde^b$ via $\mathcal{L}_{cyc}$.

        Align the semantic separability of $ \{ (\pta^f, \zta^f), (\ptb, \ztb), (\ptatilde^b, \ztatilde^b) \}$ by $\mathcal{L}_{reg}$.

    }

    Update the projection layer $g$ with $\mathcal{L}_{total} = \mathcal{L}_{cyc} + \lambda\mathcal{L}_{reg}$.
}
\Return

\end{algorithm}

\section{Detailed Description of Experiments}
\label{sec:implementation}

\subsection{Training Datasets}
\label{subsec:training_datasets}

\textbf{Kinetics-400}~\cite{Kinetics} is a widely used large-scale video benchmark comprising 400 human action categories collected from YouTube. It provides 239,789 trimmed video clips, each lasting around 10 seconds, making it suitable for various video understanding tasks. In our experiments, we sample video frames at $2$ FPS for pretraining to reduce redundancy while retaining sufficient temporal cues. In this work, unless otherwise noted, all the models equipped with our method are trained for 5 epochs using the Kinetics-400 training set.

\textbf{SSV2} (Something-Something V2)~\cite{ssv2} is a large-scale video benchmark emphasizing human-object interactions and temporal reasoning. It comprises 220,847 short, crowdsourced clips across 174 action classes, with each clip lasting a few seconds, making it well-suited for evaluating temporal understanding beyond appearance cues. In our experiments, we only use SSV2 in the ablation study on training datasets, training for 5 epochs on the SSV2 training split.

\textbf{ImageNet-1k}~\cite{Imagenet} is a well-known image dataset containing over 1.28 million training images across 1,000 real-world object categories. It has played a central role in the development of deep visual representation learning and serves as the pretraining corpus for most high-performance image encoders. In this work, most of the image models are already pretrained on this dataset, providing a strong foundation of semantic separability.

\textbf{WIT-400M} (WebImageText)~\cite{CLIP} is a large-scale web-crawled dataset consisting of 400 million image-text pairs, designed to support vision-language pretraining. The dataset was constructed using 500,000 diverse natural language queries to guide image-text pair retrieval, with up to 20,000 pairs per query to encourage approximate class balancing. Its overall scale and linguistic richness make it suitable for training multimodal models such as CLIP~\cite{CLIP}.

\textbf{LAION-400M}~\cite{LAION} is a large-scale dataset of 400 million image-text pairs designed to support vision-language pretraining. The image-text pairs are extracted from Common Crawl web pages and filtered using CLIP-based similarity to retain pairs with stronger semantic alignment between images and captions. It is used as a pretraining dataset for vision-language models, including BLIP~\cite{BLIP}.

\subsection{Training Settings}
\label{subsec:training_settings}

During training, we freeze the pretrained image encoder $f$ and update only the projection layer $g$. The training is performed on the Kinetics-400 dataset for 5 epochs with a total batch size of 512, using the first epoch for learning rate warm-up. 
We employ the AdamW optimizer~\cite{AdamW} with a cosine learning rate decay schedule. The base learning rate is set to $blr = 1 \times 10^{-4}$ and scaled according to the batch size as $lr = blr / 256$.
For each video clip, two frames are randomly sampled with a temporal interval of $\delta = 0.15$ relative to the total video length. The softmax temperature is set to $\tau = 0.03$. The output dimension of the projection head $g$ is set to $d = 768$ for ViT-Base backbones.
Detailed hyperparameter settings for training and method components are summarized in~\cref{tab:hyper_1} and~\cref{tab:hyper_2}. All experiments are implemented in PyTorch~\cite{pytorch} and conducted on a Linux server equipped with an AMD EPYC 9654 96-Core CPU and 4 NVIDIA RTX4090 GPUs.

\begin{table}[htbp]
  \centering
  \caption{Summary of hyperparameter settings used during training and evaluation.}
  \vspace{-6pt}
  \label{tab:hyper_summary}

  \begin{subtable}[t]{0.85\linewidth}
    \centering
    \caption{Training hyperparameters.}
    \label{tab:hyper_1}
    \setlength{\tabcolsep}{12pt}
    \resizebox{\linewidth}{!}{%
      \begin{tabular}{lc|c}
        \toprule
        Hyperparameter & Notation & Value \\
        \midrule
        Image size & $H\times W$ & $224\times224$ \\
        Patch size & $p$ & 16 \\
        Optimizer & / & AdamW \\
        Scheduler & / & Cosine \\
        Weight decay & / & $0.05$ \\
        Momentum & $\beta_1,\beta_2$ & $0.9, 0.95$ \\
        Base learning rate & $blr$ & $1 \times 10^{-4}$ \\
        Epochs & / & 5 \\
        Warm-up Epoch & / & 1 \\
        Batch size & $bs$ & 512 \\
        \bottomrule
      \end{tabular}
    }
  \end{subtable}

  \vspace{6pt}

  \begin{subtable}[t]{0.85\linewidth}
    \centering
    \caption{Method hyperparameters.}
    \label{tab:hyper_2}
    \setlength{\tabcolsep}{12pt}
    \resizebox{\linewidth}{!}{%
      \begin{tabular}{lc|c}
        \toprule
        Hyperparameter & Notation & Value \\
        \midrule
        Temperature of Softmax & $\tau$ & 0.03 \\
        Frame sampling interval  & $\delta$ & 0.15 \\
        Feature dim of $g$ & $d$ & 768 \\
        \bottomrule
      \end{tabular}
    }
  \end{subtable}

  \vspace{6pt}

  \begin{subtable}[t]{0.85\linewidth}
    \centering
    \caption{Evaluation hyperparameters.}
    \label{tab:hyper_3}
    \resizebox{\linewidth}{!}{%
      \begin{tabular}{c|ccc}
        \toprule
        Hyperparameter & DAVIS-2017 & VIP & JHMDB \\
        \midrule
        Image size & $480\times880$ & $480\times880$ & $320\times320$ \\
        Top-K & 7 & 10 & 7 \\
        Queue Length & 20 & 20 & 20 \\
        Neighborhood Size & 20 & 20 & 20 \\
        \bottomrule
      \end{tabular}
    }
  \end{subtable}
\end{table}

\subsection{Evaluation Settings}
\label{subsec:evaluation_settings}

\subsubsection{Evaluation on Dense-level Benchmarks}
\label{subsubsec:evaluation_dense_supp}

We first evaluate the representations on three dense video downstream tasks: video object segmentation on DAVIS-2017~\cite{DAVIS17}, human part segmentation on VIP~\cite{VIP}, and human pose propagation on JHMDB~\cite{JHMDB}. 
Following previous works~\cite{SiamMAE,CropMAE,RSP,tcore}, all tasks are evaluated under a semi-supervised protocol in which the ground-truth mask of the first frame is given, and the model propagates predictions to subsequent frames without any task-specific fine-tuning.
The hyperparameters used for each evaluation task are listed in~\cref{tab:hyper_3}. To ensure fair comparisons, we keep the evaluation hyperparameter settings fixed across all methods and tasks without additional tuning.

\textbf{DAVIS-2017}~\cite{DAVIS17} is a widely used benchmark for video object segmentation. We report three standard metrics to assess overall segmentation quality:

\noindent\textbf{1)} $\mathcal{J}_{\mathrm{m}}$ (region similarity) computes the average IoU between the predicted mask $P_i$ and the ground-truth mask $G_i$ across all videos $\bm{V}_i$:
\begin{equation}
    \mathcal{J}_{\mathrm{m}} = \frac{1}{n} \sum_{i=1}^{n} \frac{|P_i \cap G_i|}{|P_i \cup G_i|}.
\label{eq:J_mean}
\end{equation}

\noindent\textbf{2)} $\mathcal{F}_{\mathrm{m}}$ (contour accuracy) evaluates the alignment between the predicted and ground-truth boundaries by calculating the harmonic mean of precision $Pre_i$ and recall $Rec_i$:
\begin{equation}
    \mathcal{F}_{\mathrm{m}} = \frac{1}{n} \sum_{i=1}^{n} \frac{2 \cdot Pre_i \cdot Rec_i}{Pre_i + Rec_i}.
\label{eq:F_mean}
\end{equation}

\noindent\textbf{3)} $\mathcal{J}\&\mathcal{F}_{\mathrm{m}}$ provides an overall performance measure by averaging $\mathcal{J}_{\mathrm{m}}$ and $\mathcal{F}_{\mathrm{m}}$:
\begin{equation}
    \mathcal{J}\&\mathcal{F}_{\mathrm{m}} = \frac{\mathcal{J}_{\mathrm{m}} + \mathcal{F}_{\mathrm{m}}}{2}.
\label{eq:JF_mean}
\end{equation}

\textbf{VIP}~\cite{VIP} focuses on fine-grained human part segmentation and is used to evaluate semantic part propagation. The main evaluation metric is the mIoU computed by averaging the IoU across all classes $C_j$ and all videos $\bm{V}_i$:
\begin{equation}
    \text{mIoU} = \frac{1}{|C|} \sum_{j=1}^{|C|} \frac{1}{n} \sum_{i=1}^{n} \frac{|P_{i,j} \cap G_{i,j}|}{|P_{i,j} \cup G_{i,j}|}.
\label{eq:mIoU}
\end{equation}

\textbf{JHMDB}~\cite{JHMDB} is commonly used for human pose estimation. We adopt it for the pose propagation task and evaluate performance using the PCK@$k$ metric, which measures the proportion of keypoints predicted within a normalized distance threshold:
\begin{equation}
    \text{PCK@}k = \frac{1}{n} \sum_{i=1}^{n} \frac{1}{|S_i|} \sum_{j=1}^{|S_i|} \mathbbm{1} \left[ D(\hat{p}_{i,j}, p_{i,j}) < k \cdot d_i \right],
\label{eq:PCK}
\end{equation}
where $S_i$ is the keypoint set in video $\bm{V}_i$, $d_i$ denotes the scale of the human body, $D(\hat{p}_{i,j}, p_{i,j})$ is the Euclidean distance between the predicted and ground-truth positions, and $k$ is the threshold for the maximum allowable distance error. We report PCK@0.1 and PCK@0.2 in our experiments.

\subsubsection{Evaluation on Frame-/Video-level Benchmarks}
\label{subsubsec:evaluation_frame_video_supp}

We further evaluate the transferred models on several frame-level and video-level downstream tasks: temporal action localization on Breakfast~\cite{Breakfast}, video retrieval on UCF101 and HMDB51~\cite{UCF101,HMDB51}, and action classification on Something-Something-v2 (SSV2)~\cite{ssv2}.

\textbf{Breakfast}~\cite{Breakfast} contains 1,712 untrimmed videos with frame-level annotations of fine-grained actions. We perform temporal action localization on this dataset by extracting frame-wise representations with our transferred image-to-video model and training the FACT~\cite{FACT} backbone on these representations. Following a standard protocol, we train FACT on \textit{split2-4} and evaluate on \textit{split1}. We report three standard metrics as follows:

\noindent\textbf{1) $\mathrm{Edit}$} measures sequence-level similarity between the predicted and ground-truth label sequences after collapsing consecutive duplicates. Let $\mathbf{y}=(y_{1},\ldots,y_{T})$ and $\hat{\mathbf{y}}=(\hat{y}_{1},\ldots,\hat{y}_{T})$ be frame-wise labels, and let $\mathcal{C}(\cdot)$ collapse consecutive identical labels. Denote $\mathrm{Lev}(\cdot,\cdot)$ as the Levenshtein distance and $|\cdot|$ as the sequence length. The normalized edit score is
\begin{equation}
\mathrm{Edit} \;=\; 1 - \frac{\mathrm{Lev}\!\bigl(\mathcal{C}(\hat{\mathbf{y}}),\,\mathcal{C}(\mathbf{y})\bigr)}{\max\!\bigl\{\,|\mathcal{C}(\hat{\mathbf{y}})|,\,|\mathcal{C}(\mathbf{y})|\,\bigr\}}.
\label{eq:Edit}
\end{equation}

\noindent\textbf{2) $\mathrm{Acc}$} is the frame-wise accuracy, representing the percentage of correctly labeled frames:
\begin{equation}
\mathrm{Acc} \;=\; \frac{1}{T}\sum_{t=1}^{T}\mathbbm{1}\!\left\{\hat{y}_{t}=y_{t}\right\},
\label{eq:Acc}
\end{equation}
where $\mathbbm{1}\{\cdot\}$ is the indicator function.

\noindent\textbf{3) $\mathrm{F1@k}$} is the segmental $\mathrm{F1}$ at $\mathrm{IoU}$ threshold $k$. Let the ground-truth segment set be $\mathcal{S}=\{(s^{g}_{j},e^{g}_{j},c^{g}_{j})\}$ and the predicted set $\hat{\mathcal{S}}=\{(s^{p}_{i},e^{p}_{i},c^{p}_{i})\}$, where $s/e$ are start/end frames and $c$ is the class. For segments of the same class, define the temporal Intersection-over-Union as:
\begin{equation}
\mathrm{IoU}\bigl((s^{p}_{i},e^{p}_{i}),(s^{g}_{j},e^{g}_{j})\bigr)
=\frac{\bigl\{\min(e^{p}_{i},e^{g}_{j})-\max(s^{p}_{i},s^{g}_{j})\bigr\}_+}
{\max(e^{p}_{i},e^{g}_{j})-\min(s^{p}_{i},s^{g}_{j})}.
\label{eq:IoU}
\end{equation}
A prediction is a true positive (\textit{TP}) if it uniquely matches a ground-truth segment of the same class with $\mathrm{IoU}\ge k$; unmatched predictions are false positives (\textit{FP}), and unmatched ground-truth segments are false negatives (\textit{FN}). With precision $\mathit{P}=\frac{\mathit{TP}}{\mathit{TP}+\mathit{FP}}$ and recall $\mathit{R}=\frac{\mathit{TP}}{\mathit{TP}+\mathit{FN}}$, we compute
\begin{equation}
\mathrm{F1@k} \;=\; \frac{2\,\mathit{P}\cdot\mathit{R}}{\mathit{P}+\mathit{R}},
\label{eq:F1@k}
\end{equation}
where $k\in\{0.10,0.25,0.50\}$ as standard thresholds.

\textbf{UCF101}~\cite{UCF101} comprises 13,320 videos from 101 human action classes, and \textbf{HMDB51}~\cite{HMDB51} contains 6,766 videos from 51 action classes. For zero-shot video retrieval on the test set, we directly extract video representations using our transferred image-to-video model and perform retrieval following~\cite{liu2024not}: in each query round, one video is treated as the query and all remaining videos form the reference set. This process is repeated for every video. And we report the following metrics with the average.

\noindent\textbf{1) $\mathrm{mAP}$} (Mean Average Precision) is the mean of per-query Average Precision (AP). Let $|\mathcal{Q}|$ be the number of queries, $n_j$ the number of positives for query $j$, and $r_i$ the rank of the $i$-th retrieved positive for that query. Then
\begin{equation}
\mathrm{mAP} \;=\; \frac{1}{|\mathcal{Q}|}\sum_{j=1}^{|\mathcal{Q}|}\,\frac{1}{n_j}\sum_{i=1}^{n_j}\frac{i}{r_i}.
\label{eq:map_def}
\end{equation}

\noindent\textbf{2) $\mathrm{Recall@K}$} is the fraction of queries for which at least one positive appears in the top-$K$ results. Let $\mathcal{R}_j^{(K)}$ be the set of ranks $\le K$ among retrieved items for query $j$, and let $\mathcal{P}_j$ be the set of ranks of its positives. Then
\begin{equation}
\mathrm{Recall@}K \;=\; \frac{1}{|\mathcal{Q}|}\sum_{j=1}^{|\mathcal{Q}|}\mathbbm{1}\!\{\min(\mathcal{P}_j)\le K\}.
\label{eq:recall_def}
\end{equation}

\textbf{Something-Something-v2 (SSV2)}~\cite{ssv2} is a large-scale action classification benchmark consisting of 220,847 short videos from 174 fine-grained action categories without public labels. It focuses on human-object interactions with subtle motion variations, and is widely used to evaluate a model's capability for temporal reasoning and motion-sensitive action understanding.

For action classification on SSV2, each transferred image-to-video model is fine-tuned on the training set for 25 epochs and then evaluated on the validation set using single-clip sampling. Although this protocol is lighter than commonly used longer-schedule or multi-clip settings, it is applied uniformly to all compared methods to ensure effective and fair comparison. We report the standard top-$k$ accuracy metric:

\noindent\textbf{$\mathrm{Acc@}k$} measures the percentage of validation videos whose ground-truth label appears among the top-$k$ predicted classes. Let $\mathbf{z}^{(i)} \in \mathbb{R}^{C}$ be the predicted logits for the $i$-th video over $C$ classes, and let $y_i \in \{1,\dots,C\}$ be the ground-truth label. Denote by $\pi_k(\mathbf{z}^{(i)})$ the set of indices corresponding to the top-$k$ largest entries in $\mathbf{z}^{(i)}$. Then
\begin{equation}
\mathrm{Acc@}k \;=\; \frac{1}{N}\sum_{i=1}^{N}\mathbbm{1}\!\left\{ y_i \in \pi_k\!\left(\mathbf{z}^{(i)}\right) \right\},
\label{eq:acck_ssv2}
\end{equation}
where $N$ is the number of validation videos and $\mathbbm{1}\{\cdot\}$ is the indicator function. Following common practice to present $\mathrm{Acc@1}$ and $\mathrm{Acc@5}$.

\textbf{Chiral SSV2}~\cite{Chirality} is a temporal order discrimination benchmark constructed from Something-Something-v2~\cite{ssv2}. It groups temporally opposite actions into chiral pairs, such as ``sitting down'' and ``standing up'', and evaluates whether a video representation is sensitive to the ordering of visual change over time. Compared with standard action classification, this benchmark places stronger emphasis on time-awareness rather than semantic categorization.

Following~\cite{Chirality}, we evaluate each model using a linear-probe protocol on frozen representations. Specifically, for each chiral group, we extract frame-level representations from each video, concatenate representations along the temporal dimension to form the video representation, and train a linear classifier for binary classification. This procedure is repeated independently for every chiral group, and the final result is reported as the average classification accuracy across all groups.

\noindent\textbf{$\mathrm{Acc}$} measures the percentage of correctly classified videos over all evaluation samples. Let $\mathbf{z}^{(i)} \in \mathbb{R}^{2}$ be the logits predicted by the linear classifier for the $i$-th video, and let $y_i \in \{0,1\}$ denote its ground-truth label within the corresponding chiral pair. Then
\begin{equation}
\mathrm{Acc} \;=\; \frac{1}{N}\sum_{i=1}^{N}\mathbbm{1}\!\left\{\arg\max_{c} z^{(i)}_{c} = y_i \right\},
\label{eq:acc_chiral_ssv2}
\end{equation}
where $N$ is the total number of evaluation videos and $\mathbbm{1}\{\cdot\}$ is the indicator function.

\subsubsection{Distance-based Trade-off Metrics Validation}
\label{subsubsec:distance_metrics_supp}
To provide an interpretable assessment, we validate the distance-based metrics proposed in Sec.4. We randomly sample 1,000 videos from the Kinetics-400 validation set and compute each metric for both the original image-pretrained models and our transferred models. All metrics are computed on the same sample set for a fair comparison. We report four metrics as follows.

\noindent\textbf{1) $D_{inter}$ } (Inter-video distance).
Let the sample set be $\mathcal{V}=\{V^{(n)}\}_{n=1}^{M}$ with $M=1000$. For each $V^{(n)}$, select the middle frame $v^{(n)}_{t^\ast}$ and extract $N$ patch representations $\{\bm{z}^{(n)}_{t^\ast}(i)\}_{i=1}^{N}$. For each unordered pair $(u,v)$ with $u<v$, define the pair-wise inter-video distance as:
\begin{equation}
d(u,v)\;=\;\frac{1}{N}\sum_{i=1}^{N}\Bigl\|\,\bm{z}^{(u)}_{t^\ast}(i)-\bm{z}^{(v)}_{t^\ast}(i)\,\Bigr\|_{2}.
\end{equation}
The unnormalized inter-video distance is
\begin{equation}
D_{inter}^{ori}\;=\;\frac{2}{M(M-1)}\sum_{1\le u<v\le M}d(u,v).
\end{equation}
Let the global \emph{center} be the mean patch representation over videos, $\bm{c}(i)=\frac{1}{M}\sum_{n=1}^{M}\bm{z}^{(n)}_{t^\ast}(i)$, and define each video's distance to the center as
\begin{equation}
r(u)\;=\;\frac{1}{N}\sum_{i=1}^{N}\Bigl\|\,\bm{z}^{(u)}_{t^\ast}(i)-\bm{c}(i)\,\Bigr\|_{2}.
\end{equation}
The inter-video radius is $R_{inter}=\max_{u}r(u)$, and the normalized metric is
\begin{equation}
D_{inter}\;=\;\frac{D_{inter}^{ori}}{2R_{inter}}.
\label{eq:inter_dist}
\end{equation}

\noindent\textbf{2) $D_{intra}$ } (Intra-video distance).
For each $V^{(n)}$, select a set of frame pairs $\mathcal{P}^{(n)}=\{(t_a,t_b)\}$. For a given pair, measure pair-wise intra-video distance as:
\begin{equation}
d^{(n)}(t_a,t_b)\;=\;\frac{1}{N}\sum_{i=1}^{N}\Bigl\|\,\bm{z}^{(n)}_{t_a}(i)-\bm{z}^{(n)}_{t_b}(i)\,\Bigr\|_{2}.
\end{equation}
The per-video unnormalized intra distance and its normalization radius are
\begin{equation}
\begin{aligned}
D_{intra}^{ori,(n)}\;&=\;\frac{1}{|\mathcal{P}^{(n)}|}\sum_{(t_a,t_b)\in\mathcal{P}^{(n)}}d^{(n)}(t_a,t_b),
\\
R_{intra}^{(n)}\;&=\;\max_{t}\,\frac{1}{N}\sum_{i=1}^{N}\Bigl\|\,\bm{z}^{(n)}_{t}(i)-\bar{\bm{z}}^{(n)}(i)\,\Bigr\|_{2},
\end{aligned}
\end{equation}
where $\bar{\bm{z}}^{(n)}(i)$ is the per-video mean patch representation over the frames used for $\mathcal{P}^{(n)}$. We normalize each video by its own radius and then average:
\begin{equation}
D_{intra}\;=\;\frac{1}{M}\sum_{n=1}^{M}\frac{D_{intra}^{ori,(n)}}{2R_{intra}^{(n)}}.
\label{eq:intra_dist}
\end{equation}

\noindent\textbf{3) $D$ } (Distance margin).
The trade-off margin balances the two normalized distances with a scale factor $\gamma$:
\begin{equation}
\begin{aligned}
D\;&=\;D_{inter}-\gamma\,D_{intra},
\\
\gamma\;&=\;\frac{\mathbb{E}_{\mathcal{M}}\!\bigl[D_{intra}^{ori}\bigr]}{\mathbb{E}_{\mathcal{M}}\!\bigl[D_{inter}^{ori}\bigr]},
\label{eq:margin}
\end{aligned}
\end{equation}
where $\mathcal{M}$ indexes the set of models under comparison.
Model-specific values of the scale factor $\gamma$ are listed in \cref{tab:avg_gamma} and concentrate within a narrow range. Therefore, to unify the setting, we use the average $\gamma=0.3$ in practice.

\noindent\textbf{4)} ${Cyc.\ Acc.}$ (Cycle-consistency accuracy).
Given two frames forming a palindrome traversal and $N$ patches per frame, let $\bm{A}_{t_a}^{t_b}$ and $\bm{A}_{t_b}^{{t}_a}$ be the patch-wise correlation transition matrices, and set $\bm{P}=\bm{A}_{t_a}^{t_b}\,\bm{A}_{t_b}^{{t}_a}$. The cycle-consistency accuracy is the proportion of patches returning to their original indices:
\begin{equation}
\textit{Cyc.\ Acc.}\;=\;\frac{1}{N}\sum_{i=1}^{N}\mathbbm{1}\!\left\{\arg\max_{j}P_{ij}=i\right\}.
\label{eq:cyc_acc}
\end{equation}

\begin{table}[htbp]
  \centering
  \caption{The scale factor $\gamma={D_{intra}^{ori}}/{D_{inter}^{ori}}$ and the average scale factor for each model.}
  \vspace{-6pt}
  \setlength{\tabcolsep}{12pt}
  \resizebox{0.48\textwidth}{!}{%
    \begin{tabular}{cc|cc}
    \toprule
    {Method} & {$\gamma$} & {Method} & {$\gamma$} \\
    \midrule
    MAE   & \multicolumn{1}{c|}{0.1855} & MoCov3 & 0.3321  \\
    MAE +\textbf{\textit{Ours}} & \multicolumn{1}{c|}{0.3289} & MoCov3 +\textbf{\textit{Ours}} & 0.3053  \\
    \midrule
    I-JEPA & \multicolumn{1}{c|}{0.2283} & iBOT  & 0.2817  \\
    I-JEPA +\textbf{\textit{Ours}} & \multicolumn{1}{c|}{0.2645} & iBOT +\textbf{\textit{Ours}} & 0.3084  \\
    \midrule
    CLIP  & \multicolumn{1}{c|}{0.3365} & DINO  & 0.3378  \\
    CLIP +\textbf{\textit{Ours}} & \multicolumn{1}{c|}{0.3876} & DINO +\textbf{\textit{Ours}} & 0.3505  \\
    \midrule
    BLIP  & \multicolumn{1}{c|}{0.2489} & DINOv2 & 0.2730  \\
    BLIP +\textbf{\textit{Ours}} & \multicolumn{1}{c|}{0.3737} & DINOv2 +\textbf{\textit{Ours}} & 0.2916  \\
    \midrule
    \multicolumn{4}{c}{\textbf{Average $\gamma$} : 0.3021} \\
    \bottomrule
    \end{tabular}%
    }
  \label{tab:avg_gamma}%
\end{table}%

\begin{table*}[htbp]
  \centering
  \caption{Evaluation results on frame-level and video-level tasks based on representative image models. The best results are marked in \textbf{bold}.}
  \resizebox{0.92\textwidth}{!}{%
    \begin{tabular}{c|ccccc|cc|cc|cc|c}
    \toprule
    \multirow{3}[1]{*}{Model} & \multicolumn{5}{c|}{\shortstack{Action \\ Localization}} & \multicolumn{4}{c|}{\shortstack{Video \\ Retrieval}} & \multicolumn{2}{c|}{\shortstack{Action \\ Classification}} & \multicolumn{1}{c}{\shortstack{Temporal Order \\ Discrimination}} \\
\cmidrule{2-13}          & \multicolumn{5}{c|}{Breakfast} & \multicolumn{2}{c|}{UCF101} & \multicolumn{2}{c|}{HMDB51} & \multicolumn{2}{c|}{SSV2} & \multicolumn{1}{c}{Chiral  SSV2} \\
\cmidrule{2-13}          & Edit  & Acc   & F1@0.10 & F1@0.25 & F1@0.50 & mAP   & R@1   & mAP   & R@1   & ~Acc@1~ & Acc@5 & Acc \\
    \midrule
    CLIP  & \cellcolor[rgb]{ .886,  .937,  .855}53.8  & \cellcolor[rgb]{ .886,  .937,  .855}40.1  & \cellcolor[rgb]{ .863,  .925,  .827}\textbf{52.9}  & \cellcolor[rgb]{ .886,  .937,  .855}46.0  & \cellcolor[rgb]{ .886,  .937,  .855}33.8  & \cellcolor[rgb]{ .686,  .808,  .922}45.9  & \cellcolor[rgb]{ .867,  .922,  .969}90.8  & \cellcolor[rgb]{ .686,  .808,  .922}25.5  & \cellcolor[rgb]{ .486,  .69,  .871}70.1  & \cellcolor[rgb]{ .992,  .941,  .914}34.6  & \cellcolor[rgb]{ .992,  .941,  .914}67.8  & \cellcolor[rgb]{ .957,  .929,  .976} 79.1 \\
    CLIP \textbf{\textit{+Ours}} & \cellcolor[rgb]{ .824,  .902,  .773}\textbf{54.9}  & \cellcolor[rgb]{ .855,  .918,  .812}\textbf{40.9}  & \cellcolor[rgb]{ .886,  .937,  .855}52.5  & \cellcolor[rgb]{ .867,  .925,  .827}\textbf{46.4}  & \cellcolor[rgb]{ .827,  .906,  .78}\textbf{34.8}  & \cellcolor[rgb]{ .58,  .745,  .894}\textbf{49.0}  & \cellcolor[rgb]{ .439,  .659,  .859}\textbf{96.0}  & \cellcolor[rgb]{ .573,  .741,  .894}\textbf{27.1}  & \cellcolor[rgb]{ .443,  .663,  .859}\textbf{71.3}  & \cellcolor[rgb]{ .988,  .89,  .827}\textbf{35.5}  & \cellcolor[rgb]{ .98,  .855,  .773}\textbf{68.8}  & \cellcolor[rgb]{ .902,  .835,  .949} \textbf{80.5} \\
    \cmidrule{1-13}
    
    BLIP  & \cellcolor[rgb]{ .678,  .816,  .584}57.3  & \cellcolor[rgb]{ .58,  .761,  .459}47.3  & \cellcolor[rgb]{ .702,  .831,  .62}55.6  & \cellcolor[rgb]{ .698,  .827,  .612}49.4  & \cellcolor[rgb]{ .714,  .839,  .631}36.7  & \cellcolor[rgb]{ .388,  .627,  .843}54.4  & \cellcolor[rgb]{ .408,  .639,  .851}96.4  & \cellcolor[rgb]{ .396,  .631,  .847}29.4  & \cellcolor[rgb]{ .369,  .616,  .839}73.3  & \cellcolor[rgb]{ .953,  .627,  .408}39.9  & \cellcolor[rgb]{ .941,  .561,  .298}72.4  & \cellcolor[rgb]{ .843,  .741,  .922} 81.8 \\
    BLIP \textbf{\textit{+Ours}} & \cellcolor[rgb]{ .6,  .773,  .482}\textbf{58.6}  & \cellcolor[rgb]{ .502,  .714,  .357}\textbf{49.1}  & \cellcolor[rgb]{ .588,  .765,  .471}\textbf{57.5}  & \cellcolor[rgb]{ .608,  .776,  .498}\textbf{51.0}  & \cellcolor[rgb]{ .627,  .788,  .522}\textbf{38.1}  & \cellcolor[rgb]{ .357,  .608,  .835}\textbf{55.2}  & \cellcolor[rgb]{ .357,  .608,  .835}\textbf{97.0}  & \cellcolor[rgb]{ .357,  .608,  .835}\textbf{29.9}  & \cellcolor[rgb]{ .357,  .608,  .835}\textbf{73.5}  & \cellcolor[rgb]{ .937,  .537,  .263}\textbf{41.4}  & \cellcolor[rgb]{ .937,  .537,  .263}\textbf{72.6}  & \cellcolor[rgb]{ .804,  .675,  .902} \textbf{82.8} \\
    \cmidrule{1-13}
    
    iBOT  & \cellcolor[rgb]{ .784,  .878,  .725}55.5  & \cellcolor[rgb]{ .863,  .925,  .824}40.7  & \cellcolor[rgb]{ .847,  .914,  .804}53.2  & \cellcolor[rgb]{ .8,  .886,  .741}\textbf{47.6}  & \cellcolor[rgb]{ .796,  .886,  .741}35.3  & \cellcolor[rgb]{ .961,  .976,  .992}33.4  & \cellcolor[rgb]{ .776,  .867,  .945}92.0  & \cellcolor[rgb]{ .961,  .976,  .992}18.1  & \cellcolor[rgb]{ .867,  .922,  .969}59.9  & \cellcolor[rgb]{ .965,  .733,  .58}38.1  & \cellcolor[rgb]{ .98,  .847,  .761}68.9  & \cellcolor[rgb]{ .918,  .863,  .957} 80.1 \\
    iBOT \textbf{\textit{+Ours}} & \cellcolor[rgb]{ .737,  .851,  .663}\textbf{56.3}  & \cellcolor[rgb]{ .769,  .871,  .702}\textbf{42.9}  & \cellcolor[rgb]{ .827,  .906,  .78}\textbf{53.5}  & \cellcolor[rgb]{ .804,  .89,  .749}47.5  & \cellcolor[rgb]{ .678,  .816,  .584}\textbf{37.3}  & \cellcolor[rgb]{ .918,  .953,  .98}\textbf{34.6}  & \cellcolor[rgb]{ .533,  .718,  .882}\textbf{94.9}  & \cellcolor[rgb]{ .89,  .933,  .973}\textbf{18.8}  & \cellcolor[rgb]{ .722,  .831,  .933}\textbf{63.9}  & \cellcolor[rgb]{ .945,  .584,  .341}\textbf{40.6}  & \cellcolor[rgb]{ .953,  .651,  .443}\textbf{71.3}  & \cellcolor[rgb]{ .827,  .71,  .914} \textbf{82.3} \\
    \cmidrule{1-13}
    
    DINO v2 & \cellcolor[rgb]{ .604,  .776,  .49}58.5  & \cellcolor[rgb]{ .718,  .839,  .635}44.1  & \cellcolor[rgb]{ .62,  .784,  .51}57.0  & \cellcolor[rgb]{ .592,  .769,  .475}51.3  & \cellcolor[rgb]{ .624,  .784,  .514}38.2  & \cellcolor[rgb]{ .824,  .894,  .957}37.1  & \cellcolor[rgb]{ .667,  .8,  .918}93.3  & \cellcolor[rgb]{ .894,  .937,  .976}18.7  & \cellcolor[rgb]{ .784,  .871,  .949}62.1  & \cellcolor[rgb]{ .957,  .667,  .471}39.2  & \cellcolor[rgb]{ .957,  .678,  .486}71.0  & \cellcolor[rgb]{ .725,  .541,  .863} 84.8 \\
    DINO v2 \textbf{\textit{+Ours}} & \cellcolor[rgb]{ .439,  .678,  .278}\textbf{61.2}  & \cellcolor[rgb]{ .439,  .678,  .278}\textbf{50.5}  & \cellcolor[rgb]{ .439,  .678,  .278}\textbf{60.0}  & \cellcolor[rgb]{ .439,  .678,  .278}\textbf{54.0}  & \cellcolor[rgb]{ .439,  .678,  .278}\textbf{41.2}  & \cellcolor[rgb]{ .741,  .843,  .933}\textbf{39.3}  & \cellcolor[rgb]{ .51,  .702,  .878}\textbf{95.2}  & \cellcolor[rgb]{ .741,  .843,  .933}\textbf{20.0}  & \cellcolor[rgb]{ .604,  .761,  .902}\textbf{67.0}  & \cellcolor[rgb]{ .945,  .596,  .357}\textbf{40.4}  & \cellcolor[rgb]{ .945,  .584,  .337}\textbf{72.1}  & \cellcolor[rgb]{ .702,  .506,  .851} \textbf{85.2} \\
    \bottomrule
    \end{tabular}%
  }
  \label{tab:frame_video_supp}%
\end{table*}%

\subsection{Image-pretrained Fundamental Models}
\label{subsec:fundamental_models}

We evaluate our method using eight representative pretrained image encoders, which can be broadly categorized into three paradigms of self-supervised learning: 
1) \textit{Masked modeling}: MAE~\cite{MAE}, I-JEPA~\cite{I_JEPA}; 
2) \textit{Contrastive learning}: CLIP~\cite{CLIP}, BLIP~\cite{BLIP}, MoCo v3~\cite{MoCov3}; 
3) \textit{Self-distillation}: iBOT~\cite{iBOT}, DINO~\cite{DINO}, DINO v2~\cite{DINOv2}. 
All models are pretrained on ImageNet-1k with self-supervised objectives, except for CLIP and BLIP, which are pretrained with natural language supervision. 
We adopt ViT-Base~\cite{Vision_Transformer} architectures with a patch size of 16 as the backbone encoder for each model.

\begin{itemize}[leftmargin=*]
  \item \textbf{MAE}~\cite{MAE} follows an encoder-decoder architecture, where random image patches are masked and the model is trained to reconstruct the missing content at the pixel level.
  \item \textbf{I-JEPA}~\cite{I_JEPA} learns representations by predicting latent representations of masked regions. It discards the decoder and instead relies on semantic-level prediction to better capture high-level image structures.
  \item \textbf{CLIP}~\cite{CLIP} is a vision-language model trained with natural language supervision. It learns to align image and text embeddings in a shared feature space using a contrastive objective on WIT dataset.
  \item \textbf{BLIP}~\cite{BLIP} is a vision-language model that extends CLIP-style contrastive pretraining with additional image-text matching and language modeling objectives. By jointly optimizing these objectives on large-scale web data, BLIP learns richer cross-modal representations and achieves stronger performance on image captioning and visual question answering.
  \item \textbf{MoCo v3}~\cite{MoCov3} is a contrastive learning method that adapts Momentum Contrast to Vision Transformers, employing a siamese architecture with an online encoder and a momentum-updated target encoder, and discards the negative sample queue used in earlier versions.
  \item \textbf{DINO}~\cite{DINO} adopts a self-distillation structure with Vision Transformers as the encoder. It encourages consistent representations across different views of the same image.
  \item \textbf{iBOT}~\cite{iBOT} builds upon DINO by introducing additional alignment on dense patch tokens. It aligns both the global \texttt{[CLS]} token and local patch-level features between two views, thus encouraging fine-grained spatial consistency in the learned representations.
  \item \textbf{DINO v2}~\cite{DINOv2} extends iBOT by incorporating various design improvements, including better centering techniques~\cite{centering}, regularization strategies like KoLeo loss~\cite{KoLeo}, and resolution-adaptive training~\cite{adapting_resolution}. In our experiments, we exclude computationally intensive techniques to ensure a consistent and fair comparison across models.
\end{itemize}

\subsection{Competitors}
\label{subsec:competitors}

We compare our approach against two categories of strong baselines: recent state-of-the-art (SOTA) video representation learning methods and image-to-video adaptation frameworks. For all baselines, we use the officially released pretrained weights without any additional training or fine-tuning.

\noindent\textbf{1) Video representation learning methods}: 
These methods are specifically designed to learn spatiotemporal representations from raw video inputs, often relying on temporal masking or reconstruction-based objectives.

\begin{itemize}[leftmargin=*]
  \item \textbf{VideoMAE}~\cite{VideoMAE} extends the masked modeling paradigm to videos by randomly masking spatiotemporal tubes and reconstructing the missing pixels. It adopts a high masking ratio to encourage the encoder to capture both appearance and motion features.
  \item \textbf{MAE-ST}~\cite{MAE_ST} adapts MAE to spatiotemporal data by explicitly incorporating temporal modeling modules into the encoder to better capture dynamic patterns.
  \item \textbf{DropMAE}~\cite{DropMAE} applies spatial-attention dropout in masked modeling, encouraging the model to attend to motion cues for temporal discriminability.
  \item \textbf{SiamMAE}~\cite{SiamMAE} adopts a Siamese structure where the past frame and a masked version of the current frame are jointly encoded. A conditional decoder is employed to reconstruct the missing patches, thereby promoting temporal consistency across frames.
  \item \textbf{CropMAE}~\cite{CropMAE} generalizes SiamMAE by using different crops or augmentations of the same frame as input, encouraging invariance under intra-frame transformations.
  \item \textbf{RSP}~\cite{RSP} formulates temporal modeling as a stochastic frame prediction task. It learns to reconstruct future frames from current ones by modeling both prior and posterior distributions over latent motion variables.
\end{itemize}

\noindent\textbf{2) Image-to-video adaptation methods}:
These methods aim to adapt pretrained image models to instance-level video understanding tasks by integrating lightweight modules that enable temporal reasoning, while keeping most of the backbone parameters frozen and only updating a small subset.

\begin{itemize}[leftmargin=*]
  \item \textbf{AIM}~\cite{adapter3} introduces a lightweight adapter into a frozen ViT backbone, enabling spatiotemporal adaptation along spatial and temporal dimensions, facilitating efficient transfer from static to dynamic inputs.
  \item \textbf{ST-Adapter}~\cite{adapter2} proposes a 3D bottleneck adapter into the CLIP-pretrained ViT model, which enables the model to reason about dynamic video content at a small task-specific parameter cost.
  \item \textbf{ZeroI2V}~\cite{adapter4} introduces spatial-temporal dual-headed attention mechanism combined with a linear adaptation layer, thus enabling the transfer of frozen image models to video tasks and supporting zero additional inference cost via structural reparameterization.
\end{itemize}

\begin{figure*}[t]
  \centering
\includegraphics[width=0.30\linewidth]{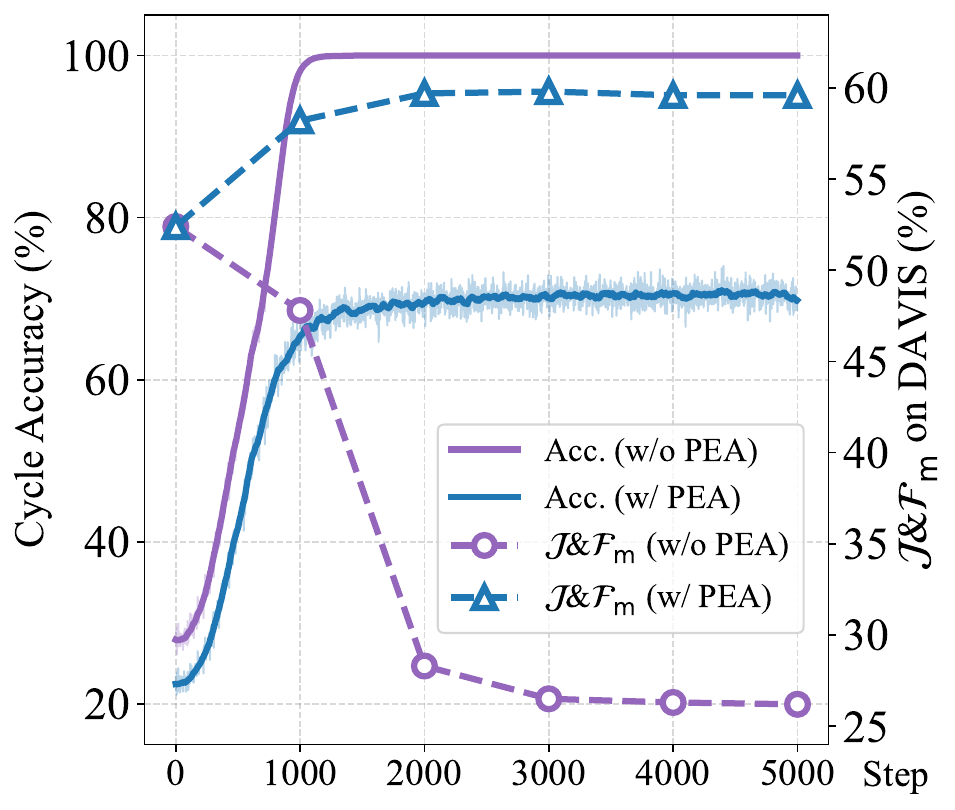}
\includegraphics[width=0.30\linewidth]{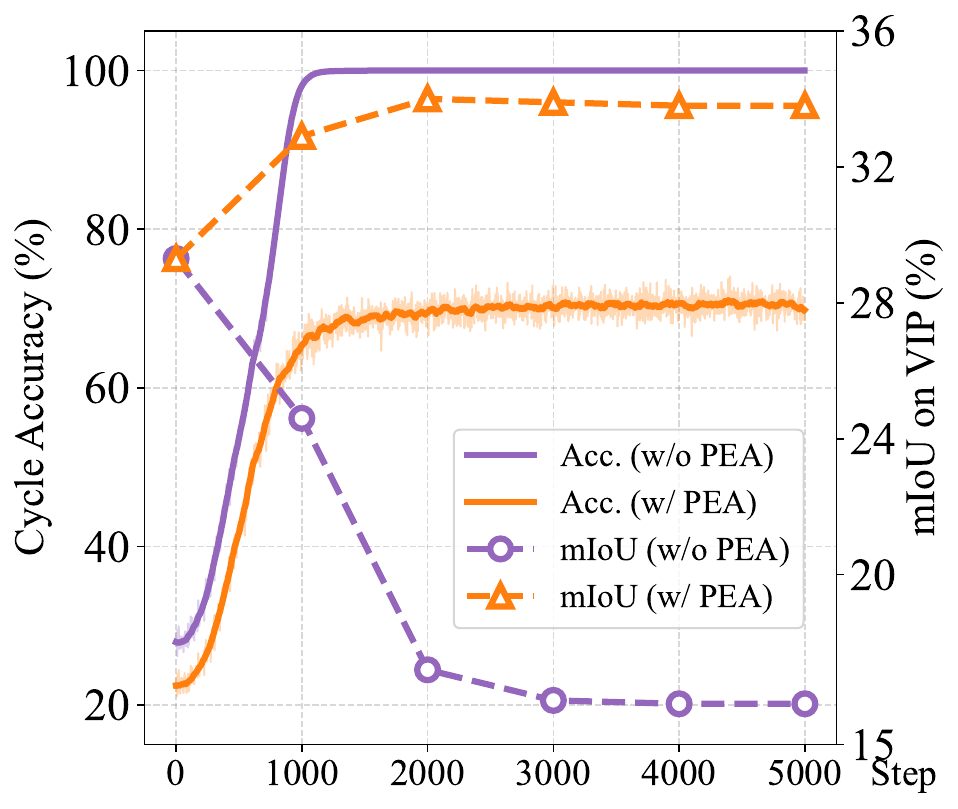}
\includegraphics[width=0.30\linewidth]{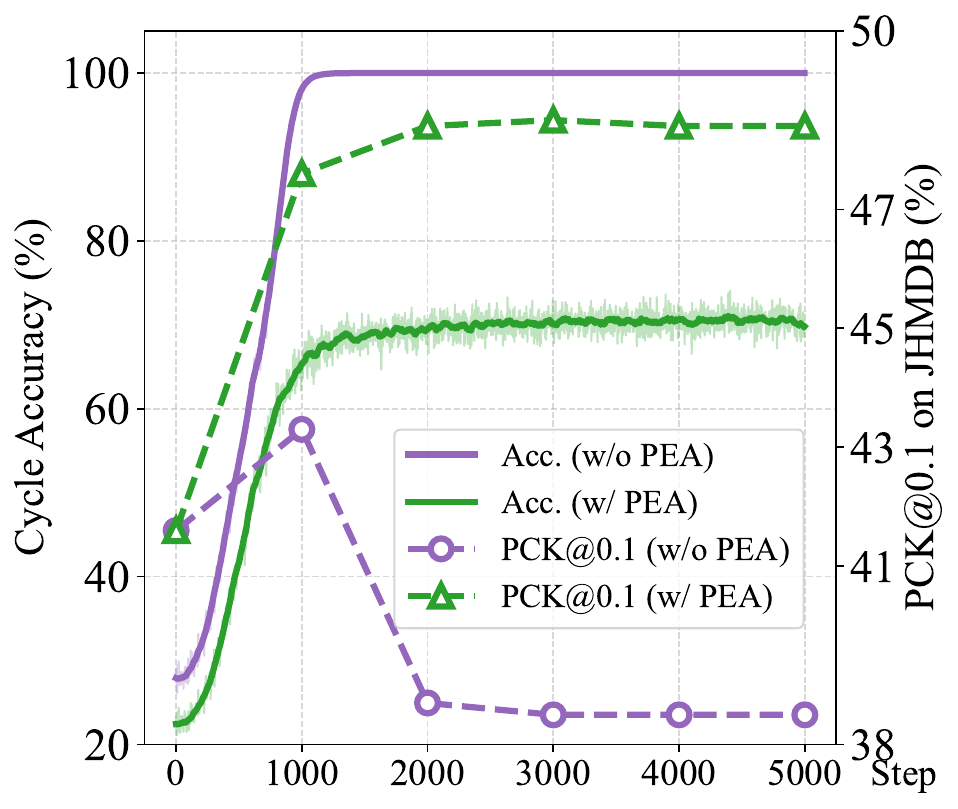}
\includegraphics[width=0.30\linewidth]{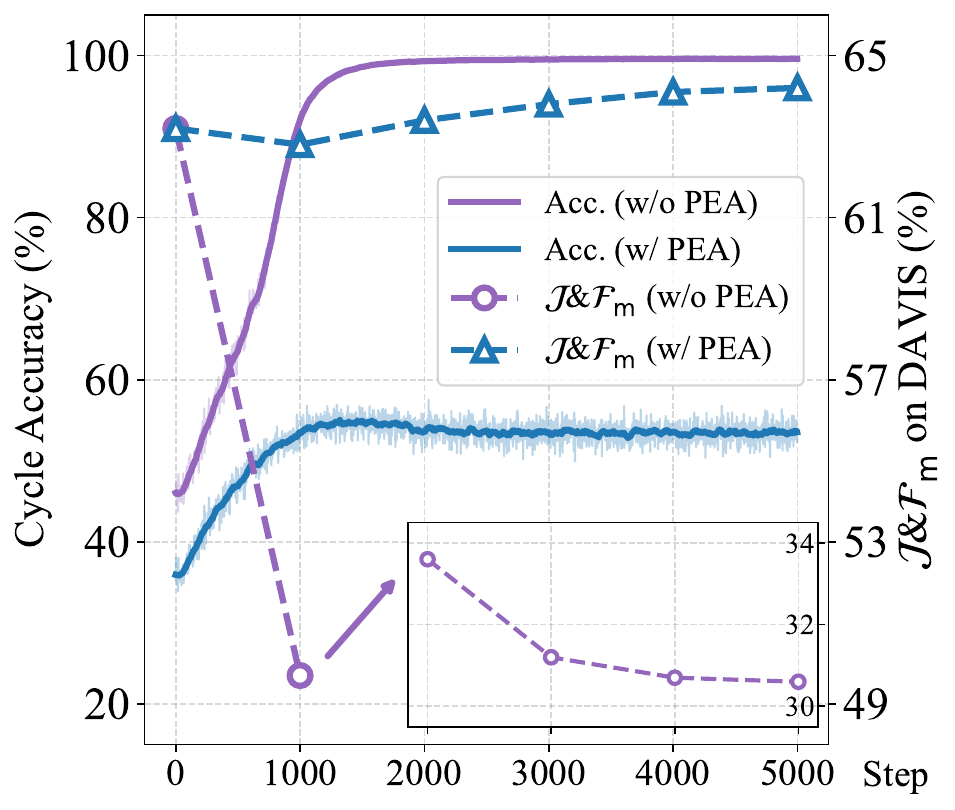}
\includegraphics[width=0.30\linewidth]{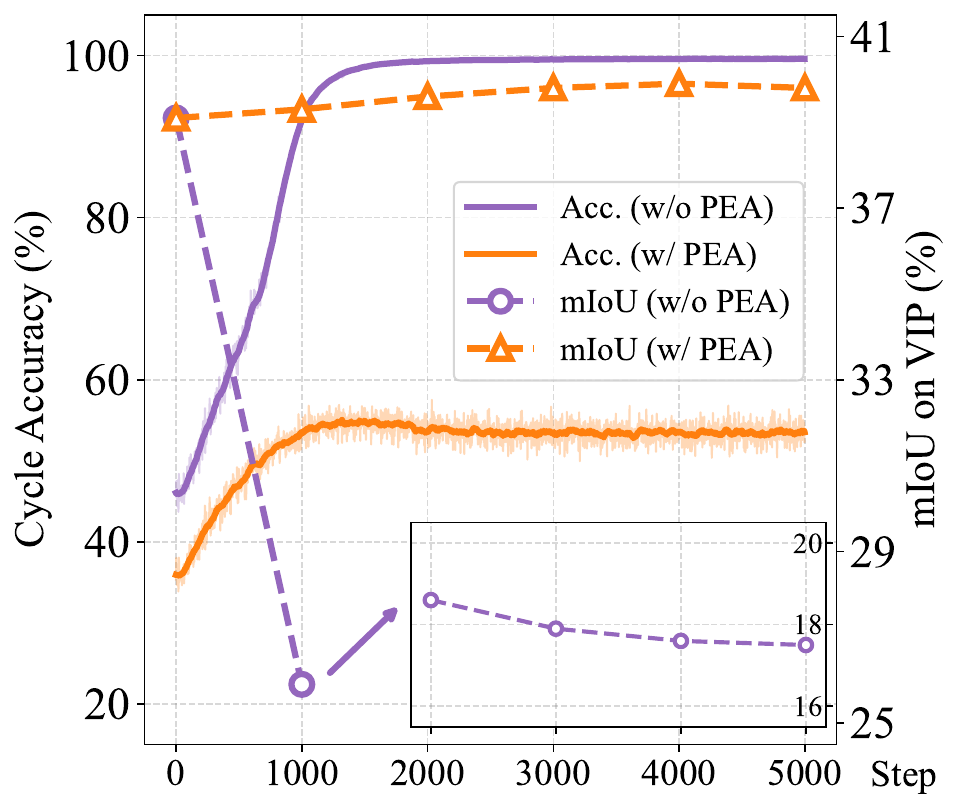}
\includegraphics[width=0.30\linewidth]{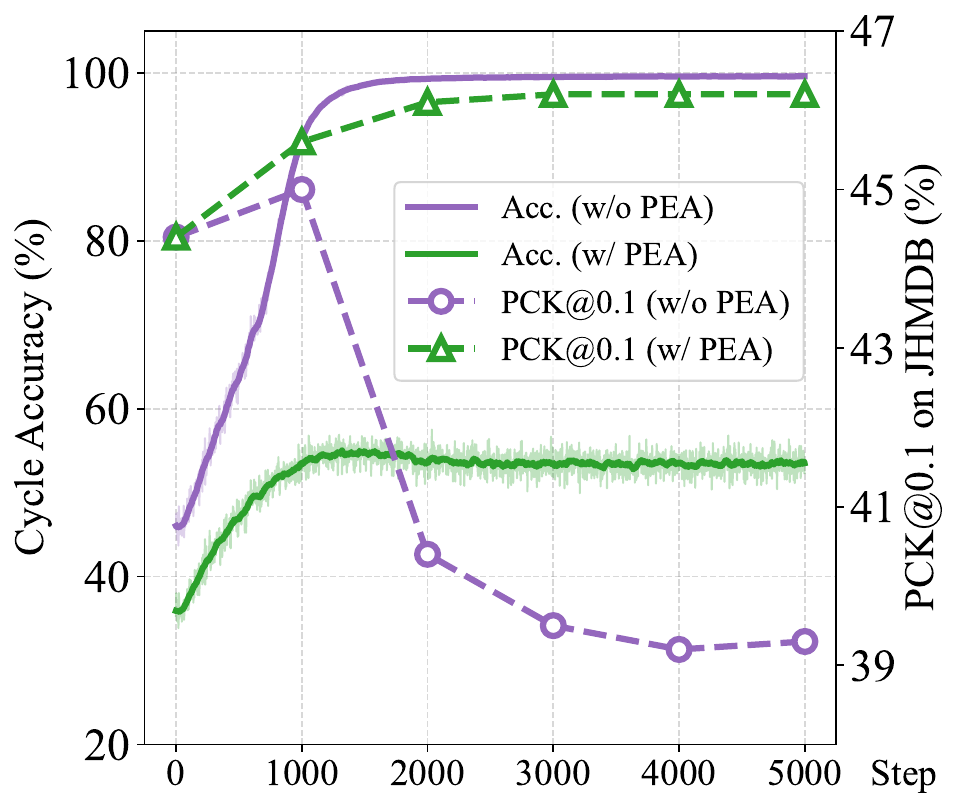}
  \caption{Cycle-consistent accuracy and downstream performance during training with or without the PEA strategies across three tasks based on MAE (\textit{line 1}) and DINO (\textit{line 2}).}
  \label{fig:acc_curve_supp}
\end{figure*}

\begin{table}[htbp]
  \centering
  \caption{Extended comparison with representative methods on the DAVIS-2017 validation set. Methods are grouped by their core settings for a broader reference.}
  \label{tab:sota_supp}
  \resizebox{\linewidth}{!}{%
    \begin{tabular}{clc|ccc}
    \toprule
    \multirow{2}{*}{Type} & \multirow{2}{*}{Method} & \multirow{2}{*}{Backbone} & \multicolumn{3}{c}{DAVIS-2017} \\
    \cmidrule{4-6}
    & & & $\mathcal{J}\&\mathcal{F}_m$ & $\mathcal{J}_m$ & $\mathcal{F}_m$ \\
    \midrule
    \multirow{4}{*}{\shortstack{Dedicated  VOS \\ Systems}}
    & STCN~\cite{STCN} & ResNet50 & 85.4 & 82.2 & 88.6 \\
    & SwinB-AOT-L~\cite{SwinB} & Swin-B & 85.4 & 82.4 & 88.4 \\
    & SimVOS-B~\cite{SimVOS} & ViT-B/16 & 88.0 & 85.0 & 91.0 \\
    & Cutie-base~\cite{Cutie} & ResNet50 & 87.9 & 84.6 & 91.1 \\
    \midrule
    \multirow{2}{*}{\shortstack{Segmentation \\ Foundation Models}}
    & SAM 2~\cite{sam2} & Hiera-B+ & 90.2 & 87.0 & 93.4 \\
    & SAM 2~\cite{sam2} & Hiera-L & 90.7 & 87.5 & 94.0 \\
    \midrule
    \multirow{5}{*}{\shortstack{Self-Supervised \\ Video Pre-training}}
    & VideoMAE~\cite{VideoMAE} & ViT-L/16 & 45.0 & 43.6 & 46.5 \\
    & MAE-ST~\cite{MAE_ST} & ViT-L/16 & 54.6 & 55.5 & 53.6 \\
    & SiamMAE~\cite{SiamMAE} & ViT-B/16 & 60.9 & 59.4 & 62.4 \\
    & CropMAE~\cite{CropMAE} & ViT-B/16 & 57.8 & 56.9 & 58.7 \\
    & RSP~\cite{RSP} & ViT-B/16 & 60.5 & 57.8 & 63.2 \\
    \midrule
    \multirow{6}{*}{\shortstack{Self-Supervised \\ Image Pre-training \\ \textbf{\textit{+Ours}}}}
    & DINO~\cite{DINO} & ViT-B/16 & 63.2 & 60.9 & 65.5 \\
    & DINO + \textbf{\textit{Ours}} & ViT-B/16 & 64.2 & 62.3 & 66.0 \\
    & DINOv2~\cite{DINOv2} & ViT-B/16 & 63.1 & 61.6 & 64.5 \\
    & DINOv2 + \textbf{\textit{Ours}} & ViT-B/16 & 63.7 & 61.9 & 65.4 \\
    & iBOT~\cite{iBOT} & ViT-B/16 & 64.6 & 63.0 & 66.1 \\
    & iBOT + \textbf{\textit{Ours}} & ViT-B/16 & 65.1 & 63.3 & 66.9 \\
    \bottomrule
    \end{tabular}%
  }
\end{table}

\section{Detailed Experiments Results}
\label{sec:results}

\subsection{Comparison with Task-Specific SOTAs}
\label{subsec:broader_vos_comparison}

To provide a broader view, we compare our method with several representative VOS systems and recent segmentation foundation models on DAVIS-2017 validation in \cref{tab:sota_supp}. 
Our work focuses on general representation pre-training for direct \textbf{transfer} across multiple tasks rather than task-specific designs, and thereby applies lightweight transfer for evaluation per standard self-supervised learning protocols. Thus, our datasets, computing resources, and architectures are not aligned with specialized SoTA methods for individual tasks such as video object segmentation (VOS).

\subsection{Detailed Results of Frame-/Video-Level Tasks}
\label{subsec:shortcut}

We further evaluate the transferred models on several frame- and video-level downstream tasks: temporal action localization on Breakfast~\cite{Breakfast} using the FACT~\cite{FACT} backbone, zero-shot video retrieval on UCF101 and HMDB51~\cite{UCF101,HMDB51}, fine-tuned action classification on Something-Something-v2 (SSV2)~\cite{ssv2}, and temporal order discrimination via linear probing on Chiral SSV2~\cite{Chirality}.

The quantitative results of transferred representations from four representative image models on both frame-level and video-level downstream tasks are depicted in \cref{tab:frame_video_supp}. Our method delivers steady performance improvements across these tasks. For instance, on frame-level tasks, it achieves an average improvement of $2.80\%~Acc$ on Breakfast, indicating enhanced temporal awareness in image models. On video-level tasks, it brings a $2.58\%~R@1$ improvement on HMDB51 and a $1.53\%~Acc@1$ gain on SSV2, which validates the preserved semantic discrimination ability. These results indicate that our method generalizes well across different task granularities, highlighting its potential as a versatile solution for image-to-video transfer.

\subsection{Training Dynamics}
\label{subsec:dynamics}

We visualize the training dynamics of MAE and DINO across three downstream tasks in \cref{fig:acc_curve_supp}. The plots show the cycle-consistency accuracy (\ie, the percentage of patches that return to their original positions after a cycle traversal) together with the downstream performance over training steps. Without the PEA strategy, the downstream performance drops sharply within the first two epochs, even when the cycle-consistency accuracy is close to $100\%$. This indicates that the model exploits the absolute positional encoding as a shortcut instead of learning temporal correspondences that remain reliable when the temporal distance between frames grows.

In contrast, when we apply the proposed PEA strategy, the cycle-consistency accuracy increases gradually, and the final value converges to a small stable range that depends on the model architecture and hyperparameter settings. This behavior is reasonable, since in real-world videos, correspondence quality naturally degrades as time passes: the first and last frames in a propagation chain can differ greatly due to camera motion and non-rigid object deformation, which leads to unavoidable information loss. On the Kinetics-400 dataset, the empirical cycle-consistency accuracy stabilizes around $50\% \sim 70\%$ when the temporal interval is set to $\delta = 0.15$. By promoting effective dense correspondences between frames and reducing reliance on positional cues, PEA leads to more stable improvements in downstream performance and highlights its role in learning robust temporal representations.

\subsection{Shortcut Phenomenon in Training}
\label{subsec:shortcut}

\cref{tab:PEA_supp} compares the performance of our method trained with and without the proposed Positional Encoding Augmentation (PEA) strategy. As shown, removing PEA consistently leads to substantial performance degradation, with $4.4\%\sim37.3\%$ drop in $\mathcal{J}\&\mathcal{F}_{\mathrm{m}}$ on DAVIS and $4.9\%\sim22.8\%$ drop in mIoU on VIP. This is primarily due to the model exploiting absolute positional encodings as shortcuts, resulting in dimensional collapse and degraded representations. The issue is particularly severe in self-distillation architectures, which rely heavily on positional alignment between teacher and student branches. 
This highlights the brittleness of image-pretrained representations when transferred to video and underscores that image-to-video transfer is a non-trivial challenge.
In contrast, applying PEA consistently improves performance across all three downstream tasks, indicating the effectiveness of resisting shortcuts induced by the positional encoding mechanism of ViT.

\begin{table}[htbp]
  \centering
  \caption{Impact of Positional Encoding Augmentation (PEA) strategy on representation quality across three downstream tasks.}
  \setlength{\tabcolsep}{12pt}
    \resizebox{0.47\textwidth}{!}{%
    \begin{tabular}{ccccc}
    \toprule
    \multirow{2}{*}{\centering \shortstack{Image \\ Model}} & \multirow{2}[2]{*}{Method} &  VIP   & \multicolumn{1}{l}{DAVIS17} & \multicolumn{1}{l}{JHMDB} \\
   &      &   mIoU   & $\mathcal{J}\&\mathcal{F}_{\mathrm{m}}$ &PCK@0.1\\
    \midrule
    \multicolumn{1}{c}{\multirow{3}[2]{*}{MAE}}  & Vanilla  & 29.3   & 52.4  & 41.6  \\
    & w/o PEA & 16.2$_{\textcolor{define_green}{-13.1}}$  & 26.2$_{\textcolor{define_green}{-26.2}}$  & 38.5$_{\textcolor{define_green}{-3.1}}$  \\
    & w/ PEA & 33.8$_{\textcolor{define_red}{+4.5}}$  & 59.6$_{\textcolor{define_red}{+7.2}}$  & 48.4$_{\textcolor{define_red}{+6.8}}$  \\
    \midrule
    
    \multicolumn{1}{c}{\multirow{3}[2]{*}{I-JEPA}} & Vanilla  & 31.5 & 53.9   & 42.6  \\
    & w/o PEA  & 26.6$_{\textcolor{define_green}{-4.9}}$  & 49.5$_{\textcolor{define_green}{-4.4}}$  & 44.1$_{\textcolor{define_red}{+1.5}}$  \\
    & w/ PEA  & 35.3$_{\textcolor{define_red}{+3.8}}$  & 58.7$_{\textcolor{define_red}{+4.8}}$  & 44.4$_{\textcolor{define_red}{+1.8}}$  \\
    \midrule
    
    \multicolumn{1}{c}{\multirow{3}[2]{*}{MoCo v3}} &  Vanilla &   38.8  & 62.6  & 43.6  \\
    & w/o PEA & 23.8$_{\textcolor{define_green}{-15.0}}$  & 42.8$_{\textcolor{define_green}{-19.8}}$  & 42.2$_{\textcolor{define_green}{-1.4}}$  \\
    & w/ PEA & 39.8$_{\textcolor{define_red}{+1.0}}$  & 62.9$_{\textcolor{define_red}{+0.3}}$  & 45.3$_{\textcolor{define_red}{+1.7}}$  \\
    \midrule
        
    \multicolumn{1}{c}{\multirow{3}[2]{*}{iBOT}} & Vanilla & 39.6   & 64.6  & 45.7  \\
        & w/o PEA & 16.8$_{\textcolor{define_green}{-22.8}}$  & 27.3$_{\textcolor{define_green}{-37.3}}$  & 38.2$_{\textcolor{define_green}{-7.5}}$  \\
        & w/ PEA & 40.8$_{\textcolor{define_red}{+1.2}}$  & 65.1$_{\textcolor{define_red}{+0.5}}$  & 46.1$_{\textcolor{define_red}{+0.4}}$  \\
    \midrule
    
    \multicolumn{1}{c}{\multirow{3}[2]{*}{DINO}}  & Vanilla & 39.1  & 63.2  &  44.4  \\
    & w/o PEA & 17.5$_{\textcolor{define_green}{-21.6}}$  & 30.6$_{\textcolor{define_green}{-32.6}}$  & 39.3$_{\textcolor{define_green}{-5.1}}$  \\
    & w/ PEA & 39.8$_{\textcolor{define_red}{+0.7}}$  & 64.2$_{\textcolor{define_red}{+1.0}}$  & 46.2$_{\textcolor{define_red}{+1.8}}$  \\
    \midrule
    
    \multicolumn{1}{c}{\multirow{3}[2]{*}{DINO v2}} & Vanilla &  38.4  &  63.1  & 46.6  \\
    & w/o PEA & 17.7$_{\textcolor{define_green}{-20.7}}$  & 30.0$_{\textcolor{define_green}{-33.1}}$  & 39.1$_{\textcolor{define_green}{-7.5}}$  \\
    & w/ PEA & 39.9$_{\textcolor{define_red}{+1.5}}$  & 63.7$_{\textcolor{define_red}{+0.6}}$  & 47.3$_{\textcolor{define_red}{+0.7}}$  \\
    \bottomrule

    \end{tabular}%
    }
  \label{tab:PEA_supp}%
\end{table}%

\begin{figure*}[htbp]
  \centering
    \includegraphics[width=0.31\linewidth]{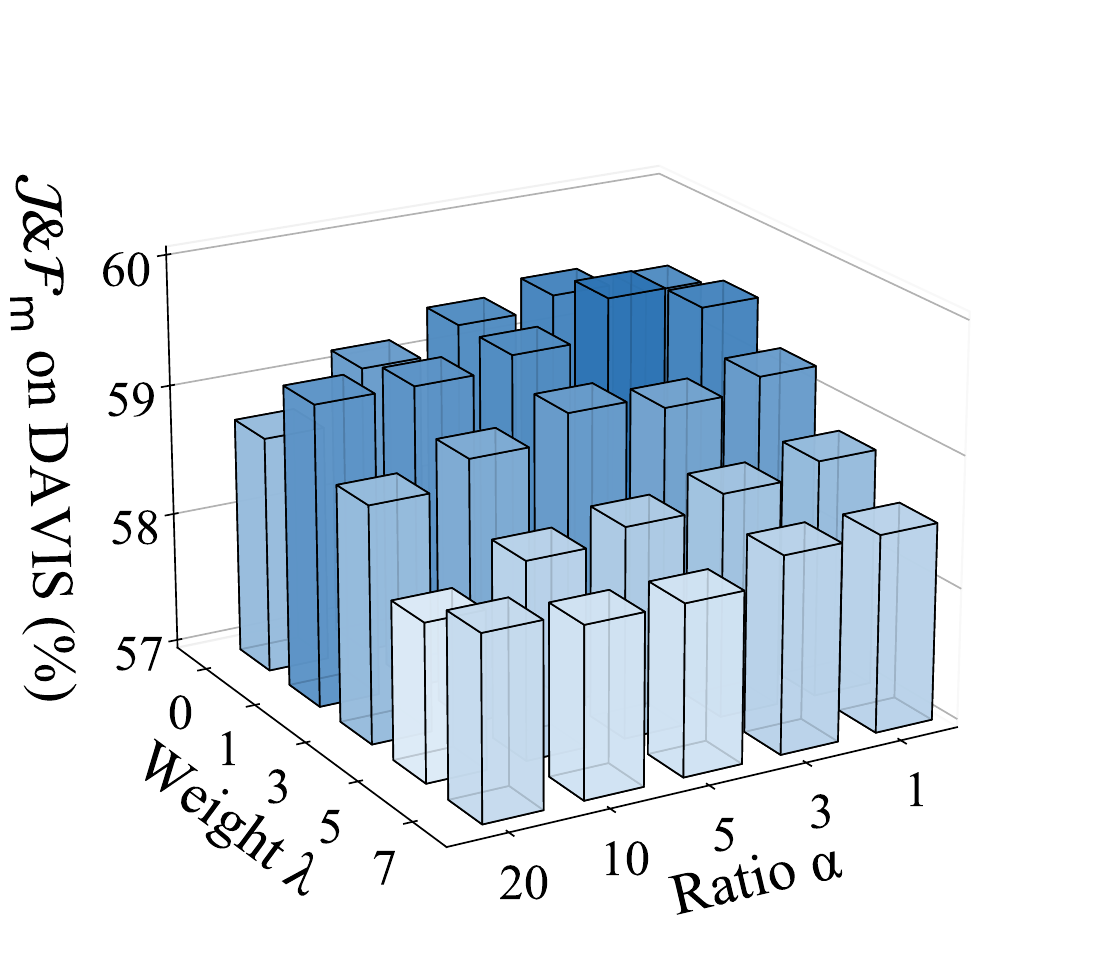}
    \includegraphics[width=0.31\linewidth]{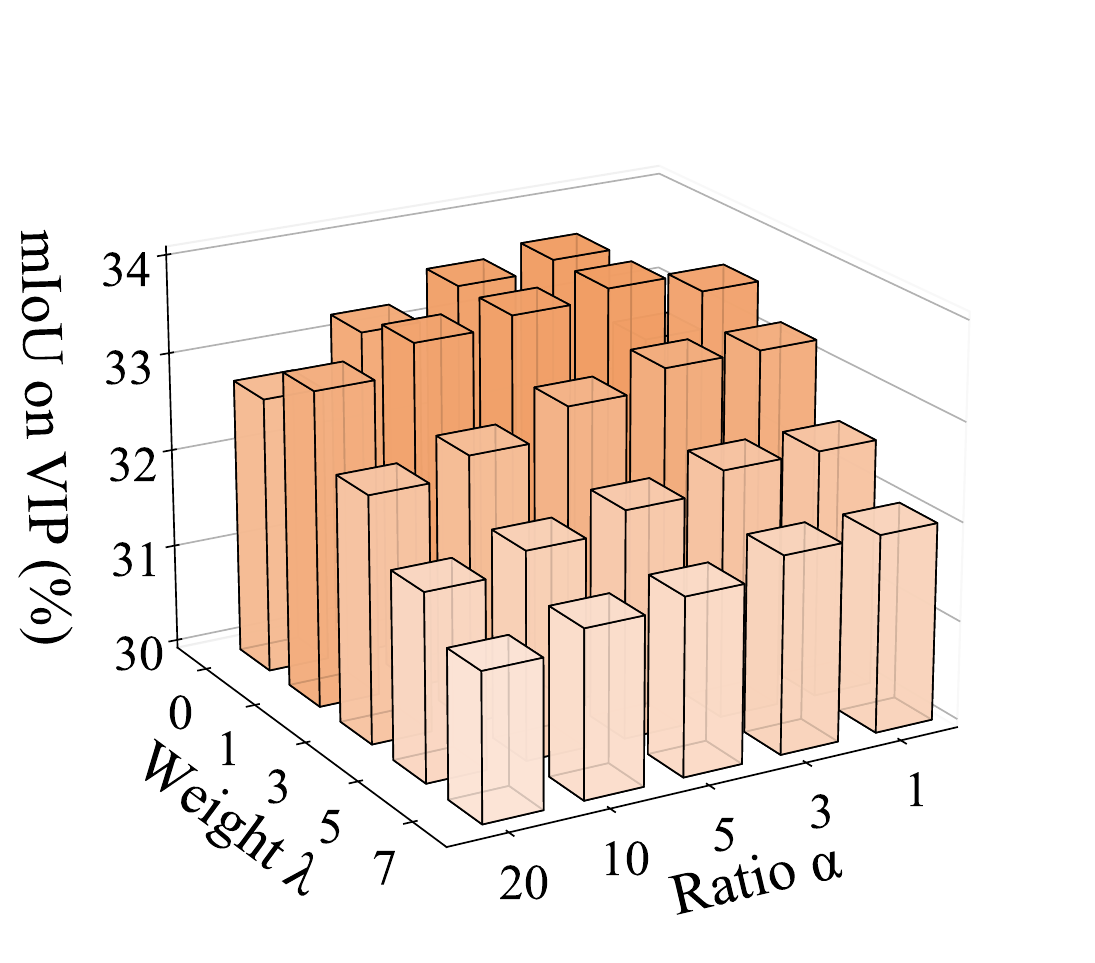}
    \includegraphics[width=0.31\linewidth]{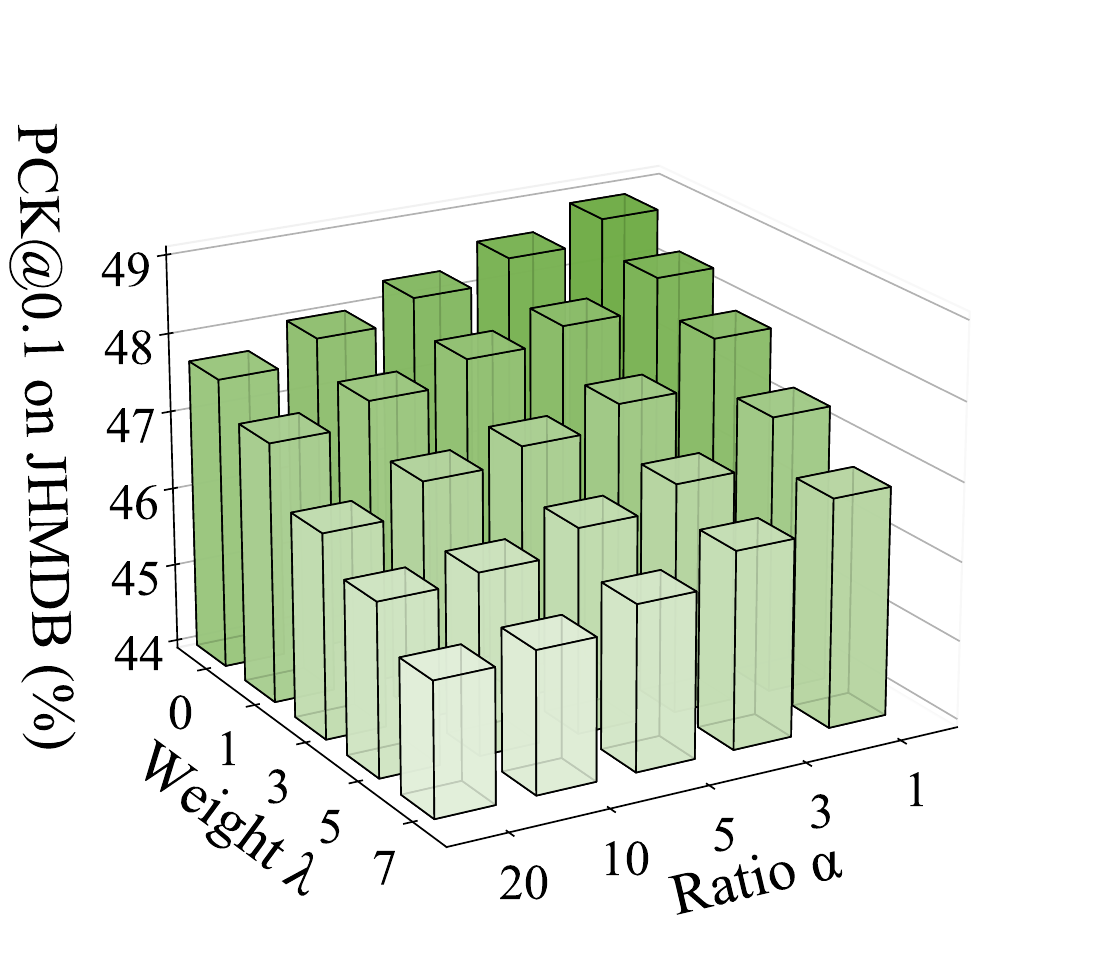}
  \caption{3D bar charts for ablation results on interpolation ratio $\alpha$ and regularization weight $\lambda$ across three tasks using MAE.}
  \label{fig:hyper_1_supp}
\end{figure*}

\begin{table*}[htbp]
  \centering
  \caption{Sensitivity analysis on temporal interval $\delta$ and softmax temporature $\tau$. Default settings are highlighted with \colorbox[rgb]{ .867,  .922,  .969}{blue}.}
  \vspace{-10pt}
  \begin{tabular}{cc}
    \begin{subtable}[t]{0.45\textwidth}
      \centering
      \caption{Ablation on temporal interval $\delta$.}
      \setlength{\tabcolsep}{12pt}
      \resizebox{\textwidth}{!}{%
        \begin{tabular}{ccccc}
        \toprule
        \multirow{2}{*}{\centering \shortstack{Base \\ Model}} & \multirow{2}{*}{$\delta$} & \multicolumn{1}{c}{VIP}  & \multicolumn{1}{c}{DAVIS17} & JHMDB \\
              &     & mIoU & $\mathcal{J}\&\mathcal{F}_{\mathrm{m}}$ &  PCK@0.1 \\
        \midrule
        \multirow{5}{*}{MAE} & 0.05 & 33.9 & 59.6 & 48.6 \\
                  & 0.10 & 33.9 & 59.7 & 48.4 \\
                  & \cellcolor[rgb]{ .867,  .922,  .969}0.15 & \cellcolor[rgb]{ .867,  .922,  .969}33.8  & \cellcolor[rgb]{ .867,  .922,  .969}59.6  & \cellcolor[rgb]{ .867,  .922,  .969}48.4 \\
                  & 0.20 & 33.6 & 59.7 & 48.5 \\
                  & 0.25 & 33.6 & 59.6 & 48.4 \\
        \midrule
        \multirow{5}{*}{DINO} & 0.05 & 39.7 & 63.9 & 46.2 \\
                  & 0.10 & 39.9 & 64.0 & 46.1 \\
                  & \cellcolor[rgb]{ .867,  .922,  .969}0.15 & \cellcolor[rgb]{ .867,  .922,  .969}39.8  & \cellcolor[rgb]{ .867,  .922,  .969}64.2  & \cellcolor[rgb]{ .867,  .922,  .969}46.2  \\
                  & 0.20 & 39.9 & 64.2 & 46.1 \\
                  & 0.25 & 39.9 & 64.2 & 46.1 \\
        \bottomrule
        \end{tabular}
      }
      \label{tab:ablation_delta}
    \end{subtable}

    &

    \begin{subtable}[t]{0.45\textwidth}
      \centering
      \caption{Ablation on softmax temporature $\tau$.}
      \setlength{\tabcolsep}{12pt}
      \resizebox{\textwidth}{!}{%
        \begin{tabular}{ccccc}
        \toprule
        \multirow{2}{*}{\centering \shortstack{Base \\ Model}} & \multirow{2}{*}{$\tau$} & \multicolumn{1}{c}{VIP}  & \multicolumn{1}{c}{DAVIS17} & JHMDB \\
              &     & mIoU & $\mathcal{J}\&\mathcal{F}_{\mathrm{m}}$ &  PCK@0.1 \\
        \midrule
        \multirow{5}{*}{MAE} & 0.01 & 32.6 & 60.0 & 48.2 \\
                  & 0.02 & 33.4 & 60.0 & 48.5 \\
                  & \cellcolor[rgb]{ .867,  .922,  .969}0.03 & \cellcolor[rgb]{ .867,  .922,  .969}33.8  & \cellcolor[rgb]{ .867,  .922,  .969}59.6  & \cellcolor[rgb]{ .867,  .922,  .969}48.4 \\
                  & 0.04 & 33.7 & 58.9 & 48.5 \\
                  & 0.05 & 33.6 & 58.6 & 48.4 \\
        \midrule
        \multirow{5}{*}{DINO} & 0.01 & 39.2 & 63.7 & 46.0 \\
                  & 0.02 & 40.0 & 63.9 & 46.1 \\
                  & \cellcolor[rgb]{ .867,  .922,  .969}0.03 & \cellcolor[rgb]{ .867,  .922,  .969}39.8  & \cellcolor[rgb]{ .867,  .922,  .969}64.2  & \cellcolor[rgb]{ .867,  .922,  .969}46.2  \\
                  & 0.04 & 39.9 & 64.0 & 46.0 \\
                  & 0.05 & 39.9 & 63.8 & 45.9 \\
        \bottomrule
        \end{tabular}
      }
      \label{tab:ablation_tau}
    \end{subtable}
  \end{tabular}
  \label{tab:hyper_2_supp}
  \vspace{-6pt}
\end{table*}

\begin{table}[htbp]
  \centering
  \caption{Robustness and generalization analysis of PEA strategy across DINO series features.}
  \vspace{-6pt}

  \begin{subtable}[b]{0.38\textwidth}
    \centering
    \caption{Robustness across PEA crop manners.}
    \label{tab:hyper_3_supp_a}
    \setlength{\tabcolsep}{8pt}
    \resizebox{\textwidth}{!}{%
        \begin{tabular}{c|c|ccc}
          \toprule
          \multirow{2}{*}{\centering \shortstack{Base \\ Model}} & \multirow{2}{*}{\centering \shortstack{PEA \\ Crop}} & \multicolumn{1}{c}{VIP}  & \multicolumn{1}{c}{DAVIS17} & JHMDB \\
          &     & mIoU & $\mathcal{J}\&\mathcal{F}_{\mathrm{m}}$ &  PCK@0.1 \\
          \midrule
          \multirow{4}{*}{DINO} & center & 64.1 & 39.3 & \textbf{46.9} \\
          & \cellcolor[rgb]{ .867,  .922,  .969}random & \cellcolor[rgb]{ .867,  .922,  .969}\underline{64.2} & \cellcolor[rgb]{ .867,  .922,  .969}\underline{39.8} & \cellcolor[rgb]{ .867,  .922,  .969}\underline{46.2} \\
          & edge   & 64.1 & \textbf{39.9} & 46.2 \\
          & multiple  & \textbf{64.3} & \underline{39.8} & 46.2 \\
          \bottomrule
        \end{tabular}
    }
  \end{subtable}
  
  \vspace{6pt}
  
  \begin{subtable}[b]{0.38\textwidth} 
      \centering
      \caption{Generalization across PE variants and patch sizes.} 
      \label{tab:hyper_3_supp_b}
      \setlength{\tabcolsep}{8pt}
      \resizebox{\textwidth}{!}{%
          \begin{tabular}{cc|ccc}
          \toprule
          \multirow{3}{*}{\centering \shortstack{PEA}} & \multirow{3}{*}{\centering \shortstack{$L_{reg}$}}  & DINO
          & DINO v2
          & DINO v3 \\ 
         & & (Abs. PE) & (Abs. PE) & (\textcolor{red}{RoPE}) \\
         & & ViT-S/\textcolor{red}{8} & ViT-S/\textcolor{red}{14} & ViT-S/16 \\
          \midrule
          \multicolumn{2}{c|}{\centering \shortstack{Vanilla}} & 71.7 & 64.7 & 67.3 \\ 
           \ding{55} & \ding{51} & 71.1  & 63.9  & 65.8 \\
          \rowcolor[rgb]{ .875,  .929,  .969}   \ding{51} & \ding{51}   & \textbf{72.3} & \textbf{65.1} & \textbf{67.9}\\
          \bottomrule
          \end{tabular}
      }
  \end{subtable}
  \label{tab:hyper_3_supp}
  \vspace{-6pt}
\end{table}

\begin{table}[htbp]
  \centering
  \caption{Ablation study on the choice of regularization loss ($L_{reg}$) using the DINO backbone. }
  \label{tab:hyper_4_supp}
  \setlength{\tabcolsep}{7pt} 
  \resizebox{\linewidth}{!}{%
    \begin{tabular}{cc|ccc|ccc}
    \toprule
    PEA & $L_{reg}$ & DAVIS & VIP & JHMDB  & $D_{inter}$ & $D_{intra}$ & $D~(\uparrow)$ \\
    \midrule
    \multicolumn{2}{c|}{\centering \shortstack{DINO}}  & 63.2 & 39.1 & 44.4  &  0.5756 & 0.2144 & 0.5112 \\
     \ding{55} & KL   & 61.8  & 38.0  &  46.1 & 0.6241 & 0.2404 & 0.5520 \\
     \ding{51} & MSE    & 62.2  & 38.1 & 46.1  & 0.6142 & 0.2370 & 0.5431 \\
    \rowcolor[rgb]{ .875,  .929,  .969} \ding{51} & KL  & \textbf{64.2} & \textbf{39.8} & \textbf{46.2} & 0.6246 & 0.2316 & \textbf{0.5551} \\
    \bottomrule
    \end{tabular}%
  }
  \vspace{-6pt}
\end{table}

\subsection{Additional Ablation Study}
\label{subsec:ablation}

In \cref{fig:hyper_1_supp}, we study the effects of the interpolation ratio $\alpha$ and the regularization weight $\lambda$. A moderate value of $\alpha$ gives the best performance since a small $\alpha$ cannot effectively suppress shortcut learning, while a large one disrupts relative positional cues and harms correspondence learning. Similarly, $\lambda$ controls the strength of the semantic separability constraint: too small values may cause dimensional collapse in the projection space, whereas overly strong regularization reduces the flexibility needed to adapt the representations. Overall, both hyperparameters influence performance in a relatively mild range, and good results can be obtained with moderate choices.

\cref{tab:hyper_2_supp} analyzes the sensitivity of the temporal sampling interval $\delta$ and the softmax temperature $\tau$. A suitable $\delta$ balances visible motion and visual continuity, which is important for learning meaningful frame-level correspondences, while a moderate $\tau$ maintains an appropriate level of similarity sharpening. The model shows limited sensitivity to variations in these two hyperparameters, and the performance remains stable across a reasonable range. To ensure consistency and fair comparison across all experiments, we fix $\tau = 0.03$ and $\delta = 0.15$.

We further conduct a robustness and generalization analysis by varying the PEA crop strategy (\cref{tab:hyper_3_supp_a}) and the model patch size alongside positional encoding variants (\cref{tab:hyper_3_supp_b}). PEA remains stable across a wide range of crop choices and generalizes well across various patch sizes (e.g., 8, 14, and 16). Moreover, PEA is compatible with modern designs such as RoPE~\cite{RoPE}. By interpolating and cropping on the RoPE coordinate grid, PEA effectively mitigates shortcut behaviors, further demonstrating the robustness of our method.

As shown in \cref{tab:hyper_4_supp}, we investigate the impact of different regularization objectives. The KL-based regularization matches the distribution of transferred video representations to that of frozen image features. This helps preserve the inherited semantic geometry and prevents feature collapse by aligning distance relationships (as discussed in Sec.~4). Compared to a strict element-wise MSE loss, KL divergence provides a softer, distribution-level constraint. This allows for sufficient temporal adaptation while effectively maintaining semantic separability. Consequently, KL regularization tends to increase the normalized inter-video distance ($D_{inter}$), which aligns perfectly with the observed improvements in downstream task performance.

\section{Additional Visualizations}
\label{sec:visualizations}

\subsection{Inter-frame Correspondence}
\label{subsec:correspondence}

\begin{figure*}[htbp]
  \centering
    \includegraphics[width=0.94\linewidth]{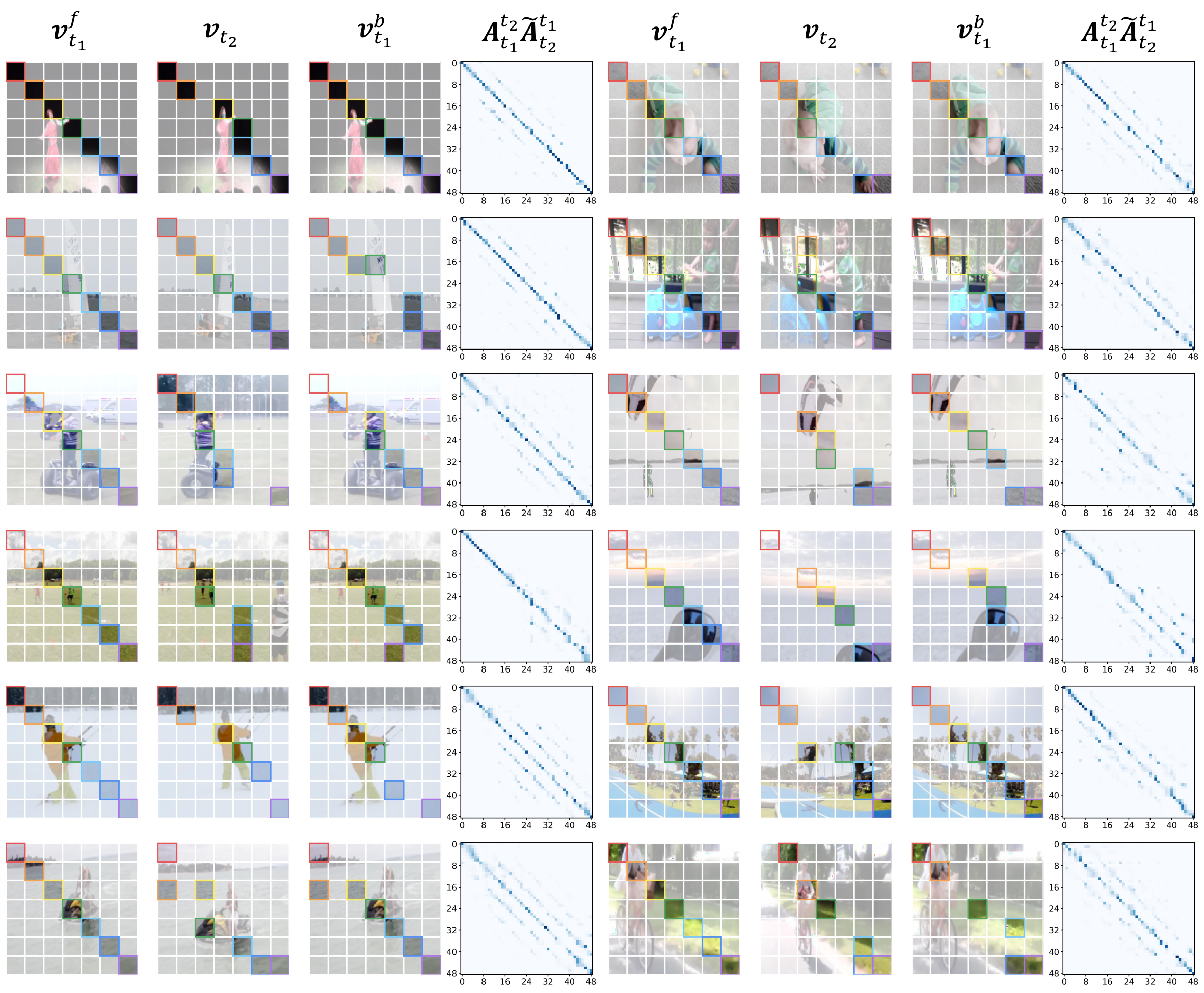}
    \caption{Cross-frame correspondence learned with our method. Patches with the same color box represent correspondence.}
    \vspace{-12pt}
  \label{fig:vis_correspondence_supp}
\end{figure*}

We visualize the inter-frame correspondence learned by the projection layer $g$ in \cref{fig:vis_correspondence_supp}. The results indicate that most patches establish consistent matches across frames and successfully return to their original locations through the forward-backward cycle.
Notably, due to factors such as camera motion and non-rigid object deformation, patch correspondences between $\vta^f$ and $\vtb$ are not strictly bijective.
A single patch in $\vta^f$ often correlates to multiple adjacent regions in $\vtb$, resulting in a correlation matrix product $\Atatb \Atbtatilde$ that exhibits a diagonally dominant structure rather than an exactly equal to the identity matrix $\bm{I}$. 
This observation reveals the dilemma of the original contrastive random walk strategy: it needs to constrain the matrix to the identity matrix to ensure good cyclic consistency, but we cannot make it a perfect identity matrix because it would allow the model to take advantage of shortcuts in displaying positional encoding.
This further justifies the necessity of our proposed PEA strategy, which effectively suppresses shortcut matching to stabilize correspondence learning.

\subsection{Downstream Task Performance}
\label{subsec:downstream}

\begin{figure*}[htbp]
  \centering
    \includegraphics[width=0.94\linewidth]{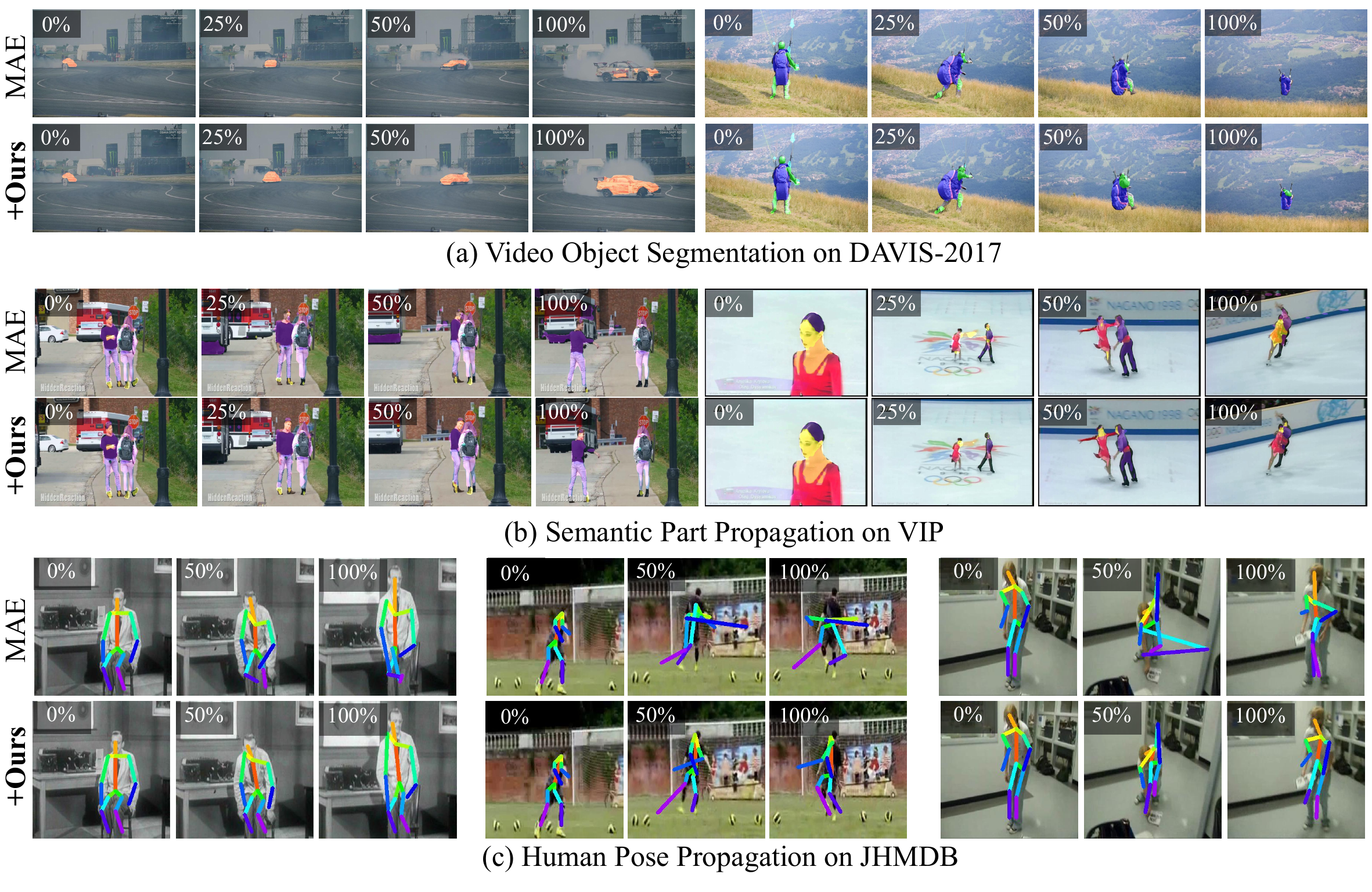}
    \vspace{-12pt}
    \caption{Visualization comparison across three downstream tasks based on MAE.}
  \label{fig:downstreams_supp_1}
\end{figure*}

\begin{figure*}[htbp]
  \centering
    \includegraphics[width=0.94\linewidth]{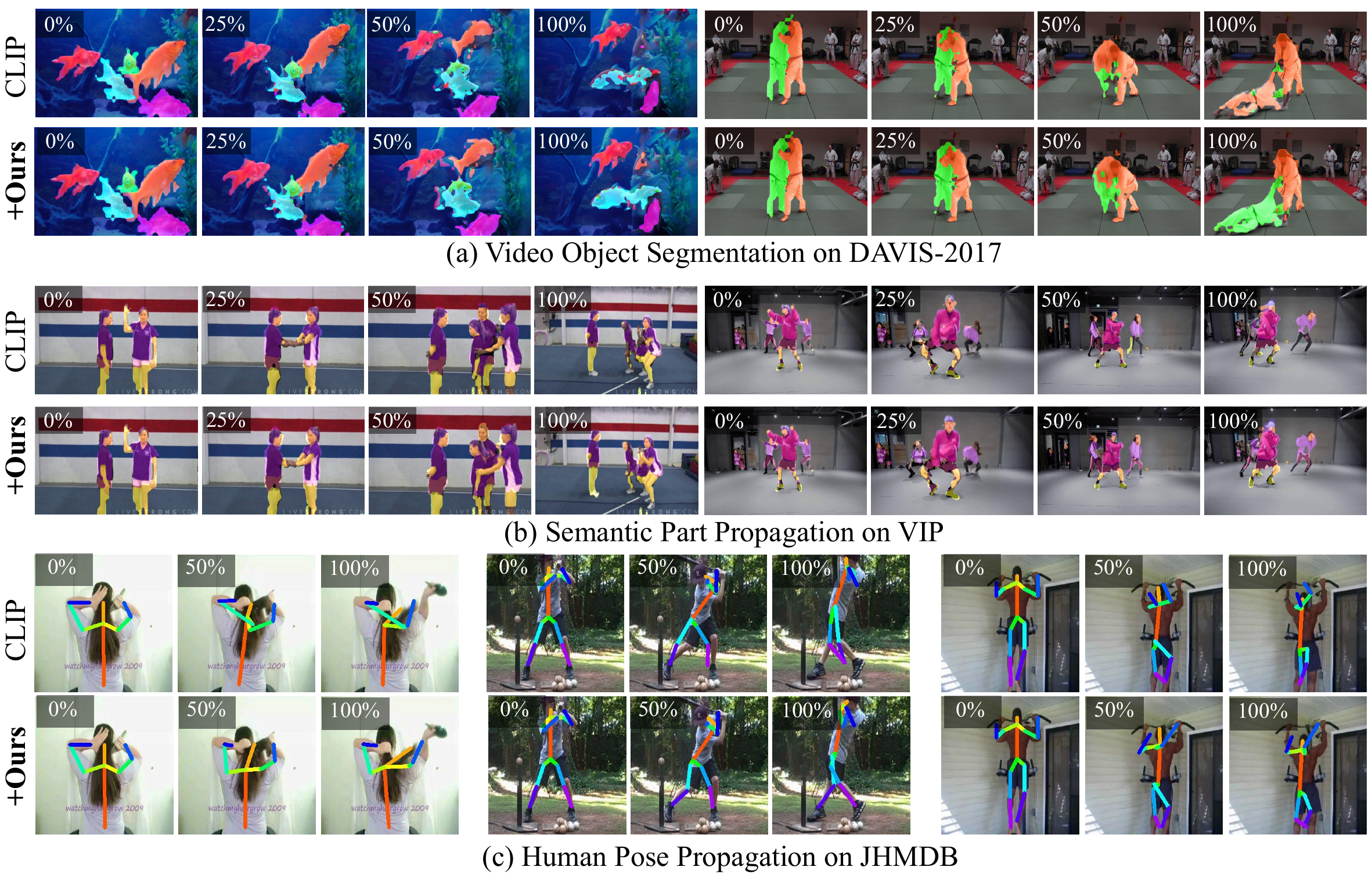}
    \vspace{-8pt}
    \caption{Visualization comparison across three downstream tasks based on CLIP.}
  \label{fig:downstreams_supp_2}
\end{figure*}

In \cref{fig:downstreams_supp_1,fig:downstreams_supp_2}, we compare the performance of original image-pretrained models and our transferred models across three downstream tasks. Our method shows visible improvements in several challenging scenarios, such as rapid movements, complex object boundaries, and motion-induced artifacts, where the original models often underperform.
These results suggest that incorporating temporal correspondence and strengthening semantic structure improves image-to-video representation transfer, validating the effectiveness of our method.

\section{Detailed Related Work}
\label{sec:detailed_related_Work}

\subsection{Self-supervised Visual Representation}
\label{subsec:related_work_SSL}

The rapid progress of self-supervised learning has enabled models to acquire generalizable representations for diverse downstream tasks in both the image and video domains. Depending on the nature of the pretraining objective, existing approaches can be broadly categorized into three paradigms.

\textbf{Contrastive learning} learns invariant representations by maximizing agreement between relevant instances while pushing apart representations of different instances. Early methods in the image domain construct positive and negative pairs~\cite{SimCLRv1,MoCov1,MoCov3} or apply diverse augmentations~\cite{BYOL,SwAV,SimSiam} to generate contrasting views. These approaches demonstrate strong generalization capabilities~\cite{huangtowards} and have been successfully extended to the video domain. By leveraging 3D convolutions~\cite{SlowFast}, temporal self-attention~\cite{TimeSformer,ViViT,XViT}, or inter-frame contrastive objectives~\cite{contrast1,contrast2,contrast3,thoker2023tubelet}, such methods benefit from spatiotemporal cues and have shown promising results on discriminative tasks such as action recognition and video retrieval~\cite{UCF101,HMDB51}.

\textbf{Masked modeling} aims to reconstruct the original RGB values of masked image patches in the pixel space~\cite{MAE,BEiTv1,I_JEPA,SimMIM,StoP,SiameseIM}. A representative method is MAE~\cite{MAE}, which employs an encoder-decoder architecture based on Vision Transformers~\cite{Vision_Transformer} to restore the masked regions, thereby capturing structural dependencies within the images. By incorporating the additional temporal dimension, MAE can be naturally extended for video representation learning~\cite {MAE_ST,VideoMAE,VideoMAEv2,DropMAE,VideoMAC}. To alleviate the computational cost of dense modeling, recent methods focus on more efficient designs. SiamMAE~\cite{SiamMAE} leverages sparsely sampled frames, asymmetric masking, and a conditional Siamese architecture, motivating subsequent works that improve frame selection and predictive mechanisms~\cite{CropMAE,RSP,STP}.

\textbf{Self-distillation} methods supervise a student network using outputs from a teacher network without relying on explicit labels, often focusing on restoring latent representations rather than raw pixels. This encourages the learning of high-level semantic information, aligning with principles of information compression~\cite{compression_2,compression_3}. DINO~\cite{DINO} adopts a self-distillation framework with Vision Transformers to align patch-level representations across views. Subsequently, iBOT~\cite{iBOT} and DINO v2~\cite{DINOv2} extend this paradigm by enforcing consistency in both global \texttt{[CLS]} tokens and dense patch representations.

\subsection{Image-to-video Transfer Learning}
\label{subsec:related_work_I2V}

\textbf{Temporal structure enhancement methods} typically design training objectives in a two-stage training manner based on self-supervised image contrastive learning frameworks~\cite{MoCov1,BYOL}.
In the first stage, models are pretrained on image datasets to learn static representations for instance-level discrimination~\cite{two_stage_1,two_stage_2}, or on synthetic videos to capture object motion patterns~\cite{two_stage_3,two_stage_4}. In the second stage, the models are fine-tuned on real video datasets to refine temporal correspondences, enabling them to perform specific video tasks.
However, the high spatiotemporal complexity hinders swift cross-domain representation transfer, motivating the exploration of parameter-efficient fine-tuning alternatives in subsequent works.

\textbf{Parameter-efficient fine-tuning methods} aim to adapt pretrained models to video tasks by updating only a small fraction of parameters.
Specifically, several methods insert adapters into Vision Transformer~\cite{Vision_Transformer} pretrained by CLIP~\cite{CLIP} in a series or parallel way, enabling spatial-temporal joint adaptation through expanded convolution or attention modules~\cite{adapter1,adapter2,adapter3,adapter4,adapter5}.
Other approaches decouple spatial and temporal modeling using dual-branch architectures~\cite{dual1,dual2}, enabling separate reasoning across spatial and temporal dimensions.
These adaptation methods are often trained on supervised action recognition datasets~\cite{Kinetics,ssv2}, which require further fine-tuning when applied to different benchmarks.
More recent work explores object-centric adaptation via slot attention~\cite{slot_attenion}, demonstrating the potential of using image-pretrained encoders for dense prediction tasks~\cite{dense_adapter}.
In a related direction, ProLIP~\cite{clip_adapter} shows that fine-tuning only the visual projector is effective for few-shot CLIP adaptation, showing the strong transfer capacity of lightweight projection-based adaptation.

\subsection{Temporal Cycle Consistency}
\label{subsec:related_work_temporal}
The inherent visual correspondence between temporally adjacent observations provides a powerful supervisory signal to capture spatiotemporal coherence in videos~\cite{continuity,smooth}.
Leveraging this property, numerous studies attempt to learn semantically consistent representations with a cycle structure, showing effectiveness in dense-level video tasks, including object segmentation~\cite{DAVIS17,VIP}, motion estimation~\cite{JHMDB}, and point tracking~\cite{BADJA,Kinetics_track}. 
Early methods mainly focus on tracking patches or objects across frames in a bidirectional manner~\cite{cycle_patch_1,cycle_patch_2,cycle_patch_3}, while others align the feature distributions among videos from the same category to enforce semantic consistency~\cite{cycle_match_1,cycle_match_2,cycle_match_3}.
Another line of work introduces random walk strategies~\cite{cycle_walk_1,cycle_walk_2,cycle_walk_3}, where representation learning is guided by maximizing the probability of each patch returning to itself via a forward-backward cycle.

\section{Additional Discussions}
\label{sec:discussion}

\subsection{Limitation and Future Work}
\label{subsec:limitation}

This work explores a more efficient and effective approach to transferring image representations to the video domain. In this work, we mainly focus on ViT-based backbones under the evaluated settings.
For future work, we plan to extend the method to other visual backbones, including lightweight architectures (\eg, CNNs, ResNets) and emerging large-scale vision models, to assess whether the observed trade-off is a general property of visual representations for video understanding.

\subsection{Broader Impact}
\label{subsec:impact}

We examine the trade-off between intra-video temporal consistency and inter-video semantic separability in visual representations and, based on this view, propose a method for image-to-video representation transfer learning. The proposed method achieves competitive or superior performance compared with models pretrained on video from scratch, providing a lightweight alternative for video representation learning. It may also provide a useful perspective for future research on image-to-video transfer in broader scenarios.


\end{document}